\documentclass{article}

\usepackage{arxiv}

\usepackage[utf8]{inputenc} 
\usepackage[T1]{fontenc}    
\usepackage{hyperref}       
\usepackage{xurl}            
\usepackage{booktabs}       
\usepackage{amsfonts}       
\usepackage{nicefrac}       
\usepackage{microtype}      
\usepackage{lmodern,needspace,longtable}
\usepackage{graphicx}

\usepackage{multirow}%
\usepackage{amsmath,amssymb}%
\usepackage{amsthm}%
\usepackage{mathrsfs}%
\usepackage[title]{appendix}%
\usepackage{xcolor}%
\usepackage{textcomp}%
\usepackage{manyfoot}%
\usepackage{listings}%

\usepackage{xspace}
\usepackage{tabularx}
\usepackage{array}

\graphicspath{{figures/}{paper_assets/figures/}{./}}
\usepackage{tikz}
\usetikzlibrary{arrows.meta,positioning,fit,calc}
\definecolor{AccentYellow}{HTML}{D4A017}
\newcommand{\A}{\mathcal{A}}
\newcommand{\Hh}{\mathcal{H}}
\newcommand{\Toks}{\mathcal{T}}
\newcommand{\Msgs}{\mathcal{M}}
\newcommand{\Prob}{\Delta}
\newcommand{\conf}{\mathrm{conf}}
\newcommand{\argmax}{\mathrm{argmax}}
\newcommand{\TV}{\mathrm{TV}}
\newcommand{\PBRC}{{\normalfont\textsc{PBRC}}\xspace}
\newcommand{\CDDL}{{\normalfont\textsc{CDDL}}\xspace}
\newcommand{\FO}{\mathrm{FO}}

\newcommand{\pow}{\mathcal{P}}
\newcommand{\Nat}{\mathbb{N}}
\newcommand{\R}{\mathbb{R}} 
\newcommand{\Ev}{\mathbb{E}\mathrm{v}}
\renewcommand{\paragraph}[1]{\medskip\noindent\textbf{#1}\quad}
\newtheorem{Theorem}{Theorem}
\newtheorem{Proposition}{Proposition}
\newtheorem{Lemma}{Lemma}
\newtheorem{Remark}{Remark}
\newtheorem{example}{Example}%
\newtheorem{Corollary}{Corollary}
\newtheorem{Definition}{Definition}%

\title{Preregistered Belief Revision Contracts}

\author{
  Saad Alqithami \\
  \texttt{alqithami@gmail.com} \\
}

\hypersetup{hidelinks,pdftitle={Preregistered Belief Revision Contracts},pdfauthor={Saad Alqithami}}
\date{}
\begin{document}
\maketitle

\begin{abstract}
Multi-agent language-model systems can change their answers in response to social cues without acquiring new externally validated evidence. We study \PBRC{} (\emph{Preregistered Belief Revision Contracts}), which associate public evidence conditions with permitted updates to an explicit belief state. A router enforces the selected contract branch and records a certificate for each accepted non-fallback transition. Under evidence gating and an argmax-preserving fallback, token-empty interaction preserves the top-hypothesis set; a nonamplifying fallback also prevents growth of the specified confidence functional. We characterize these social-only properties by local transition conditions and distinguish behavior-preserving normal forms from safety-enforcing restrictions. For token-invariant contracts, identical validated-token exposure traces induce identical enforced trajectories under the specified routing assumptions. Under synchronous flooding, accumulated token knowledge follows ordinary distance reachability, whereas incoming exposure follows positive-length walks; this distinction yields exact trace characterizations and the standard tight diameter bound for universal knowledge closure. We give source-labelled accountability conditions for finite execution logs, prove soundness of a companion KD45/PDL specification system, and establish bounded assignment certificates for a guarded conjunctive trigger fragment. Controlled simulations and a matched synthetic study illustrate the conditional guarantees and vulnerability to correlated label errors. KAIROS experiments demonstrate token-empty answer retention on 3,000 rows. A separate 200-claim SciFact study records evidence-enabled admission: PBRC achieves 53.0\% accuracy, versus 49.0\% for RAW and 61.5\% for clean RAG; its gain over RAW is statistically inconclusive, with 12 harmful and 20 beneficial answer changes. The framework governs external protocol state and accepted outputs; it does not by itself establish semantic truth, calibrated probabilities, or improved task accuracy.
\end{abstract}

\keywords{belief revision\and dynamic doxastic logic\and protocol semantics\and auditability\and graph reachability\and multi-agent deliberation\and large language models.}



\section{Introduction}

Deliberation can improve the information available to an agent, but it can also make agreement, confidence, prestige, or majority size substitutes for independently checkable reasons. Recent LLM studies report conformity and peer-pressure effects \cite{weng2025benchform,song2025kairos}, including sensitivity to topology and self--social weighting \cite{han2026conformity}. Emergent conventions and collective biases are related but distinct phenomena \cite{ashery2025sciadv}. These results motivate a protocol question: which changes to a public belief state should a deliberation system permit, and what evidence should accompany them?

We introduce \PBRC{}, a public contract associating first-order event predicates with permitted update operators, a deterministic priority order, and a fallback. A triggered update requires a nonempty validated witness, full-event satisfaction of the selected trigger, and faithful execution of its operator. Fallback remains an explicit transition and may reduce confidence without changing the top set. Thus the framework controls the admissibility of externally represented transitions rather than asserting that every numerical change must carry new evidence.

The separation between communication and admissible revision is conditional. An evidential trigger merely requires the presence of a validated token; it can still depend on message metadata unless token invariance is imposed. A signature authenticates an artifact, but does not make its semantic label correct. A state-holding router enforces the public state, but does not directly constrain an LLM's private reasoning. These distinctions delimit all claims below.

\subsection{Contributions}
We develop five connected results.
\begin{enumerate}
\item \textbf{Contracts, witnesses, and enforcement.} Section~\ref{sec:pbrc} separates full-event trigger selection, restricted witness checking, and numerical operator compliance. The default semantics uses an authoritative state-holding router.
\item \textbf{Conditional social-only guarantees.} Sections~\ref{sec:main}--\ref{sec:minimality} prove top-set stability and maximum-confidence nonamplification under conservative fallback, and characterize these properties by the actually selected one-step transition (Theorem~\ref{thm:necessity_evidence}). Evidence gating is a sufficient design discipline, not a necessary syntactic form for every robust protocol.
\item \textbf{Normal forms and licensed transitions.} Section~\ref{sec:normalform} gives behavior-preserving and enforcement-based translations under explicit witness and routing hypotheses. The former may absorb unsafe behavior into fallback; it does not create safety. Under faithful enforcement, each top-set change has an accepted certificate licensing that transition (Theorem~\ref{thm:accountability}).
\item \textbf{Token-trace factorization and flooding.} Token-invariant contracts separate rhetorical presentation from transition inputs under token-determined or complete routing. Incoming exposure is characterized by positive-length walk relations (Theorems~\ref{thm:flood_distance} and \ref{thm:trace_tight}); accumulated knowledge has the standard tight diameter bound (Corollary~\ref{cor:diameter_tight}). Universal trace equivalence includes adjacency equality and is not claimed as a nontrivial graph compression.
\item \textbf{Threat assumptions, specification, and checking.} Sections~\ref{sec:adversary}--\ref{sec:complexity} separate provenance, semantic correctness, replay age, and validator compromise. A KD45/PDL specification system is sound, and its belief-free fragment is checkable on faithful finite logs. Guarded existential triggers admit small assignment certificates with linear logical verification cost after lookup indexing; search and validation costs remain separate. Section~\ref{sec:eval} evaluates protocol behavior through controlled simulations, a matched synthetic study, and KAIROS and SciFact experiments.
\end{enumerate}

\begin{center}
\fbox{\begin{minipage}{0.94\linewidth}
\textbf{Scope.} The guarantees concern fixed finite hypothesis sets, public transition inputs, conservative fallback, and faithful enforcement. Truth preservation, statistical calibration, secure evidence generation, dynamic hypothesis expansion, and empirical benefit under new evidence require additional assumptions or evaluation. The logical results establish soundness and finite-model monitoring.
\end{minipage}}
\end{center}

\subsection{Relation to underlying revision operators}
\PBRC{} leaves the choice of numerical update operator to the contract designer. It is compatible with evidence aggregation but does not establish AGM rationality postulates for probability vectors, select a universally optimal inference rule, or make social consensus intrinsically evidential. It constrains a recorded transition according to public rules. The relation to runtime enforcement, evidence logics, and agent protocols is discussed next.

The operational model and contract semantics appear in Sections~\ref{sec:operational}--\ref{sec:pbrc}, followed by a worked example and the formal guarantees. Sections~\ref{sec:adversary}--\ref{sec:complexity} address threats, monitoring, implementation, and cost. Section~\ref{sec:eval} presents the experimental methods and results. Appendix~\ref{app:pdl} gives the proofs of logical soundness and finite-model monitoring.

\section{Related work}
\label{sec:related}

Empirical studies of LLM multi-agent systems (MAS) document systematic conformity and peer-pressure effects whose strength depends on the interaction protocol, peer reliability, rapport, and network structure. BenchForm provides a benchmarked characterization of conformity across interaction protocols \cite{weng2025benchform}. KAIROS studies peer-pressure under heterogeneous peer reliability and rapport \cite{song2025kairos}. Recent results also emphasize that topology and self--social weighting can modulate conformity and ``wrong-but-sure'' cascades \cite{han2026conformity}. Complementary work adapts classic social-psychology paradigms and reports that AI agents exhibit conformity patterns aligned with Social Impact Theory \cite{bellina2026socialimpact}, connecting to foundational human-group evidence on conformity and social impact \cite{asch1951effects,latane1981socialimpact}. Our contribution concerns admissibility rather than measuring conformity: PBRC specifies when a public state transition is licensed. Rhetoric independence additionally requires token-invariant predicates and operators.

Repeated interaction among LLM agents can also yield emergent conventions and collective biases \cite{ashery2025sciadv}; the earlier preprint is a version of the same study \cite{ashery2024conventions}. These phenomena motivate mechanisms that distinguish coordination signals from epistemic justification. PBRC targets the epistemic side: contracts restrict which belief transitions are admissible and require certificates with validated evidence tokens, so convention formation cannot by itself justify epistemic flips absent evidence.

Belief revision and belief-change formalisms provide the canonical language for rational change of epistemic states. AGM belief revision axiomatizes rational postulates for theory change \cite{agm1985,gardenfors1988knowledge,hanssonSEP}. Belief update distinguishes revising by new information from updating after world change \cite{katsuno1991update}, and iterated revision studies sequential change \cite{darwiche1997}. Ranking-theoretic approaches provide alternative representations of epistemic states and revision dynamics \cite{spohn2012lawsofbelief}. Dynamic epistemic logic (DEL) models informational actions and announcements \cite{delSEP,ditmarsch2007dynamic}, and dynamic logics of belief change connect AGM-style ideas with dynamic and plausibility semantics in both single- and multi-agent settings \cite{vanbenthem2015dlbc}. PBRC does not propose a new underlying revision operator; instead it constrains which transitions are admissible under social interaction via preregistered, token-witnessable triggers, and it supports auditing via checkable certificates. Relatedly, belief merging and judgment aggregation study how to combine information from multiple sources under inconsistency \cite{konieczny2002merging,konieczny2011logicbased}; PBRC is compatible with such operators but adds an enforceable evidence gate for when they may be applied.

Runtime enforcement already studies monitors that restrict or transform executions according to specified policies \cite{schneider2000enforceable,ligatti2005edit}. PBRC uses that established architecture for evidence-labelled transitions of a public belief state; the presence of a trusted gate is not itself a new enforcement theorem. Evidence-based dynamic logics also distinguish evidence structures from resulting beliefs \cite{vanbenthem2011evidence}. Our contribution lies in the explicit combination of contract selection, validated token witnesses, router-held numerical state, and exposure-trace analysis, under the assumptions stated in each result.

Combining belief modalities with dynamic/program operators has a substantial history in dynamic doxastic logic and related systems \cite{leitgeb2007ddl,schmidt2008pdl_doxastic}; see \cite{fagin1995reasoning,benthem2011logical} for broader background on multi-agent epistemic/doxastic logic and logical dynamics. Our novelty is not the presence of KD45 + PDL per se, but the use of program structure to encode PBRC contracts and to specify invariants over enforced, certificate-carrying belief transitions (Section~\ref{sec:cddl}).

Compliance monitoring in MAS is often framed in terms of commitment-based protocols and social/interactional commitments \cite{singh1999commit,yolum2002commitment}. PBRC enables compliance monitoring for epistemic transitions: accepted non-fallback transitions require verifiable evidence tokens and witness sets that an external router or auditor can check; fallback can still change numerical confidence. This certificate discipline is reminiscent of proof-carrying mechanisms, where untrusted producers attach checkable artifacts enabling efficient validation by a consumer \cite{necula1997pcc}. Tamper-evident audit logs are a standard accountability primitive \cite{kelsey1999secure}; PBRC certificates are designed to be stored in such append-only logs. If one additionally wants to bind correct execution of updates (not just admissibility), verifiable computation and succinct proof systems provide relevant primitives \cite{gennaro2010verifiable,parno2013pinocchio} (Section~\ref{sec:enforcement}).

Trust and reputation models are another common response to unreliable peers in MAS \cite{sabater2005review}. PBRC takes a complementary stance: rather than granting epistemic force to socially mediated reputation, it treats reliability as an evidence problem to be expressed through verifiable artifacts (tool provenance, signed attestations, audited verifier judgments). Socially persuasive claims remain non-binding absent checkable evidence.

Our notion of ``wrong-but-sure'' cascades connects to classic informational cascades and herding in economics and social learning, where agents update based on others' actions and can converge on incorrect beliefs \cite{banerjee1992herd,bikhchandani1992theory}. Networked aggregation and Bayesian social learning further show how topology shapes convergence and error propagation \cite{degroot1974consensus,bala1998learning,golub2010naive,acemoglu2011bayesian}. PBRC differs in that it does not prescribe a specific update rule; instead, it constrains which belief transitions are admissible via preregistered, checkable evidence triggers, yielding a factorization through token dissemination when the contract and router satisfy token invariance.

Recent LLM MAS often use debate/critique/reflection patterns to improve reasoning \cite{liang2023multiagentdebate,li2023camel,shinn2023reflexion,madaan2023selfrefine}. Practical orchestration frameworks operationalize such protocols in agentic programming settings \cite{wu2023autogen,hong2023metagpt}. These approaches are typically implemented as prompting patterns without externally checkable admissibility constraints; PBRC is complementary, providing enforceable constraints that can prevent purely social persuasion steps absent admissible evidence.

Evidence tokens can be instantiated using retrieval-augmented and tool-augmented LLM pipelines, where snippets, citations, and tool outputs provide structured artifacts that a verifier can validate \cite{lewis2020rag,yao2022react,nakano2021webgpt,menick2022verifiedquotes,schick2023toolformer}. Retrieval-and-revise methods and hallucination detectors can further serve as candidate verifiers or token labelers \cite{asai2024selfrag,manakul2023selfcheckgpt}. On the systems side, provenance standards such as W3C PROV provide a natural representation for evidence lineage \cite{w3cprov2013}. PBRC abstracts these implementation details into a token interface ($\mathrm{Valid}$/$\mathrm{Supports}$/$\mathrm{Contradicts}$/$\mathrm{Type}$/$\mathrm{Fresh}$) and studies protocol-level guarantees when admissibility is enforced via checkable certificates.

Taken together, these lines of work show that social interaction can induce conformity, cascades, and convention formation, while belief-change logics and systems mechanisms provide tools for specifying and enforcing rational updates. PBRC connects these perspectives by making social robustness and auditability enforceable at the protocol level: preregistered contracts specify which evidence patterns license which updates; routers enforce admissibility by routing token-empty steps to fallback; and auditors can inspect the accepted license alongside full transition inputs and outcomes.


\section{Operational model: evidence-annotated deliberation}
\label{sec:operational}

We formalize a deliberation run among agents who exchange messages containing \emph{evidence tokens}. Tokens represent checkable artifacts (tool outputs, verified citations, signed sensor readings, etc.). The model separates persuasive language (unvalidated message content) from validated evidence (tokens that satisfy $\mathrm{Valid}$ under a router-side validity layer).

\subsection{Agents, hypotheses, and beliefs}

Let $\A=\{1,\dots,n\}$ be agents and $\Hh=\{h_1,\dots,h_m\}$ a finite hypothesis set. Each agent $i$ has a belief distribution $b_i^t\in \Prob(\Hh)$ at round $t$:
\[
  b_i^t(h)\ge 0,\quad \sum_{h\in\Hh} b_i^t(h)=1.
\]
Define:
\[
  \conf(b):=\max_{h\in\Hh} b(h),\qquad
  \argmax(b):=\{h\in\Hh : b(h)=\conf(b)\}.
\]
For distances between beliefs we use total variation:
\[
  \TV(b,b') := \frac12 \sum_{h\in\Hh} |b(h)-b'(h)|.
\]

\subsection{Belief state as a router-side artifact (LLM instantiation)}
\label{sec:llm_inst}

The formal object $b_i^t\in\Prob(\Hh)$ is a \emph{protocol state} used for deliberation, enforcement, and auditing. In deployed LLM-agent systems it is often unreasonable to assume that agents maintain stable, calibrated internal Bayesian beliefs, or that self-reported probabilities are trustworthy. PBRC does not require these assumptions: a router/orchestrator can maintain an authoritative belief state and compute contract updates itself (Section~\ref{sec:enforcement}), while agents contribute candidate hypotheses, arguments, and evidence tokens.

\paragraph{Finite candidate sets.}
When the task exposes a finite set of candidates (multiple choice, ranked options, tool-selected actions), let $\Hh$ be the candidate set and let the router maintain a score vector $s_i^t(h)\in\mathbb{R}$ (e.g., log-probabilities from the LLM, an auxiliary verifier model, or a task-specific scoring function). Define $b_i^t=\mathrm{softmax}(s_i^t)$ (optionally after calibration). In this setting, the belief distribution is fully external and auditable.

\paragraph{Open-ended tasks and dynamic hypothesis sets.}
For open-ended tasks, a practical pattern is to maintain a \emph{dynamic} hypothesis set $\Hh_t$ produced by candidate generation and pruning. A proposed extension would treat \emph{hypothesis introduction} as a contract-governed operation: new hypotheses may enter $\Hh_t$ only when an evidential trigger fires (e.g., a validated retrieval token that introduces a new entity/answer), while social-only rounds cannot expand $\Hh_t$. Formally, one can model each episode as a sequence of \emph{stages}, each with a fixed finite hypothesis set. A stage-transition operator (triggered by evidence) expands $\Hh$ and reinitializes $b$ over the expanded set. The theorems below apply within each fixed stage. Correctness of an expansion or reinitialization operator requires a separate specification and is not established here.

\paragraph{Takeaway.}
The fixed finite $\Hh$ in the core theorems should be read as a stage-wise view: candidates are explicitly managed by the router and expanded only under evidential triggers.

\subsection{Evidence tokens and first-order annotations}
\label{sec:tok_fo}

Let $\Toks$ be a universe of evidence tokens and $\Msgs$ a universe of messages. We use a many-sorted first-order signature with sorts:
\[
\mathsf{Ag},\; \mathsf{Tok},\; \mathsf{Msg},\; \mathsf{Hyp},\; \mathsf{Time},\; \mathsf{Schema},\; \mathsf{Context}.
\]
Predicates include:
\begin{align*}
&\mathrm{Send}(a,b,m,t) && \text{agent $a$ sends message $m$ to $b$ at time $t$},\\
&\mathrm{HasTok}(m,\tau) && \text{$m$ contains token $\tau$},\\
&\mathrm{Valid}(\tau) && \text{$\tau$ is externally verifiable under the validity layer},\\
&\mathrm{Supports}(\tau,h),\ \mathrm{Contradicts}(\tau,h) && \text{$\tau$ supports/contradicts hypothesis $h$},\\
&\mathrm{Type}(\tau,\kappa) && \text{$\tau$ has schema/type $\kappa$ (e.g., tool vs.\ retrieval)}.
\end{align*}

\paragraph{Events and token projection.}
At each round $t$, agent $i$ receives a multiset $E_i^t\subseteq \Msgs$ (its \emph{event}). Define the set of tokens \emph{mentioned} in an event and the subset that validate:
\[
\mathcal{T}_{\mathrm{all}}(E)\;:=\;\{\,\tau\in\Toks : \exists m\in E\ \mathrm{HasTok}(m,\tau)\,\},\qquad
\mathcal{T}(E)\;:=\;\{\,\tau\in\mathcal{T}_{\mathrm{all}}(E) : \mathrm{Valid}(\tau)\,\}.
\]
The validated tokens available to agent $i$ at time $t$ are $\mathcal{T}_i^t := \mathcal{T}(E_i^t)$.

\paragraph{Events as finite relational structures.}
To interpret first-order triggers, associate to each event $E_i^t$ a finite relational structure $\mathcal{S}(E_i^t)$ whose domains are:
agents $\A$, message occurrences in $E_i^t$, tokens in $\mathcal{T}_{\mathrm{all}}(E_i^t)$, hypotheses $\Hh$, the current and token-issue times used in the event, the finite public schema vocabulary, and the fixed public context identifiers.
Predicates such as $\mathrm{Send}$ and $\mathrm{HasTok}$ are interpreted from message metadata, while $\mathrm{Valid}$, $\mathrm{Supports}$, and $\mathrm{Contradicts}$ are interpreted via the external validity/labeling layer.
We write
\[
E_i^t\models \varphi \quad\text{iff}\quad \mathcal{S}(E_i^t)\models \varphi
\]
under Tarskian semantics with explicitly permitted empty message and token sorts: existential quantification over an empty sort is false and universal quantification is true. Each message occurrence has its own identifier, so multiplicity is represented even when message text repeats. Triggers use no token-sort constants or token-valued functions; token identities can be tested through relational attributes. This makes restriction to a witness subset well-defined. All non-token constants belong to their fixed public domains.

The event also fixes the instance identifier, current time, validator version, and semantic annotations used by the router. Comparisons of events in invariance results hold this public context fixed (or compare matching context at each round). Token identity alone does not justify equating different label interpretations or fresh and expired uses. The static flooding results additionally assume fixed validity and annotations during their horizon.

\subsubsection{Concrete instantiation: typed tokens and entailment-based labeling}
\label{sec:concrete_semantics}

The abstract interface above---a validity predicate $\mathrm{Valid}(\tau)$ together with $\mathrm{Supports}/\mathrm{Contradicts}$---allows different evidence representations. The following design uses typed records and entailment-based labels for LLM-mediated tasks. Section~\ref{sec:scifact} evaluates a restricted answer-level instantiation with retrieved passages and NLI labels; authenticated provenance remains an additional requirement for adversarial deployments. The formal results use the abstract interface and their explicit soundness assumptions.

\paragraph{Token typing.}
Treat evidence tokens as typed, self-describing records with explicit provenance, e.g.:
\begin{itemize}
  \item $\mathrm{ToolResult}(\textsf{tool}, \textsf{input\_hash}, \textsf{output}, \textsf{output\_hash}, \textsf{ts}, \textsf{sig})$:
  a deterministic tool invocation with authenticated I/O.
  \item $\mathrm{RetrievedSnippet}(\textsf{source\_id}, \textsf{url\_hash}, \textsf{snippet}, \textsf{snippet\_hash}, \textsf{ts}, \textsf{sig})$:
  a retrieval return from an allowed source with immutable snippet hashing.
  \item $\mathrm{VerifierJudgment}(\textsf{verifier\_id}, \textsf{claim}, \textsf{context\_hash}, \textsf{label}, \textsf{conf}, \textsf{model\_hash}, \textsf{sig})$:
  an authenticated judgment (e.g., NLI entailment) binding the label to a specific claim and context.
\end{itemize}
The contract fixes admissible token schemas and the authentication mechanism for each schema.

\paragraph{Validity.}
$\mathrm{Valid}(\tau)$ checks (i) schema well-formedness, (ii) authenticity of $\textsf{sig}$, and (iii) type-specific replay/consistency conditions. Examples:
for $\mathrm{ToolResult}$, (iii) may include deterministic re-execution or recomputation to confirm $\textsf{output\_hash}$;
for $\mathrm{RetrievedSnippet}$, (iii) binds the snippet to an allowed source and time window;
for $\mathrm{VerifierJudgment}$, (iii) binds the judgment to a pinned $\textsf{model\_hash}$ and an exact $\textsf{context\_hash}$.

\paragraph{Supports/Contradicts via a deterministic labeling pipeline.}
For tool tokens, $\mathrm{Supports}/\mathrm{Contradicts}$ is a deterministic mapping from structured tool outputs to hypotheses.
For retrieval tokens, a simple auditable pattern is:
(i) canonicalize $h$ into a natural-language claim $\mathsf{claim}(h)$;
(ii) run a pinned NLI model on (\textsf{premise}=\textsf{snippet}, \textsf{hypothesis}=$\mathsf{claim}(h)$);
(iii) threshold entailment/contradiction scores to produce $\textsf{label}\in\{\textsf{entails},\textsf{contradicts},\textsf{neutral}\}$; and
(iv) emit a \emph{signed} $\mathrm{VerifierJudgment}$ token.
Then define
\[
\mathrm{Supports}(\tau,h)\ \text{iff}\ \tau \text{ is a VerifierJudgment with label entails and conf }\ge \theta_{\mathrm{sup}},
\]
and analogously for $\mathrm{Contradicts}(\tau,h)$ at threshold $\theta_{\mathrm{con}}$.

\paragraph{Partial/uncertain evidence and ``semantic laundering.''}
Binary predicates can be coarse. This instantiation retains the verifier confidence $\textsf{conf}$ and supports multi-attestation or $k$-of-$K$ checks (Section~\ref{sec:adversary}) for high-impact updates (e.g., belief flips), making agreement and provenance inspectable. Agreement among correlated or fallible semantic verifiers is not a proved truth guarantee; the compromise bounds below apply only under their stated independence and honest-validator assumptions.

\paragraph{Minimal token interface.}
A compact schema sufficient for many deployments is
\[
\tau=\langle \mathrm{type},\mathrm{payload},\mathrm{provenance},\mathrm{time},\mathrm{sig}\rangle,
\]
where $\mathrm{sig}$ is a validator signature or checkable attestation. The validity layer computes $\mathrm{Valid}(\tau)$ router-side. The relations $\mathrm{Supports}(\tau,h)$ and $\mathrm{Contradicts}(\tau,h)$ can be computed deterministically from a typed payload, or can themselves be asserted by a signed label (as in $\mathrm{VerifierJudgment}$).

\paragraph{Semantic mismatch is an explicit failure mode.}
Even if $\mathrm{Valid}(\tau)$ holds, the semantic interpretation of $\tau$ (and hence $\mathrm{Supports}/\mathrm{Contradicts}$) can be wrong or irrelevant. PBRC does not eliminate semantic alignment risk; it \emph{localizes} it to the validity/labeling layer and makes it auditable (Section~\ref{sec:adversary}).

\subsection{Token-empty events (social-only)}

\begin{Definition}[Social-only event]
An event $E$ is \emph{social-only} if it contains no validated tokens:
\[
\mathrm{SocialOnly}(E)\;:\iff\; \mathcal{T}(E)=\emptyset.
\]
Equivalently,
\[
\mathrm{SocialOnly}(E)\;:\iff\; \neg\exists m\in E,\ \exists \tau\in\Toks\;(\mathrm{HasTok}(m,\tau)\wedge \mathrm{Valid}(\tau)).
\]
\end{Definition}

\begin{Remark}[Social-only vs.\ trigger-inactive]
\label{rem:social_vs_trigger}
Social-only refers strictly to \emph{token-emptiness}: $\mathrm{SocialOnly}(E)$ means $\mathcal{T}(E)=\emptyset$.
Separately, for a given contract, a step is \emph{trigger-inactive} if $J_i^t=\emptyset$ (Definition~\ref{def:update}). Under an evidential contract, $\mathrm{SocialOnly}(E_i^t)$ implies $J_i^t=\emptyset$, but the converse need not hold (tokens may be present but fail to satisfy any preregistered trigger).
\end{Remark}


\section{Preregistered Belief Revision Contracts}
\label{sec:pbrc}

\paragraph{Semantic convention.}
As in Section~\ref{sec:operational}, each message event $E$ (a multiset of messages) induces a first-order structure $\mathcal{S}(E)$ over the public event/evidence signature.
We write $E \models \varphi$ as shorthand for $\mathcal{S}(E)\models \varphi$.

\paragraph{Token notation.}
Let $\Toks$ be the universe of tokens.
For an event $E$, define the \emph{candidate} and \emph{validated} token sets
\[
\mathcal{T}_{\mathrm{all}}(E)\;:=\;\{\tau\in\Toks:\exists m\in E\ \mathrm{HasTok}(m,\tau)\},
\qquad
\mathcal{T}(E)\;:=\;\{\tau\in\mathcal{T}_{\mathrm{all}}(E): \mathrm{Valid}(\tau)\}.
\]
For a step event $E_i^t$, we write $\mathcal{T}_i^t:=\mathcal{T}(E_i^t)$.

\subsection{Syntax}

\begin{Definition}[Contract]
A \PBRC{} contract for agent $i$ is a publicly preregistered tuple
\[
C_i=\langle (\varphi_{i,1},U_{i,1}),\dots,(\varphi_{i,k_i},U_{i,k_i}),\prec_i,\delta_i\rangle
\]
where:
\begin{enumerate}
\item each trigger $\varphi_{i,j}\in \FO$ is an event predicate evaluated on $\mathcal{S}(E)$ at runtime;
\item each $U_{i,j}$ is a (possibly partial) revision operator that maps a current belief state and the current event to a next belief state,
\[
U_{i,j}:\Prob(\Hh)\times \mathsf{Evt}\ \rightharpoonup\ \Prob(\Hh),
\]
with $\mathsf{Evt}$ the set of events (multisets of messages);
\item $\prec_i$ is a strict total priority order on $\{1,\dots,k_i\}$ used to resolve multiple satisfied triggers (so the selected trigger index is unique);
\item $\delta_i$ is a fallback operator used when no trigger is satisfied (or when enforcement rejects a proposed triggered update).
\end{enumerate}
\end{Definition}

\begin{Definition}[Evidential contract]
A contract $C_i$ is \emph{evidential} if every trigger can fire only on non-social events, i.e.,
\[
\forall j,\ \forall E,\quad E\models \varphi_{i,j}\ \Longrightarrow\ \mathcal{T}(E)\neq \emptyset.
\]
Equivalently (using $\mathrm{SocialOnly}(E)\iff \mathcal{T}(E)=\emptyset$),
\[
\forall j,\ \forall E,\quad E\models \varphi_{i,j}\ \Longrightarrow\ \neg \mathrm{SocialOnly}(E).
\]
\end{Definition}

\begin{Remark}[Well-formedness]
To avoid undefined transitions, we assume contracts are \emph{well-formed} in the sense that whenever a trigger is selected by the priority rule, the associated operator is defined for every belief and event on which it is selected, and fallback is total on all belief/event pairs. (If an implementation encounters an undefined operator at runtime, it must fall back and log the step as non-triggered.)
\end{Remark}

\subsection{Contract-constrained update}

\begin{Definition}[Trigger set and contract update]
\label{def:update}
Given belief $b_i^t$ and event $E_i^t$, define the \emph{trigger set}
\[
J_i^t := \{j\in\{1,\dots,k_i\} : E_i^t \models \varphi_{i,j}\}.
\]
The contract-constrained update is:
\begin{equation}
\label{eq:update}
b_i^{t+1} :=
\begin{cases}
U_{i,j^\star}(b_i^t,E_i^t), & \text{if } J_i^t\neq \emptyset,\ j^\star=\min_{\prec_i} J_i^t,\\
\delta_i(b_i^t,E_i^t), & \text{if } J_i^t=\emptyset.
\end{cases}
\end{equation}
\end{Definition}

\subsection{Certificates and auditability}

\begin{Definition}[Update certificate]
An abstract step certificate is a pair $\pi_i^t=(\ell_i^t,W_i^t)$, with label $\ell_i^t\in\{1,\ldots,k_i\}\cup\{\bot\}$ and token set $W_i^t\subseteq\Toks$. Fallback is represented by $(\bot,\emptyset)$. An accepted triggered step requires a nonempty validated witness, satisfaction of the named trigger on the full event, and selection of that trigger by the contract's total priority order. Implementations may additionally carry a proof payload $\sigma$; for conjunctive existential triggers this is the variable assignment checked in Section~\ref{sec:complexity}. The pair notation projects away that payload and is not a claim that an unstructured set alone can be checked in linear time.
\end{Definition}

\begin{Definition}[Token-restricted satisfaction and witnesses]
\label{def:token_restricted_sat}
For $W\subseteq\mathcal{T}_{\rm all}(E)$, let $\mathcal S_W(E)$ restrict the event's token domain and all token-dependent relations to $W$, retaining the public context and non-token event relations. Define
\[
E\models_W\varphi\quad\Longleftrightarrow\quad\mathcal S_W(E)\models\varphi.
\]
A witness for $\varphi$ on $E$ is a set $W\subseteq\mathcal T(E)$ such that $E\models_W\varphi$. A deterministic extractor for an evidential trigger must return such a nonempty set whenever the trigger is satisfied. Existence of an extractor is an additional contract requirement; it does not follow for arbitrary first-order formulas. For general trigger protocols, Definition~\ref{def:wit} states the precise conditional requirement used by the translations.
\end{Definition}

A router checks (i) $W\ne\emptyset$ and current-event membership and validity, (ii) restricted satisfaction $E\models_W\varphi_{i,\ell}$ using any required proof payload, (iii) full-event satisfaction $E\models\varphi_{i,\ell}$, and (iv) $\ell=\min_{\prec_i}J(E)$. These are distinct checks. For nonmonotone formulas, restricting tokens can remove a counterexample, so restricted satisfaction need not imply full-event satisfaction. Conversely, a trigger that requires an invalid token need not admit a witness consisting only of validated tokens.

The definition is \emph{event-relative}: sender-dependent or message-dependent predicates require the relevant event metadata during verification. A token-only verifier is justified only for the restricted class of Definition~\ref{def:token_witnessable}. A certificate licenses a transition under the specified policy; sole dependence on evidence content requires token invariance, and numerical operator compliance requires the enforcement regime below.

\subsection{Enforcement semantics, trust boundary, and operator compliance}
\label{sec:enforcement}

Certificates make belief revision \emph{auditable} as an admissibility claim.
A distinct trust boundary concerns \emph{operator compliance}: even with a valid witness, can an agent output an arbitrary next belief/decision while claiming it applied $U_{i,j^\star}$?
We therefore distinguish three enforcement architectures and state explicitly which results depend on which.

\paragraph{(A) Gate-only admissibility checking.}
A minimal \emph{output gate} verifies admissibility (e.g., $W\neq\emptyset$, $W\subseteq\mathcal{T}(E)$, and $E\models_W \varphi_{i,\ell}$) and then accepts or rejects an \emph{agent-supplied} update.
This supports checking admissibility of the certificate, but the numerical router-projection/normal-form results in Section~\ref{sec:routerNF} additionally require honest or verified operator execution. The gate does not, by itself, guarantee that an accepted numerical transition equals applying $U_{i,\ell}$ to the authoritative prior state unless agents are assumed honest with respect to operators.

\paragraph{(B) Proof-carrying updates (verifiable computation).}
To enforce operator compliance without executing $U_{i,j}$ inside the router, the agent can attach a proof that its output equals
$U_{i,j^\star}(b_i^t,E_i^t)$ (or equals the contractually specified decision rule), in the spirit of proof-carrying mechanisms \cite{necula1997pcc}.
We treat this as an optional deployment path; our logical development does not depend on any specific proof system.

\paragraph{(C) State-holding routers (router-executed dynamics; default).}
A practically common and conceptually clean architecture is a \emph{state-holding router} that maintains the authoritative protocol state and executes the contract update itself.
Concretely, the router stores $b_i^t$, validates tokens, checks triggers, extracts a witness, and then sets
\[
b_i^{t+1} :=
\begin{cases}
U_{i,j^\star}(b_i^t,E_i^t), & \text{if a trigger $j^\star$ fires with a nonempty witness},\\
\delta_i(b_i^t,E_i^t), & \text{otherwise},
\end{cases}
\]
while recording the corresponding certificate $\pi_i^t$.
Unless stated otherwise, all results that reason about belief trajectories $b_i^t$ are interpreted under this router-executed semantics.
Gate-only checking (A) supports certificate-admissibility statements. Audit equivalence as defined below includes numerical beliefs and therefore also requires operator compliance. Rejections record the proposed certificate and rejection reason separately from the executed fallback label.

\subsection{Protocol diagram}

Figure~\ref{fig:diagram} provides a round-level view of the \PBRC{} enforcement pipeline: upon receiving an event $E_i^t$, the system validates candidate evidence tokens, checks preregistered triggers, and records an admissibility certificate $\pi_i^t=(j^\star,W_i^t)$.

\begin{figure}[!ht]
\centering
\resizebox{\linewidth}{!}{
\begin{tikzpicture}[
  font=\footnotesize,
  node distance=8mm and 18mm,
  artifact/.style={
    draw, rounded corners=2pt, align=center, inner sep=4pt,
    text width=36mm, minimum height=11mm
  },
  process/.style={
    artifact,
    fill=black!5
  },
  outcome/.style={
    artifact,
    text width=46mm,
    fill=black!3
  },
  arrow/.style={->, thick},
  config/.style={->, thick, dashed},
  lab/.style={midway, fill=white, inner sep=1pt, font=\scriptsize}
]

\node[artifact] (contract) {Preregistered\\contract $C_i$};
\node[artifact, right=20mm of contract] (event) {Incoming event\\$E_i^t$};

\node[process, below=5mm of event] (validate) {Validate tokens\\$\mathcal{T}(E_i^t)$};
\node[process, below=5mm of validate] (check) {Check triggers \& extract witness\\select $j^\star$ and $W_i^t$};

\node[outcome, below left=7mm and 10mm of check] (fallback)
{Fallback\\$b_i^{t+1}=\delta_i(b_i^t,E_i^t)$\\$\pi_i^t=(\bot,\emptyset)$};

\node[outcome, below right=7mm and 10mm of check] (update)
{Evidence-triggered update\\$b_i^{t+1}=U_{i,j^\star}(b_i^t,E_i^t)$\\$\pi_i^t=(j^\star,W_i^t)$};

\node[artifact, below=20mm of check, text width=54mm] (log)
{Append-only audit log\\inputs, decision, $(b_i^{t+1},\pi_i^t)$};

\draw[arrow] (event) -- (validate);
\draw[arrow] (validate) -- (check);
\draw[config] (contract) |- node[lab, near start]{public spec} (check.west);
\draw[arrow] (check) -- node[lab]{if rejected} (fallback);
\draw[arrow] (check) -- node[lab]{if verified} (update);
\draw[arrow] (fallback) |- (log.west);
\draw[arrow] (update) |- (log.east);

\node[
  draw, dashed, rounded corners=3pt,
  fit=(event)(validate)(check)(fallback)(update)(log),
  inner sep=6pt,
  label={[font=\scriptsize]above:Router-side enforcement (state-holding)}
] {};

\end{tikzpicture} }
\caption{\PBRC{} enforcement pipeline (state-holding router): tokens are validated, triggers are checked, and a non-fallback update is executed only when the selected trigger admits a \emph{nonempty} witness set; full-event selection and witness sufficiency are verified; otherwise the step devolves to fallback. The log includes replay inputs, proposed and accepted certificates, and the resulting state.}
\label{fig:diagram}
\end{figure}

\begin{figure}[t]
\centering
\fbox{\parbox{0.96\linewidth}{
\small
\textbf{PBRC router step (agent $i$ at round $t$).}\\[2pt]
\textbf{Input:} contract $C_i$, current state $b_i^t$, incoming event $E_i^t$ (multiset of messages).\\
\textbf{Output:} enforced next state $b_i^{t+1}$ and certificate $\pi_i^t=(\ell,W)$.
\begin{enumerate}
\item Extract candidate tokens $\mathcal{T}_{\mathrm{all}}(E_i^t)=\{\tau:\exists m\in E_i^t\ \mathrm{HasTok}(m,\tau)\}$.
\item Validate tokens (router-side): $\mathcal{T}(E_i^t):=\{\tau\in\mathcal{T}_{\mathrm{all}}(E_i^t):\mathrm{Valid}(\tau)\}$.
\item Compute trigger set $J:=\{j: E_i^t\models \varphi_{i,j}\}$ (equivalently, $\mathcal{S}(E_i^t)\models \varphi_{i,j}$).
\item If $J=\emptyset$, set $b_i^{t+1}:=\delta_i(b_i^t,E_i^t)$ and output $\pi_i^t:=(\bot,\emptyset)$.
\item Else let $j^\star=\min_{\prec_i}J$ and compute witness $W:=\mathrm{Wit}_{\varphi_{i,j^\star}}(E_i^t)\subseteq \mathcal{T}(E_i^t)$.
      Verify $W\ne\emptyset$, $W\subseteq\mathcal T(E_i^t)$, and $E_i^t\models_W\varphi_{i,j^\star}$ using the required payload. If verification fails or an execution budget/definedness check fails, execute fallback and record rejection. Otherwise set $b_i^{t+1}:=U_{i,j^\star}(b_i^t,E_i^t)$ and output $\pi_i^t:=(j^\star,W)$.
\item Record the event or sufficient annotated transition inputs, public context, contract/validator versions, proposal, decision/reason, proof payload, and executed state, including fallback rounds.
\end{enumerate}
\textbf{Note:} In the gate-only architecture (A) the router may only check $\pi_i^t$ and accept/reject an agent-supplied output; in the state-holding architecture (C) the router executes the update above.
}}
\caption{PBRC enforcement pseudocode (router-executed semantics).}
\label{fig:pseudocode}
\end{figure}

\section{Running example: ``peer pressure'' misinformation triage}
\label{sec:running}

We use a minimal running example to illustrate three distinct phenomena that \PBRC{} separates by construction:
(i) purely social persuasion (rhetoric without validated evidence),
(ii) evidence-gated belief revision (admissible changes of mind justified by tokens and witnesses), and
(iii) the distinction between \emph{replay} (stale-but-valid evidence) and \emph{forgery} (manufactured ``valid'' evidence).

We will return to this example throughout the paper to ground the formal results. Specifically, we will show how the social-only guarantees (Section~\ref{sec:main}) prevent the cascade described in Section~\ref{sec:running_social}, how the normal form representation (Section~\ref{sec:normalform}) canonicalizes the evidence-gated revision in Section~\ref{sec:running_evidence}, and how the adversary models (Section~\ref{sec:adversary}) formalize the replay and forgery attacks in Section~\ref{sec:running_replay}.

\subsection{Setup}
\label{sec:running_setup}

Let $\Hh=\{\textsf{True},\textsf{False}\}$ encode whether a concrete news claim is true or false.
Agent $i$ maintains a belief distribution $b_i^t\in\Prob(\Hh)$.
Suppose the initial belief is moderately confident in truth:
\[
b_i^0(\textsf{True})=0.6,\qquad b_i^0(\textsf{False})=0.4.
\]

Tokens represent externally checkable evidence artifacts (e.g., signed tool outputs or retrieval receipts) and are validated by the system's validity layer.
For an event $E$ (a multiset of received messages), let $\mathcal{T}(E)$ denote its set of valid tokens; $\mathrm{SocialOnly}(E)$ is equivalent to $\mathcal{T}(E)=\emptyset$.

Agent $i$ preregisters a simple evidential contract with two prioritized triggers:
\begin{enumerate}
\item \emph{Verified falsifier} (high priority). There exists a valid token contradicting \textsf{True}:
\[
\varphi_1 \;:=\; \exists \tau\ \big(\mathrm{Valid}(\tau)\wedge \mathrm{Contradicts}(\tau,\textsf{True})\big).
\]
\item \emph{Verified support}. There exists a valid token supporting \textsf{True}:
\[
\varphi_2 \;:=\; \exists \tau\ \big(\mathrm{Valid}(\tau)\wedge \mathrm{Supports}(\tau,\textsf{True})\big).
\]
\end{enumerate}
The priority order is $\varphi_1 \prec \varphi_2$.
When no trigger is admissibly fired, the fallback is skeptical dilution $\delta_\lambda$ (Definition~\ref{def:skeptical_dilution}).

For a complete numerical example take $\lambda=0.1$. Writing $p=b(\textsf{True})$, define $U_1$ by $p'=p/(p+4(1-p))$ and $U_2$ by $p'=4p/(4p+1-p)$, with complementary mass on \textsf{False}. These are fixed bounded log-odds changes on interior beliefs and well-defined at the endpoints. Assume a sound and complete state-holding router verifies the selected witness and computes these updates (Section~\ref{sec:routerNF}).

\subsection{Outcome under purely social persuasion (token-empty round)}
\label{sec:running_social}

At round $t=0$, a peer sends a confident message such as
\emph{``Everyone agrees it is \textsf{False}.''}
Crucially, the message contains no validated citations or tool outputs; formally, the received event $E_i^0$ is social-only:
\[
\mathrm{SocialOnly}(E_i^0)\quad\Longleftrightarrow\quad \mathcal{T}(E_i^0)=\emptyset.
\]
Because the contract is evidential, neither $\varphi_1$ nor $\varphi_2$ can hold on a token-empty event, hence $J_i^0=\emptyset$ and the enforced update is necessarily fallback:
\[
b_i^1 \;=\; \delta_\lambda(b_i^0,E_i^0).
\]
By Theorem~\ref{thm:nonamp}, confidence cannot increase in such a round, and by Lemma~\ref{lem:dilutionprops} the top hypothesis is preserved.
In particular, the peer's rhetoric cannot flip the decision from \textsf{True} to \textsf{False} without validated evidence, even if the rhetoric is unanimous, prestigious, or highly confident.

\subsection{Outcome with a valid evidence token (auditable flip)}
\label{sec:running_evidence}

Suppose at a later round a peer provides a validated citation token $\tau^\star$ (e.g., a retrieval-verified snippet or a signed fact-check result) such that
\[
\mathrm{Valid}(\tau^\star)\quad\text{and}\quad \mathrm{Contradicts}(\tau^\star,\textsf{True}).
\]
Then $E_i^1$ is non-social and satisfies $\varphi_1$, so the highest-priority trigger fires.
Under enforcement, the update is accepted only if it is justified by a nonempty witness set; here a natural witness is $W_i^1=\{\tau^\star\}$.
After the first fallback, $b_i^1(\textsf{True})=0.59$. Applying $U_1$ gives $b_i^2(\textsf{True})=0.59/(0.59+4\cdot0.41)\approx0.265$, so the top changes to \textsf{False}. The certificate records the selected license (Theorem~\ref{thm:accountability}); correctness of the new answer still depends on the evidence and operator assumptions.

\begin{table}[t]
\centering \small
\begin{tabularx}{\linewidth}{@{}c l l l >{\raggedright\arraybackslash}X@{}}
\toprule
Round $t$ & Event description & $\mathcal{T}(E_i^t)$ & Canonical Cert. & Enforced effect on belief\\
\midrule
$0$ & social-only persuasion & $\emptyset$ & $(\bot,\emptyset)$ &
fallback $b_i^1=\delta_\lambda(b_i^0,E_i^0)$; confidence $\downarrow$, $\argmax$ unchanged\\
$1$ & validated evidence arrives & $\{\tau^\star\}$ & $(1,\{\tau^\star\})$ &
$\varphi_1$ fires; admissible (and auditable) revision toward \textsf{False}\\
\bottomrule
\end{tabularx}
\caption{Running example trace under empty-witness rejection.
Token-empty persuasion canonicalizes to fallback, whereas an evidence-triggered step carries a nonempty witness set that an auditor can verify.}
\label{tab:running_trace}
\end{table}

\subsection{Replay versus forgery (preview)}
\label{sec:running_replay}

This example also clarifies two orthogonal integrity threats.

\emph{Replay.} A peer may retransmit an \emph{old} token $\tau^\star$ that was once valid but is no longer contextually applicable (e.g., wrong claim instance, stale timestamp, or mismatched retrieval context).
If $\mathrm{Valid}(\tau)$ does not bind tokens to an instance identifier and a freshness condition, then replayed evidence can satisfy $\varphi_1$ despite being irrelevant.
This is addressed by freshness/context predicates inside triggers (Section~\ref{sec:adversary}), e.g., requiring $\mathrm{Fresh}_\Delta$ and/or an instance hash match as part of admissibility.

\emph{Forgery.} If an adversary can fabricate tokens that pass $\mathrm{Valid}$ (e.g., by compromising validators or the verification channel), then it can manufacture admissible ``evidence'' for arbitrary conclusions.
If the forged certificate passes all full-event and priority checks, is accepted, and its selected operator produces an adverse output at the current belief, that output occurs (Theorem~\ref{thm:imposs_forge}). Forgeability alone does not imply that any chosen contract can be forced to flip. Integrity, applicability, operator design, and faithful execution all remain relevant.

\section{Main guarantees (theorem-centric)}
\label{sec:main}

This section establishes the core \emph{social-dynamics} guarantees of \PBRC{}.
The central theme is that, under evidential contracts and conservative fallback, \emph{purely social interaction} cannot (i) increase confidence or (ii) induce a new high-confidence wrong consensus.

To orient the reader, we proceed in four steps. First, we define the necessary properties of fallback operators (Section~\ref{sec:fallback_ops}). Second, we prove that in token-empty (social-only) rounds, confidence cannot increase (Section~\ref{sec:social_non_amp}). Third, we establish the central impossibility result: purely conformity-driven cascades cannot occur under \PBRC{} (Section~\ref{sec:impossibility_cascades}). Finally, we provide a minimal topology statement showing why identical admitted evidence yields identical beliefs under an evidence-additive policy (Section~\ref{sec:minimal_topology}). These formal results directly explain the cascade suppression observed in the misinformation triage example (Section~\ref{sec:running}).

\begin{table}[t]
\centering
\small
\renewcommand{\arraystretch}{1.2}
\begin{tabularx}{\linewidth}{@{}l l >{\raggedright\arraybackslash}X@{}}
\toprule
\textbf{Result} & \textbf{Enforcement regime} & \textbf{Key assumptions (informal)}\\
\midrule
Thm.~\ref{thm:nonamp} (token-empty non-amplification) & (B) or (C) & evidential triggers; non-amplifying fallback\\
Thm.~\ref{thm:nocascade} (no wrong-but-sure social cascades) & (B) or (C) & evidential; argmax-preserving \& non-amplifying fallback; reject empty witnesses\\
Thm.~\ref{thm:normalform} (representation/compilation) & (C) or honest & token-witnessable triggers; token-sound witnesses; certificate canonicalization\\
Thm.~\ref{thm:routernf} (router normal form) & (B), (C), or honest & nonempty-witness gate; witness non-vacuity on non-social satisfaction\\
Thm.~\ref{thm:tokensuff} (token-sufficiency) & (B) or (C) & token-invariant contract; router soundness \& token-determinedness\\
Thm.~\ref{thm:flood_distance}/\ref{thm:flood_optimal} (topology/diameter) & any & flooding dissemination; no exogenous token births; synchronous rounds\\
\bottomrule
\end{tabularx}
\vspace{0.5em}
\caption{Assumption matrix for representative core results. This table maps each major theorem to the required enforcement regime and the necessary structural assumptions on the contract or environment. Enforcement regimes are: (A) gate-only admissibility checking, (B) proof-carrying updates, (C) state-holding router execution (Section~\ref{sec:enforcement}).}
\label{tab:assumptions}
\end{table}

\subsection{Fallback operators and skeptical dilution}
\label{sec:fallback_ops}

\begin{Definition}[Argmax-preserving and non-amplifying fallback]
\label{def:fallback_props}
Fix a confidence functional $\conf:\Prob(\Hh)\to\R$ (in the main text we take $\conf(b)=\max_{h\in\Hh} b(h)$).
A fallback operator $\delta_i:\Prob(\Hh)\times\Ev\to\Prob(\Hh)$ is:
\begin{enumerate}
\item \emph{argmax-preserving} if $\argmax(\delta_i(b,E))=\argmax(b)$ for all $b,E$;
\item \emph{non-amplifying} if $\conf(\delta_i(b,E))\le \conf(b)$ for all $b,E$.
\end{enumerate}
\end{Definition}

\begin{Definition}[Skeptical dilution]
\label{def:skeptical_dilution}
Fix $\lambda\in[0,1)$ and let $u$ denote the uniform distribution on $\Hh$.
Define the skeptical-dilution fallback
\[
\delta_\lambda(b,E)\ :=\ (1-\lambda)b+\lambda u.
\]
\end{Definition}

\begin{Lemma}[Closed form under repeated skeptical dilution]
\label{lem:closedform}
Let $b^{0}\in\Prob(\Hh)$ and define $b^{t+1}=\delta_\lambda(b^t,E^t)$ for an arbitrary sequence of events $(E^t)_{t\ge 0}$.
Then for all $t\ge 0$,
\[
b^{t}=(1-\lambda)^t\, b^{0} + \bigl(1-(1-\lambda)^t\bigr)\,u,
\]
and, writing $\TV(p,q)=\tfrac12\|p-q\|_1$ for total variation distance,
\[
\TV(b^{t},u)=(1-\lambda)^t\,\TV(b^{0},u).
\]
\end{Lemma}

\begin{proof}
By Definition~\ref{def:skeptical_dilution},
\[
b^{t+1}=(1-\lambda)b^t+\lambda u.
\]
Iterating the affine recursion yields
$b^t=(1-\lambda)^t b^0+\bigl(1-(1-\lambda)^t\bigr)u$.
Subtracting $u$ gives $b^t-u=(1-\lambda)^t(b^0-u)$, and taking $\TV(\cdot,u)=\tfrac12\|\cdot-u\|_1$ pulls out the constant factor $(1-\lambda)^t$.
\end{proof}

\begin{Lemma}[Skeptical dilution preserves argmax and does not amplify confidence]
\label{lem:dilutionprops}
For any $\lambda\in[0,1)$, $\delta_\lambda$ is argmax-preserving and non-amplifying in the sense of Definition~\ref{def:fallback_props}.
Moreover, if $\lambda>0$ and $b$ is not uniform (equivalently, $\conf(b)>1/|\Hh|$), then $\conf(\delta_\lambda(b,E))<\conf(b)$.
\end{Lemma}

\begin{proof}
For any $h,h'\in\Hh$,
\[
\delta_\lambda(b,E)(h)-\delta_\lambda(b,E)(h')=(1-\lambda)\bigl(b(h)-b(h')\bigr),
\]
so strict inequalities and equalities among coordinates are preserved; hence the entire $\argmax$ set is preserved.

For confidence, let $m=\conf(b)=\max_h b(h)$. Then
\[
\conf(\delta_\lambda(b,E))
=\max_h\Bigl((1-\lambda)b(h)+\lambda/|\Hh|\Bigr)
=(1-\lambda)m+\lambda/|\Hh|
\le m,
\]
since $m\ge 1/|\Hh|$ for any distribution. Strictness holds exactly when $\lambda>0$ and $m>1/|\Hh|$. At $\lambda=0$ the operator is the identity. The proofs use exact arithmetic; finite precision needs an explicit tie policy, and $\lambda=1$ is excluded because it can introduce new maximizers.
\end{proof}

\subsection{Token-empty rounds and social non-amplification}
\label{sec:social_non_amp}

Recall that $\mathrm{SocialOnly}(E)$ means $E$ contains no valid tokens, i.e.\ $\mathcal{T}(E)=\emptyset$.

\begin{Definition}[Trigger-inactive round]
\label{def:trigger_inactive}
Round $t$ is \emph{trigger-inactive for agent $i$} if $J_i^t=\emptyset$, i.e., no preregistered trigger is satisfied at that step and the intended contract update is fallback.
\end{Definition}

\begin{Remark}[Token-empty implies trigger-inactive under evidential contracts]
If agent $i$ uses an \emph{evidential} contract (every trigger implies $\neg\mathrm{SocialOnly}$), then $\mathrm{SocialOnly}(E_i^t)$ implies $J_i^t=\emptyset$ at round $t$.
Under enforcement regimes (B) or (C) with empty-witness rejection, such token-empty rounds are therefore necessarily enforced as fallback steps.
\end{Remark}

\begin{Theorem}[Social non-amplification]
\label{thm:nonamp}
If $\delta_i$ is non-amplifying, then for any agent $i$ and any trigger-inactive round $t$ (Definition~\ref{def:trigger_inactive}),
\[
\conf(b_i^{t+1})\le \conf(b_i^t).
\]
If, moreover, $\delta_i=\delta_\lambda$ with $0<\lambda<1$ and $b_i^t$ is non-uniform, then the inequality is strict.
\end{Theorem}

\begin{proof}
If round $t$ is trigger-inactive for agent $i$, then by the one-step update semantics (Equation~\eqref{eq:update}) we have
$b_i^{t+1}=\delta_i(b_i^t,E_i^t)$.
Non-amplification of $\delta_i$ gives $\conf(b_i^{t+1})\le \conf(b_i^t)$.
If $\delta_i=\delta_\lambda$ with $0<\lambda<1$ and $b_i^t$ is non-uniform, strictness follows from Lemma~\ref{lem:dilutionprops}.
\end{proof}

\begin{Corollary}[Convergence to maximal uncertainty under repeated trigger-inactive rounds]
\label{cor:converge}
Assume $C_i$ is evidential and $\delta_i=\delta_\lambda$ with $0<\lambda<1$.
If rounds $0,1,\dots,T-1$ are trigger-inactive for agent $i$, then
\[
b_i^T=(1-\lambda)^T b_i^0 + \bigl(1-(1-\lambda)^T\bigr)u,
\qquad
\TV(b_i^T,u)=(1-\lambda)^T\TV(b_i^0,u).
\]
\end{Corollary}

\begin{proof}
Under the hypothesis, $b_i^{t+1}=\delta_\lambda(b_i^t,E_i^t)$ for all $t<T$. Apply Lemma~\ref{lem:closedform}. Since $0<1-\lambda<1$, the displayed distance tends to zero if the inactive interval continues indefinitely.
\end{proof}

\subsection{Impossibility of purely conformity-driven wrong-but-sure cascades}
\label{sec:impossibility_cascades}

The most dangerous failure mode in deliberative systems is a ``wrong-but-sure'' cascade: the population converges to a high-confidence consensus on an incorrect hypothesis, driven entirely by mutual reinforcement rather than evidence. The running example shows how the contract blocks a token-empty top change; the pooling baseline in Section~\ref{sec:sim_cascades} illustrates population amplification. We analyze fallback-only intervals, a class that includes token-empty intervals under evidential enforcement but may also contain rejected evidence-bearing steps.

Fix a designated true hypothesis $h^\star\in\Hh$ and a confidence threshold $\gamma\in(1/|\Hh|,1]$.

\begin{Definition}[Social-driven wrong-but-sure cascade]
\label{def:social_cascade}
A \emph{social-driven wrong-but-sure cascade} to a false hypothesis $h\neq h^\star$ occurs over an interval $[t_0,t_1]$ if:
\begin{enumerate}
\item (\emph{fallback-only interval}) for every agent $i$ and every round $t\in\{t_0,\dots,t_1-1\}$, the executed update is fallback; and
\item (\emph{wrong but sure at the end}) for every agent $i$, $\argmax(b_i^{t_1})=\{h\}$ and $\conf(b_i^{t_1})\ge \gamma$.
\end{enumerate}
\end{Definition}

\begin{Theorem}[No social-driven cascades under argmax-preserving fallback]
\label{thm:nocascade}
Assume each agent $i$ uses a fallback $\delta_i$ that is argmax-preserving and non-amplifying (Definition~\ref{def:fallback_props}).
If rounds $t_0,\dots,t_1-1$ are trigger-inactive for all agents, then for every agent $i$,
\[
\argmax(b_i^{t_1})=\argmax(b_i^{t_0}),
\qquad
\conf(b_i^{t_1})\le \conf(b_i^{t_0}).
\]
Consequently, no social-driven wrong-but-sure cascade (Definition~\ref{def:social_cascade}) can occur unless it already held at time $t_0$.
\end{Theorem}

\begin{proof}
Fix an agent $i$. Since rounds $t_0,\dots,t_1-1$ are trigger-inactive for $i$, we have
$b_i^{t+1}=\delta_i(b_i^t,E_i^t)$ for all $t_0\le t<t_1$.
Argmax-preservation yields $\argmax(b_i^{t+1})=\argmax(b_i^t)$ at each step, hence $\argmax(b_i^{t_1})=\argmax(b_i^{t_0})$ by iteration.
Non-amplification yields $\conf(b_i^{t+1})\le \conf(b_i^t)$ at each step, hence $\conf(b_i^{t_1})\le \conf(b_i^{t_0})$ by iteration.
\end{proof}

\subsection{A minimal topology statement under evidence closure}
\label{sec:minimal_topology}

The detailed topology-to-trace factorization results appear later (Sections~\ref{sec:normalform}--\ref{sec:topologytrace}). For completeness, we record here a simple ``closure implies topology-irrelevance'' principle for policies whose epistemic state is a function of admitted evidence sets.

\begin{Definition}[Evidence closure]
\label{def:evidence_closure}
Let $\mathcal{K}^{<T}$ denote the set of all valid tokens that appear anywhere in the system before time $T$:
\[
\mathcal{K}^{<T}\ :=\ \bigcup_{i\in\A}\ \bigcup_{t<T}\ \mathcal{T}(E_i^t).
\]
A run achieves \emph{evidence closure by time $T$} if every agent has access to every such token at time $T$:
\[
\forall i\in\A,\quad \mathcal{K}^{<T}\ \subseteq\ \mathcal{K}_i^{T},
\]
where $\mathcal{K}_i^{T}$ denotes the set of valid tokens available to agent $i$ after processing rounds $0,\dots,T-1$
(e.g., $\mathcal{K}_i^{T}=\bigcup_{t<T}\mathcal{T}(E_i^t)$, or with an explicit initial set if present).
\end{Definition}

\begin{Definition}[Evidence-additive revision]
\label{def:evidence_additive}
A revision policy for agent $i$ is \emph{evidence-additive} if there exists a function $F_i$ and an ``admission'' map $\mathrm{Admit}$ such that
\[
b_i^t = F_i(\mathcal{S}_i^t),
\qquad
\mathcal{S}_i^{t+1}=\mathcal{S}_i^t\cup \mathrm{Admit}(C_i,E_i^t),
\]
where $\mathcal{S}_i^t$ is the cumulative set of admitted tokens up to time $t$, and
$\mathrm{Admit}(C_i,E_i^t)\subseteq \mathcal{T}(E_i^t)$ is a (possibly empty) witness set used to justify admissible revision at round $t$.
\end{Definition}

\begin{Definition}[Sender-invariant triggers]
\label{def:sender_invariant}
Fix a receiver $i$. A trigger $\chi$ is \emph{sender-invariant} if its satisfaction depends only on token-level facts available at $i$
(e.g., $\mathrm{Valid}$, $\mathrm{Supports}$, $\mathrm{Contradicts}$, token types, and freshness predicates),
and not on the identities of message senders \emph{except insofar as those identities are explicitly bound into tokens}.
Equivalently, $\chi$ is definable in the token-only signature obtained by removing $\mathrm{Send}$ and other non-token message metadata.
\end{Definition}

\begin{Theorem}[Topology irrelevance given identical admitted evidence sets]
\label{thm:topology}
Assume agent $i$ uses evidence-additive revision (Definition~\ref{def:evidence_additive}) and all triggers are sender-invariant (Definition~\ref{def:sender_invariant}).
Consider two runs (possibly on different topologies and with different message orderings/rhetoric) that induce the same admitted evidence set at time $T$, i.e.\ $\mathcal{S}_{i}^{T}$ is identical in both runs.
Then the terminal belief is identical: $b_i^{T}$ is the same in both runs.
In particular, under evidence closure (Definition~\ref{def:evidence_closure}), any topology/order/rhetoric variation that does not change $\mathcal{S}_{i}^{T}$ cannot change $b_i^{T}$.
\end{Theorem}

\begin{proof}
By evidence-additivity, $b_i^{T}=F_i(\mathcal{S}_i^{T})$. If $\mathcal{S}_i^{T}$ is the same across the two runs, then the right-hand side is the same, hence $b_i^{T}$ is the same.
Sender-invariance is included to emphasize that admissibility judgments can be designed to depend only on token-level evidence rather than social metadata; the stronger token-sufficiency and topology-to-trace results later make this dependence precise.
\end{proof}


\section{Minimality and necessity results}
\label{sec:minimality}

This section formalizes a class of \emph{auditable trigger protocols} and characterizes their behavior under evidence-free (social-only) interaction. Fallback constraints apply only where fallback can be selected; trigger constraints apply to the actual selected transition. The results are intentionally elementary: they are ``one-step'' necessity statements that isolate which part of a protocol must carry the burden of social robustness.

We structure this analysis as follows. First, we define the general class of auditable trigger protocols (Section~\ref{sec:auditable_trigger_protocols}). Second, we formalize evidence-free robustness properties (Section~\ref{sec:evidence_free_robustness}). Third, we prove that fallback operators must be conservative to satisfy these properties (Section~\ref{sec:fallback_minimality}). Finally, we characterize robustness through every selected social-only transition and show why syntactic evidence gating is sufficient but not necessary (Section~\ref{sec:necessity_evidence_gating}).

\subsection{Auditable trigger protocols}
\label{sec:auditable_trigger_protocols}

\begin{Definition}[Auditable trigger protocol]
\label{def:audit_trigger_protocol}
Fix an agent $i$ and hypothesis space $\Hh$.
An \emph{auditable trigger protocol} for agent $i$ is a tuple
\[
P_i=\langle (\psi_{i,1},V_{i,1}),\dots,(\psi_{i,r_i},V_{i,r_i}),\prec_i,\delta_i\rangle
\]
where:
(i) each $\psi_{i,j}$ is a public predicate over events $E$ (a \emph{trigger});
(ii) each $V_{i,j}:\Prob(\Hh)\times \Ev \to \Prob(\Hh)$ is an update operator;
(iii) $\prec_i$ is a strict total priority order on $\{1,\dots,r_i\}$;
(iv) $\delta_i:\Prob(\Hh)\times \Ev \to \Prob(\Hh)$ is a fallback operator.
On input $(b,E)$ the protocol selects the highest-priority satisfied trigger (if any) and applies its operator; otherwise it applies fallback:
\[
F_{P_i}(b,E)=
\begin{cases}
V_{i,j^\star}(b,E), & J_{P_i}(E)\neq\emptyset,\\
\delta_i(b,E), & J_{P_i}(E)=\emptyset,
\end{cases}
\qquad
j^\star=\min_{\prec_i} J_{P_i}(E),
\]
where $J_{P_i}(E):=\{j: E\models \psi_{i,j}\}$.
\end{Definition}

\begin{Remark}[When necessity statements are non-vacuous]
\label{rem:nonvacuous_fallback}
Theorems~\ref{thm:necessity_argmax} and \ref{thm:necessity_nonamp} only constrain $\delta_i$ on social-only inputs that can actually reach fallback, i.e., on social-only events $E$ such that $J_{P_i}(E)=\emptyset$ for the relevant run. We state the conditions explicitly below to avoid vacuity.
\end{Remark}

\subsection{Evidence-free robustness properties}
\label{sec:evidence_free_robustness}

Recall that $\mathrm{SocialOnly}(E)$ means $E$ contains no valid evidence tokens (equivalently, $\mathcal{T}(E)=\emptyset$).

\begin{Definition}[Evidence-free peer-pressure immunity]
\label{def:peer_pressure_immunity}
A protocol $P_i$ has \emph{evidence-free peer-pressure immunity} if for every initial belief $b_i^0\in\Prob(\Hh)$ and every event stream $(E^t)_{t\ge 0}$ with $\mathrm{SocialOnly}(E^t)$ for all $t$, the induced belief trajectory $(b_i^t)_{t\ge 0}$ satisfies
\[
\argmax(b_i^t)=\argmax(b_i^0)\quad\text{for all }t\ge 0.
\]
\end{Definition}

\begin{Definition}[Evidence-free non-amplification]
\label{def:nonamplification}
Fix a confidence functional $\conf:\Prob(\Hh)\to \R$ (e.g., $\conf(b)=\max_{h\in\Hh} b(h)$).
A protocol $P_i$ has \emph{evidence-free non-amplification} if for every initial belief $b_i^0\in\Prob(\Hh)$ and every event stream $(E^t)_{t\ge 0}$ with $\mathrm{SocialOnly}(E^t)$ for all $t$, the induced belief trajectory satisfies
\[
\conf(b_i^{t+1})\le \conf(b_i^t)\quad\text{for all }t\ge 0.
\]
\end{Definition}

\subsection{Fallback minimality}
\label{sec:fallback_minimality}

\begin{Theorem}[Argmax-preserving fallback is necessary]
\label{thm:necessity_argmax}
Let $P_i$ be an auditable trigger protocol with evidence-free peer-pressure immunity (Definition~\ref{def:peer_pressure_immunity}).
Assume there exists a social-only event $E$ such that $J_{P_i}(E)=\emptyset$ (so fallback can be reached on a social-only input).
Then $\delta_i$ must be argmax-preserving on every social-only $(b,E)$ with $J_{P_i}(E)=\emptyset$:
\[
\forall b\in\Prob(\Hh)\ \forall E,\quad
\big(\mathrm{SocialOnly}(E)\ \land\ J_{P_i}(E)=\emptyset\big)
\Rightarrow
\argmax(\delta_i(b,E))=\argmax(b).
\]
\end{Theorem}

\begin{proof}
Suppose for contradiction that there exist $b$ and a social-only event $E$ with $J_{P_i}(E)=\emptyset$ such that
$\argmax(\delta_i(b,E))\neq \argmax(b)$.
Consider the length-one social-only run with initial belief $b_i^0=b$ and first event $E^0=E$.
Because $J_{P_i}(E^0)=\emptyset$, the protocol applies fallback at $t=0$, so
$b_i^1=\delta_i(b_i^0,E^0)=\delta_i(b,E)$.
By construction $\argmax(b_i^1)\neq \argmax(b_i^0)$, contradicting evidence-free peer-pressure immunity.
\end{proof}

\begin{Theorem}[Non-amplifying fallback is necessary]
\label{thm:necessity_nonamp}
Let $P_i$ be an auditable trigger protocol with evidence-free non-amplification (Definition~\ref{def:nonamplification}).
Assume there exists a social-only event $E$ such that $J_{P_i}(E)=\emptyset$.
Then $\delta_i$ must be non-amplifying on every social-only $(b,E)$ with $J_{P_i}(E)=\emptyset$:
\[
\forall b\in\Prob(\Hh)\ \forall E,\quad
\big(\mathrm{SocialOnly}(E)\ \land\ J_{P_i}(E)=\emptyset\big)
\Rightarrow
\conf(\delta_i(b,E))\le \conf(b).
\]
\end{Theorem}

\begin{proof}
Suppose for contradiction that there exist $b$ and a social-only event $E$ with $J_{P_i}(E)=\emptyset$ such that
$\conf(\delta_i(b,E))>\conf(b)$.
Consider the length-one social-only run with $b_i^0=b$ and $E^0=E$.
As above, fallback applies at $t=0$, so $b_i^1=\delta_i(b,E)$, hence
$\conf(b_i^1)>\conf(b_i^0)$, contradicting evidence-free non-amplification.
\end{proof}

\subsection{Local characterization and the role of evidence gating}
\label{sec:necessity_evidence_gating}

\begin{Definition}[Nontrivial revision operator]
\label{def:nontrivial_revision}
An operator $V$ is nontrivial if for some $(b,E)$ it changes the top set or the confidence functional. This existential property makes no claim about its effect on a particular social-only input or about whether its trigger is ever selected.
\end{Definition}

\begin{Theorem}[Local characterization of social-only robustness]
\label{thm:necessity_evidence}
Let $F_{P_i}$ be the actually selected deterministic transition of a trigger protocol, including fallback. Evidence-free peer-pressure immunity for all initial beliefs holds if and only if
\[
\forall b,E:\quad\mathrm{SocialOnly}(E)\ \Longrightarrow\
\argmax(F_{P_i}(b,E))=\argmax(b).
\]
Evidence-free nonamplification holds if and only if
\[
\forall b,E:\quad\mathrm{SocialOnly}(E)\ \Longrightarrow\
\conf(F_{P_i}(b,E))\le\conf(b).
\]
\end{Theorem}
\begin{proof}
For necessity choose the one-step social-only run initialized at the specified $b$ and receiving $E$. The desired run property implies the local condition. For sufficiency apply the local condition at each step of an arbitrary social-only stream; induction preserves the initial top set, and the confidence inequalities hold consecutively.
\end{proof}

\begin{example}[Why evidence gating is not a necessary syntactic condition]
A sole always-true trigger with operator $V(b,E)=(1-\lambda)b+\lambda u$, $0<\lambda<1$, preserves the full top set and never increases maximum confidence, including on token-empty events. It is nevertheless nontrivial: for $b=(0.8,0.2)$ and $\lambda=0.5$, it returns $(0.65,0.35)$. Thus robust protocols can have nontrivial triggers satisfied by social-only events. A lower-priority trigger that is never selected is likewise unconstrained by an immunity property of realized transitions.
\end{example}

\begin{Remark}[Stability, confidence, and calibration]
Theorems~\ref{thm:necessity_argmax}--\ref{thm:necessity_evidence} constrain actual selected transitions on social-only inputs. Evidence gating plus conservative fallback is one sufficient implementation of these conditions. Maximum-confidence nonamplification is not statistical calibration, and does not establish improved accuracy or a better proper scoring rule.
\end{Remark}

\section{PBRC normal form, equivalence notions, and representation}
\label{sec:normalform}

The local characterization in Section~\ref{sec:minimality} constrains behavior, not a unique syntax. We now distinguish two transformations. A behavior-preserving normal form moves the original social-only behavior into fallback; it preserves trajectories but need not make an unsafe protocol safe. Evidence-gated enforcement instead retains the original fallback and can suppress social-only triggered behavior. The results below assume the specified witness extractors exist and that original trigger predicates are expressible in the target contract language. For FO contracts, the original triggers are therefore assumed FO-definable.

We compare numerical trajectories and projected certificate traces. Canonicalizing an empty witness to the fallback label is a mathematical projection of the audit record, not an instruction to delete an agent's original proposal or to certify that an arbitrary ungated social transition was conservative. Complete logs retain proposals, decisions, and actual updates.

\subsection{One-step semantics of trigger protocols}
\label{sec:one_step_semantics}

Fix an agent $i$ and an auditable trigger protocol
\[
P_i=\langle (\psi_{i,1},V_{i,1}),\dots,(\psi_{i,r_i},V_{i,r_i}),\prec_i,\delta_i\rangle,
\]
where each $\psi_{i,j}$ is an event predicate (trigger), each $V_{i,j}$ is an update operator, $\delta_i$ is the fallback, and $\prec_i$ is a strict total priority order.\footnote{For determinism we assume $\prec_i$ is a strict \emph{total} order on $\{1,\dots,r_i\}$. If priorities are specified only partially, fix any public tie-breaking refinement to a strict total order.}

Given a belief state $b\in\Prob(\Hh)$ and an event $E$, define the set of satisfied triggers:
\[
J_{P_i}(E):=\{j\in\{1,\dots,r_i\}: E\models \psi_{i,j}\},
\qquad
j^\star_{P_i}(E):=
\begin{cases}
\min_{\prec_i} J_{P_i}(E), & J_{P_i}(E)\neq\emptyset,\\
\text{undefined}, & J_{P_i}(E)=\emptyset.
\end{cases}
\]
The one-step belief update induced by $P_i$ is
\begin{equation}
\label{eq:Pi_update}
F_{P_i}(b,E):=
\begin{cases}
V_{i,j^\star_{P_i}(E)}(b,E), & J_{P_i}(E)\neq\emptyset,\\
\delta_i(b,E), & J_{P_i}(E)=\emptyset.
\end{cases}
\end{equation}

\subsection{Two notions of equivalence: belief trajectories vs.\ audit traces}
\label{sec:equivalence_notions}

\paragraph{Validated tokens and social-only events.}
For any event $E$ (multiset of messages), define the set of \emph{validated} tokens contained in $E$:
\[
\mathcal{T}(E) \;:=\; \{\tau\in\Toks : \exists m\in E\ (\mathrm{HasTok}(m,\tau)\wedge \mathrm{Valid}(\tau))\}.
\]
By definition, $\mathrm{SocialOnly}(E)$ holds iff $\mathcal{T}(E)=\emptyset$.

\paragraph{Witness extraction and token-witnessability.}
To speak meaningfully about certificates, we assume that triggers are designed so that, whenever they hold in a non-social event, their satisfaction can be justified by exhibiting a concrete (nonempty) set of validated tokens.

\begin{Definition}[Witness extractor for a trigger protocol]
\label{def:wit}
An extractor $\mathrm{Wit}_\chi$ is a deterministic token-selection map such that, for every event $E$:
\begin{enumerate}
\item $\mathrm{Wit}_\chi(E)\subseteq\mathcal T(E)$ (token soundness);
\item if $E\models\chi$ and $\mathcal T(E)\ne\emptyset$, then $\mathrm{Wit}_\chi(E)\ne\emptyset$ (nonvacuity);
\item under that same condition, $E\models_{\mathrm{Wit}_\chi(E)}\chi$ (witness correctness).
\end{enumerate}
On token-empty events the map returns $\emptyset$; no semantic certificate of an accepted evidence-bearing step is claimed there. Existence is assumed when a theorem uses such a map, rather than asserted for arbitrary FO predicates.
\end{Definition}

\begin{Definition}[Token-only witnessability]
\label{def:token_witnessable}
A trigger is token-only witnessable if on every satisfying non-social event it has a nonempty validated witness whose restricted predicate instances can be verified from the annotated witness, fixed public context, and any required proof payload, without non-token event metadata. This is a restriction beyond event-relative witnessing in Definition~\ref{def:token_restricted_sat}. For guarded positive existential triggers, such a certificate also establishes full-event satisfaction; for nonmonotone triggers a separate full-event check remains necessary.
\end{Definition}

\begin{Remark}
Extractor existence must be checked for the chosen semantics. A formula requiring an invalid token may be true on the full event but false on every validated-token restriction. Absence conditions can be witnessable yet still require a full-event check. Token-only verification additionally excludes dependence on non-token event metadata.
\end{Remark}

\paragraph{Certificates and canonicalization.}
A (step) certificate is a pair $\pi=(\ell,W)$ where $\ell\in\{1,\dots,r_i\}\cup\{\bot\}$ is a label (a trigger index or the fallback symbol $\bot$) and $W\subseteq \Toks$ is the witness set.

\begin{Definition}[Certificate canonicalization]
\label{def:can}
Define $\mathrm{can}(\ell,W)$ by
\[
\mathrm{can}(\ell,W)\;:=\;
\begin{cases}
(\bot,\emptyset), & W=\emptyset,\\
(\ell,W), & W\neq\emptyset.
\end{cases}
\]
\end{Definition}

Canonicalization records that an empty witness cannot authorize an evidence-bearing transition. It is a projection used for equivalence proofs; the original proposal and rejection reason remain in the operational log.

\paragraph{Certified one-step semantics.}
Fix witness extractors $\mathrm{Wit}_{\psi_{i,j}}$ for each trigger $\psi_{i,j}$. Define the \emph{certified} one-step semantics of $P_i$ by
\[
\widehat{F}_{P_i}(b,E)\;:=\;\Big(F_{P_i}(b,E),\ \mathrm{can}\big(\pi_{P_i}(E)\big)\Big),
\]
where the protocol's (pre-canonical) certificate map is
\[
\pi_{P_i}(E):=
\begin{cases}
\big(j^\star_{P_i}(E),\ \mathrm{Wit}_{\psi_{i,j^\star_{P_i}(E)}}(E)\big), & J_{P_i}(E)\neq\emptyset,\\
(\bot,\emptyset), & J_{P_i}(E)=\emptyset.
\end{cases}
\]
Similarly, for a \PBRC{} contract $C_i$ with triggers $\varphi_{i,j}$ and witness extractors $\mathrm{Wit}_{\varphi_{i,j}}$, define $\widehat{F}_{C_i}$ analogously.

\begin{Definition}[Belief equivalence and audit equivalence]
\label{def:equiv}
Two protocols/constraints $X$ and $Y$ (for the same agent and hypothesis space) are:
\begin{enumerate}
\item \emph{belief-equivalent}, written $X\equiv_{\mathrm{B}}Y$, if $F_X(b,E)=F_Y(b,E)$ for all beliefs $b$ and events $E$ (equivalently, they induce identical belief trajectories on every event stream);
\item \emph{audit-equivalent}, written $X\equiv_{\mathrm{A}}Y$, if $\widehat{F}_X(b,E)=\widehat{F}_Y(b,E)$ for all beliefs $b$ and events $E$ (equivalently, they induce identical trajectories of beliefs together with canonicalized certificates on every event stream).
\end{enumerate}
\end{Definition}

\subsection{Evidence-splitting compilation to an evidential \PBRC{} contract}
\label{sec:nf_compilation}

Recall that $\mathrm{SocialOnly}(E)$ means that $E$ contains \emph{no} validated token, i.e.\ $\mathcal{T}(E)=\emptyset$.
Equivalently, $\neg\mathrm{SocialOnly}(E)$ holds iff $E$ contains at least one validated token.

The compilation below separates (i) \emph{evidence-present} events, where triggers are made explicitly evidence-gated, from (ii) \emph{social-only} events, where the compiled contract is forced to reproduce the original protocol's behavior by pushing that behavior into the fallback.\footnote{This ``evidence-splitting'' is intentional. It yields an \emph{evidential} contract (all triggers require nonempty validated evidence) while preserving the original one-step belief semantics on all events. In later sections, when we impose socially robust constraints (e.g., argmax-preserving fallback on social-only events), the social-only behavior absorbed into the fallback is correspondingly constrained.}

\begin{Definition}[\PBRC{} normal form of a trigger protocol]
\label{def:nf}
Given $P_i=\langle (\psi_{i,1},V_{i,1}),\dots,(\psi_{i,r_i},V_{i,r_i}),\prec_i,\delta_i\rangle$, define its \emph{\PBRC{} normal form} $\mathrm{NF}(P_i)$ as the contract
\[
\mathrm{NF}(P_i)
=\langle (\varphi_{i,1},U_{i,1}),\dots,(\varphi_{i,r_i},U_{i,r_i}),\prec_i,\delta_i^{\mathrm{NF}}\rangle
\]
where for each $j\in\{1,\dots,r_i\}$,
\[
\varphi_{i,j}\;:=\;\psi_{i,j}\ \land\ \neg \mathrm{SocialOnly},
\qquad
U_{i,j}\;:=\;V_{i,j},
\]
and the fallback $\delta_i^{\mathrm{NF}}$ is defined by
\begin{equation}
\label{eq:nf_fallback}
\delta_i^{\mathrm{NF}}(b,E):=
\begin{cases}
F_{P_i}(b,E), & \mathrm{SocialOnly}(E),\\
\delta_i(b,E), & \neg \mathrm{SocialOnly}(E).
\end{cases}
\end{equation}
\end{Definition}

\begin{Lemma}[Evidentiality of $\mathrm{NF}(P_i)$]
\label{lem:nf_evidential}
$\mathrm{NF}(P_i)$ is an evidential \PBRC{} contract: for every $j$ and every event $E$,
\[
E\models \varphi_{i,j}\quad\Rightarrow\quad \neg \mathrm{SocialOnly}(E)\quad\Rightarrow\quad \exists \tau\ \mathrm{Valid}(\tau).
\]
\end{Lemma}

\begin{proof}
If $E\models\varphi_{i,j}$ then by definition of $\varphi_{i,j}$ we have $E\models \neg \mathrm{SocialOnly}$. Unfolding $\neg \mathrm{SocialOnly}(E)$ (equivalently, $\mathcal{T}(E)\neq\emptyset$) yields that $E$ contains at least one validated token, i.e.\ $\exists \tau\,\mathrm{Valid}(\tau)$.
\end{proof}

\subsection{Representation theorem (belief and audit equivalence)}
\label{sec:nf_representation}

We next show that $\mathrm{NF}(P_i)$ preserves both the induced belief dynamics and the auditor-visible certificate trace (after canonicalization). The only subtlety is the treatment of social-only events: when $\mathcal{T}(E)=\emptyset$, any token-sound witness extractor must return the empty set, so all such ``justifications'' collapse to fallback under canonicalization.

\begin{Lemma}[Token-sound witnesses are empty on social-only events]
\label{lem:wit_empty_socialonly}
Let $\chi$ be any trigger and let $\mathrm{Wit}_\chi$ satisfy token soundness (Definition~\ref{def:wit}). If $\mathrm{SocialOnly}(E)$ then $\mathrm{Wit}_\chi(E)=\emptyset$.
\end{Lemma}

\begin{proof}
If $\mathrm{SocialOnly}(E)$ then $\mathcal{T}(E)=\emptyset$. Token soundness gives $\mathrm{Wit}_\chi(E)\subseteq \mathcal{T}(E)=\emptyset$, hence $\mathrm{Wit}_\chi(E)=\emptyset$.
\end{proof}

\begin{Theorem}[\PBRC{} normal form representation (two equivalences)]
\label{thm:normalform}
Let $P_i$ be any auditable trigger protocol and let $\mathrm{NF}(P_i)$ be its \PBRC{} normal form (Definition~\ref{def:nf}). Then:
\begin{enumerate}
\item (\emph{belief equivalence}) $\mathrm{NF}(P_i)\equiv_{\mathrm{B}} P_i$; i.e.\ for all beliefs $b$ and events $E$,
\[
F_{\mathrm{NF}(P_i)}(b,E)=F_{P_i}(b,E).
\]
\item (\emph{audit equivalence}) Fix witness extractors $\mathrm{Wit}_{\psi_{i,j}}$ satisfying Definition~\ref{def:wit} for the triggers of $P_i$ and define witness extractors for $\mathrm{NF}(P_i)$ by
\[
\mathrm{Wit}_{\varphi_{i,j}}(E):=\mathrm{Wit}_{\psi_{i,j}}(E)\qquad\text{for all }E,j.
\]
Then $\mathrm{NF}(P_i)\equiv_{\mathrm{A}} P_i$; i.e.\ for all $b,E$,
\[
\widehat{F}_{\mathrm{NF}(P_i)}(b,E)=\widehat{F}_{P_i}(b,E).
\]
Moreover, if $\neg\mathrm{SocialOnly}(E)$ then the \emph{uncanonicalized} certificates coincide:
\[
\pi_{\mathrm{NF}(P_i)}(E)=\pi_{P_i}(E).
\]
\end{enumerate}
Consequently, for any event stream $(E^t)_{t\ge 0}$ and initial belief $b^0$, the two protocols generate identical belief trajectories, and any auditor who canonicalizes empty-witness certificates observes identical certificate traces.
\end{Theorem}

\begin{proof}
\paragraph{(1) Belief equivalence.}
Fix $b$ and $E$.

\emph{Case 1: $\mathrm{SocialOnly}(E)$.}
Then $\neg \mathrm{SocialOnly}(E)$ is false, so for every $j$ the compiled trigger $\varphi_{i,j}=\psi_{i,j}\land \neg\mathrm{SocialOnly}$ is false. Hence $J_{\mathrm{NF}(P_i)}(E)=\emptyset$ and $\mathrm{NF}(P_i)$ applies its fallback:
\[
F_{\mathrm{NF}(P_i)}(b,E)=\delta_i^{\mathrm{NF}}(b,E).
\]
By \eqref{eq:nf_fallback}, $\delta_i^{\mathrm{NF}}(b,E)=F_{P_i}(b,E)$ on social-only events. Thus $F_{\mathrm{NF}(P_i)}(b,E)=F_{P_i}(b,E)$.

\emph{Case 2: $\neg \mathrm{SocialOnly}(E)$.}
Then for each $j$ we have $E\models\varphi_{i,j}$ iff $E\models\psi_{i,j}$, because the conjunct $\neg \mathrm{SocialOnly}$ holds. Therefore
\[
J_{\mathrm{NF}(P_i)}(E)=J_{P_i}(E),
\qquad
j^\star_{\mathrm{NF}(P_i)}(E)=j^\star_{P_i}(E),
\]
(with the selected-index equality understood only when the common set is nonempty, and using the same priority order $\prec_i$). If $J_{P_i}(E)\neq\emptyset$ then both apply the same triggered operator:
\[
F_{\mathrm{NF}(P_i)}(b,E)=U_{i,j^\star}(b,E)=V_{i,j^\star}(b,E)=F_{P_i}(b,E).
\]
If $J_{P_i}(E)=\emptyset$, then both take fallback and \eqref{eq:nf_fallback} yields
\[
F_{\mathrm{NF}(P_i)}(b,E)=\delta_i^{\mathrm{NF}}(b,E)=\delta_i(b,E)=F_{P_i}(b,E).
\]
Thus $F_{\mathrm{NF}(P_i)}(b,E)=F_{P_i}(b,E)$ in all cases, proving $\mathrm{NF}(P_i)\equiv_{\mathrm{B}}P_i$.

\paragraph{(2) Audit equivalence.}
On a non-social satisfying event, the original extractor returns a nonempty validated witness that satisfies $\psi$ in its restricted structure. That same restricted structure contains a valid token, so it also satisfies $\psi\land\neg\mathrm{SocialOnly}$. Thus the inherited map is a correct extractor for the compiled trigger.
Fix $b$ and $E$.

\emph{Case 1: $\mathrm{SocialOnly}(E)$.}
Then $\mathcal{T}(E)=\emptyset$, and by Lemma~\ref{lem:wit_empty_socialonly} every token-sound witness extractor returns the empty set:
$\mathrm{Wit}_{\psi_{i,j}}(E)=\emptyset$ for all $j$.
Hence, regardless of whether $P_i$ selects a trigger label, its (pre-canonical) certificate has empty witness, so
\[
\mathrm{can}\big(\pi_{P_i}(E)\big)=(\bot,\emptyset).
\]
For $\mathrm{NF}(P_i)$, we already established that $J_{\mathrm{NF}(P_i)}(E)=\emptyset$, so $\pi_{\mathrm{NF}(P_i)}(E)=(\bot,\emptyset)$ and therefore
$\mathrm{can}\big(\pi_{\mathrm{NF}(P_i)}(E)\big)=(\bot,\emptyset)$.
Together with belief equivalence from part (1), this gives
\[
\widehat{F}_{\mathrm{NF}(P_i)}(b,E)=\widehat{F}_{P_i}(b,E).
\]

\emph{Case 2: $\neg \mathrm{SocialOnly}(E)$.}
Then $J_{\mathrm{NF}(P_i)}(E)=J_{P_i}(E)$ and hence both protocols select the same trigger label whenever any trigger holds. If $J_{P_i}(E)=\emptyset$, both certificates are $(\bot,\emptyset)$ and canonicalized certificates match.
If $J_{P_i}(E)\neq\emptyset$, let $j^\star=j^\star_{P_i}(E)=j^\star_{\mathrm{NF}(P_i)}(E)$. By construction of witness extractors,
\[
\mathrm{Wit}_{\varphi_{i,j^\star}}(E)=\mathrm{Wit}_{\psi_{i,j^\star}}(E),
\]
so the (pre-canonical) certificates coincide:
\[
\pi_{\mathrm{NF}(P_i)}(E)
=\big(j^\star,\mathrm{Wit}_{\varphi_{i,j^\star}}(E)\big)
=\big(j^\star,\mathrm{Wit}_{\psi_{i,j^\star}}(E)\big)
=\pi_{P_i}(E).
\]
In particular, the canonicalized certificates coincide. Combining with part (1) yields
$\widehat{F}_{\mathrm{NF}(P_i)}(b,E)=\widehat{F}_{P_i}(b,E)$.

\paragraph{Trajectories.}
Since the one-step belief outputs and canonicalized certificates agree for all $(b,E)$, equality of belief trajectories and canonicalized audit traces over any event stream follows by induction on time.
\end{proof}

\begin{Corollary}[Canonical reduction to evidential \PBRC{} without losing auditable evidence-traces]
\label{cor:canon}
Within the class of FO-definable auditable trigger protocols equipped with extractors satisfying Definition~\ref{def:wit}, any update policy is without loss of generality representable as an \emph{evidential} \PBRC{} contract that is both belief-equivalent and audit-equivalent (Definition~\ref{def:equiv}): all non-fallback updates are explicitly evidence-gated, and the projected certificate trace is preserved. The compiled fallback reproduces the original social-only behavior, which need not be conservative.
\end{Corollary}

\begin{proof}
Immediate from Lemma~\ref{lem:nf_evidential} and Theorem~\ref{thm:normalform}.
\end{proof}

\subsection{Routers, enforcement, and router-level normal form}
\label{sec:routerNF}

\PBRC{} is designed to be \emph{enforceable}: a system-level router (or ``output gate'') can reject belief changes that are not justified by a verifiable, nonempty evidence witness. The audit-equivalence notion in Section~\ref{sec:normalform} compares \emph{canonicalized} certificates, but it does not yet capture the operational effect of enforcement, where the gate may \emph{override} a proposed update and force fallback. We now formalize this enforcement semantics and show that it induces a canonical \emph{evidence-gated} normal form.

\paragraph{Witnessed protocols.}
Throughout this subsection, we treat a trigger protocol as equipped with fixed witness extractors.
Formally, fix an agent $i$ and a trigger protocol
\[
P_i=\langle (\psi_{i,1},V_{i,1}),\dots,(\psi_{i,r_i},V_{i,r_i}),\prec_i,\delta_i\rangle
\]
together with witness extractors $\mathrm{Wit}_{\psi_{i,j}}$ satisfying Definition~\ref{def:wit}.
Define the induced certificate map $\pi_{P_i}$ by
\[
\pi_{P_i}(E):=
\begin{cases}
\big(j^\star_{P_i}(E),\ \mathrm{Wit}_{\psi_{i,j^\star_{P_i}(E)}}(E)\big), & J_{P_i}(E)\neq\emptyset,\\
(\bot,\emptyset), & J_{P_i}(E)=\emptyset,
\end{cases}
\]
where $J_{P_i}(E)$ and $j^\star_{P_i}(E)$ are as in \eqref{eq:Pi_update}.
As before, $\mathrm{can}$ denotes certificate canonicalization (Definition~\ref{def:can}).

\begin{Definition}[Nonempty-witness enforcement (output gate)]
\label{def:gate}
Fix a witnessed trigger protocol $P_i$ with fallback $\delta_i$ and certificate map $\pi_{P_i}$.
The \emph{gate-enforced} one-step update is
\[
F^{\mathsf{gate}}_{P_i}(b,E)\;:=\;
\begin{cases}
F_{P_i}(b,E), & \mathrm{can}(\pi_{P_i}(E))\neq (\bot,\emptyset),\\
\delta_i(b,E), & \mathrm{can}(\pi_{P_i}(E))= (\bot,\emptyset).
\end{cases}
\]
This idealized gate is applied to the honestly computed $F_{P_i}$ and its specified extractor. It preserves that candidate update iff the certificate is nonempty; otherwise it forces fallback. It is not a validator for arbitrary untrusted certificates or arbitrary agent-supplied numerical outputs.
\end{Definition}

\begin{Definition}[Evidence-gated compilation]
\label{def:eg}
Given $P_i=\langle (\psi_{i,1},V_{i,1}),\dots,(\psi_{i,r_i},V_{i,r_i}),\prec_i,\delta_i\rangle$, define the \emph{evidence-gated} \PBRC{} contract
\[
\mathrm{EG}(P_i)\;:=\;\langle (\psi_{i,1}\land \neg\mathrm{SocialOnly},V_{i,1}),\dots,(\psi_{i,r_i}\land \neg\mathrm{SocialOnly},V_{i,r_i}),\prec_i,\delta_i\rangle.
\]
That is, $\mathrm{EG}(P_i)$ differs from $P_i$ only by conjoining $\neg\mathrm{SocialOnly}$ to every trigger; the fallback is unchanged.
\end{Definition}

\begin{Theorem}[Router normal form (output-gated equivalence)]
\label{thm:routernf}
Assume the witness extractors $\mathrm{Wit}_{\psi_{i,j}}$ satisfy token soundness and non-vacuity on non-social satisfaction (Definition~\ref{def:wit}).
Then for all beliefs $b$ and events $E$,
\[
F^{\mathsf{gate}}_{P_i}(b,E) \;=\; F_{\mathrm{EG}(P_i)}(b,E).
\]
Consequently, for any event stream $(E^t)_{t\ge 0}$ and initial belief $b^0$, the belief trajectory produced by running $P_i$ under the nonempty-witness gate is identical to the belief trajectory produced by running the evidence-gated contract $\mathrm{EG}(P_i)$ without an external gate.
\end{Theorem}

\begin{proof}
Fix $b$ and $E$. We proceed by cases.

\emph{Case 1: $\mathrm{SocialOnly}(E)$.}
Then $\mathcal{T}(E)=\emptyset$. By token soundness (Definition~\ref{def:wit}), every extracted witness set is empty; in particular, if $J_{P_i}(E)\neq\emptyset$ then $\pi_{P_i}(E)=(j,\emptyset)$ for some label $j$, hence $\mathrm{can}(\pi_{P_i}(E))=(\bot,\emptyset)$, while if $J_{P_i}(E)=\emptyset$ then $\pi_{P_i}(E)=(\bot,\emptyset)$ already. In either subcase, Definition~\ref{def:gate} yields
\[
F^{\mathsf{gate}}_{P_i}(b,E)=\delta_i(b,E).
\]
For $\mathrm{EG}(P_i)$, each compiled trigger contains the conjunct $\neg\mathrm{SocialOnly}$, which is false in this case, so no trigger is satisfied and the contract outputs fallback:
\[
F_{\mathrm{EG}(P_i)}(b,E)=\delta_i(b,E).
\]
Thus the two sides agree.

\emph{Case 2: $\neg\mathrm{SocialOnly}(E)$.}
Then for every $j$, $E\models(\psi_{i,j}\land \neg\mathrm{SocialOnly})$ iff $E\models \psi_{i,j}$. Hence
$J_{\mathrm{EG}(P_i)}(E)=J_{P_i}(E)$ and, when nonempty, both select the same highest-priority trigger $j^\star$.

If $J_{P_i}(E)=\emptyset$, then both $P_i$ and $\mathrm{EG}(P_i)$ apply fallback, and
\[
F^{\mathsf{gate}}_{P_i}(b,E)=\delta_i(b,E)=F_{\mathrm{EG}(P_i)}(b,E).
\]

If $J_{P_i}(E)\neq\emptyset$, let $j^\star=j^\star_{P_i}(E)=j^\star_{\mathrm{EG}(P_i)}(E)$. Since $E\models\psi_{i,j^\star}$ and $\neg\mathrm{SocialOnly}(E)$, non-vacuity implies
$\mathrm{Wit}_{\psi_{i,j^\star}}(E)\neq\emptyset$, so
$\mathrm{can}(\pi_{P_i}(E))\neq(\bot,\emptyset)$ and the gate accepts:
\[
F^{\mathsf{gate}}_{P_i}(b,E)=F_{P_i}(b,E)=V_{i,j^\star}(b,E).
\]
The evidence-gated contract $\mathrm{EG}(P_i)$ also applies $V_{i,j^\star}$ on this event, so
$F_{\mathrm{EG}(P_i)}(b,E)=V_{i,j^\star}(b,E)$.
Thus the two sides agree.

Since the one-step outputs coincide for all $(b,E)$, equality of full trajectories follows by induction over time.
\end{proof}

\paragraph{From ``nonempty'' gates to full routers.}
The gate in Definition~\ref{def:gate} only inspects whether the witness is nonempty after canonicalization. In realistic deployments, routers also verify that a proposed certificate is \emph{correctly witnessed}.

\begin{Definition}[Soundness and completeness of a router]
\label{def:router_sc}
Fix the ordered trigger family $(\chi_j)$ and extractors. A router is sound if accepting $(j,W)$ implies: $j$ is a genuine trigger label, $W\ne\emptyset$, $W\subseteq\mathcal T(E)$, $E\models_W\chi_j$, $E\models\chi_j$, and $j$ is the highest-priority satisfied trigger on the full event. It is complete if it accepts the canonical certificate of that selected trigger whenever the event is non-social and the extractor meets its contract. A certificate for a satisfied but unselected lower-priority trigger need not be accepted. Numerical execution is governed separately by Section~\ref{sec:enforcement}.
\end{Definition}

\begin{Remark}[Why Theorem~\ref{thm:routernf} is the right ``normal form'' statement]
Theorem~\ref{thm:routernf} isolates the \emph{semantic projection} induced by enforcing \emph{nonempty, token-sound witnessing}: once empty-witness steps are rejected, the only behavior that survives enforcement is the behavior of the evidence-gated projection $\mathrm{EG}(P_i)$. For full routers the same equality additionally needs selected-branch completeness. A sound but incomplete router may replace further evidence-bearing transitions by fallback.
\end{Remark}

\subsubsection{A router class and semantic projection}
\label{sec:routerclass_nf}

We now lift the single ``nonempty-witness'' gate to an abstract class of routers and show that evidence-gating is a \emph{semantic projection} induced by enforcement. The key point is that, once a router rejects unauditable (empty-witness) justifications and is complete on selected non-social satisfied triggers, any trigger protocol behaves exactly like its explicit evidence-gated compilation.

\paragraph{Router-enforced one-step semantics.}
Recall that a witnessed protocol/contract $X$ determines (i) a one-step update map $F_X(b,E)$ and (ii) a certificate map $\pi_X(E)=(\ell,W)$; canonicalization $\mathrm{can}$ is defined in Definition~\ref{def:can}.
We treat routers as accept/reject predicates over canonicalized certificates.

\begin{Definition}[Router acceptance predicate and enforced semantics]
\label{def:router_sem}
A (deterministic) router is a predicate
\[
\mathsf{Acc}:\; \mathsf{Evt}\times(\{1,\dots,r_i\}\cup\{\bot\})\times \pow(\Toks)\ \to\ \{0,1\},
\]
where $\mathsf{Acc}(E,\ell,W)=1$ means ``accept the step justified by label $\ell$ with witness $W$ on event $E$.''

Given a witnessed protocol/contract $X$ with fallback $\delta_i$ and certificate map $\pi_X(E)$, define the \emph{router-enforced} one-step update:
\[
F_X^{\mathsf{Acc}}(b,E)\;:=\;
\begin{cases}
F_X(b,E), & \mathsf{Acc}\big(E,\mathrm{can}(\pi_X(E))\big)=1,\\
\delta_i(b,E), & \mathsf{Acc}\big(E,\mathrm{can}(\pi_X(E))\big)=0.
\end{cases}
\]
The nonempty-witness gate of Definition~\ref{def:gate} is the special case
\[
\mathsf{Acc}(E,\ell,W)\;=\;\mathbb{I}\big[\ell\in\{1,\ldots,r_i\}\ \land\ W\ne\emptyset\big].
\]
\end{Definition}

\paragraph{A router class.}
We formalize the minimal router axioms needed for ``semantic projection to evidence-gating.'' These axioms are parameterized by a trigger family $(\chi_j)$ and corresponding witness extractors $\mathrm{Wit}_{\chi_j}$.

\begin{Definition}[Axioms for a sound-and-complete router class]
\label{def:router_class}
Fix the ordered trigger family $(\chi_j)$ and extractors satisfying Definition~\ref{def:wit}. The class $\mathfrak R(\chi,\mathrm{Wit})$ consists of acceptance predicates with:
\begin{enumerate}
\item \emph{Malformed/empty rejection:} reject every certificate with an invalid or fallback label, or with $W=\emptyset$.
\item \emph{Soundness:} acceptance of $(j,W)$ implies nonempty validated membership, restricted and full-event satisfaction, and correct priority selection as in Definition~\ref{def:router_sc}.
\item \emph{Selected-branch completeness:} when $J(E)\ne\emptyset$, $j=\min_{\prec_i}J(E)$, and $\mathcal T(E)\ne\emptyset$, accept $(j,\mathrm{Wit}_{\chi_j}(E))$.
\end{enumerate}
These axioms concern certificates generated from the same fixed contract and public context. For the numerical equalities below, the router computes or verifies the stated update.
\end{Definition}

\begin{Definition}[Sound routers and token-determinedness]
\label{def:sound_router}
With $(\chi_j)$ and $\mathrm{Wit}_{\chi_j}$ fixed, write $\mathsf{Acc}\in\mathfrak{R}_{\mathrm{snd}}(\chi,\mathrm{Wit})$ if $\mathsf{Acc}$ satisfies the \emph{empty rejection} and \emph{soundness} clauses of Definition~\ref{def:router_class} (but not necessarily completeness).

We also say that $\mathsf{Acc}$ is \emph{token-determined} if for all token-equivalent events $E\equiv_{\Toks}E'$ and all certificates $(\ell,W)$ with $W\subseteq \mathcal{T}(E)=\mathcal{T}(E')$,
\[
\mathsf{Acc}(E,\ell,W)=\mathsf{Acc}(E',\ell,W).
\]
\end{Definition}

\paragraph{Safety--liveness separation.}
The next proposition makes precise a commonly used engineering heuristic: sound-but-incomplete routers preserve safety properties but may harm liveness (by rejecting admissible evidence-bearing updates).

\begin{Proposition}[Safety--liveness separation for routers]
\label{prop:safety_liveness_router}
Let $X$ be any witnessed trigger protocol or contract with fallback $\delta_i$, and let $\mathsf{Acc}\in\mathfrak{R}_{\mathrm{snd}}(\chi,\mathrm{Wit})$ be any sound router that rejects empty-witness certificates.
Then every accepted non-fallback step of $F_X^{\mathsf{Acc}}$ is admissible (in the sense that it is supported by a token-sound witness for the named trigger), while any rejected step devolves to fallback.
Consequently, \PBRC{} \emph{safety} guarantees that are proved under the assumption ``either a trigger update is admissible and evidence-witnessed or else fallback is applied'' remain valid under sound-but-incomplete routers; completeness is only needed for \emph{liveness} (eventual acceptance of admissible evidence-bearing steps) and for equivalence results that require that all admissible steps are accepted.
\end{Proposition}

\begin{proof}
By Definition~\ref{def:router_sem}, if the router rejects then the enforced semantics applies fallback. If the router accepts, then by soundness (Definition~\ref{def:router_class}) the accepted certificate $(j,W)$ satisfies $W\subseteq \mathcal{T}(E)$ and $W$ witnesses $E\models \chi_j$, i.e., the step is evidence-justified in the intended sense. Safety results that rely only on (i) fallback behavior and (ii) the fact that non-fallback steps are evidence-witnessed therefore continue to hold; completeness affects only whether admissible steps are eventually accepted.
\end{proof}

\paragraph{Evidence-gating as an idempotent projection.}

\begin{Proposition}[Idempotence of evidence gating]
\label{prop:eg_idem}
For any trigger protocol $P_i$, $\mathrm{EG}(\mathrm{EG}(P_i))=\mathrm{EG}(P_i)$ up to logical equivalence of triggers.
Moreover, if $P_i$ is already evidence-gated (i.e., every trigger $\psi_{i,j}$ implies $\neg\mathrm{SocialOnly}$), then $\mathrm{EG}(P_i)=P_i$ (up to logical equivalence of triggers).
\end{Proposition}

\begin{proof}
By Definition~\ref{def:eg}, $\mathrm{EG}$ conjoins $\neg\mathrm{SocialOnly}$ to every trigger. Conjoining the same conjunct twice is idempotent, hence $\mathrm{EG}(\mathrm{EG}(P_i))=\mathrm{EG}(P_i)$. If each $\psi_{i,j}$ already implies $\neg\mathrm{SocialOnly}$, then $\psi_{i,j}\land \neg\mathrm{SocialOnly}$ is logically equivalent to $\psi_{i,j}$, so the compilation does not change the trigger satisfaction relation and yields the same protocol/contract (up to propositional equivalence).
\end{proof}

\paragraph{Semantic projection theorem.}
We now state the router-class normal form: any protocol, when enforced by a router in $\mathfrak{R}$, behaves exactly like its explicit evidence-gated compilation.

\begin{Theorem}[Router-class normal form (semantic projection to evidence gating)]
\label{thm:routerclass}
Let $P_i=\langle(\psi_{i,1},V_{i,1}),\dots,(\psi_{i,r_i},V_{i,r_i}),\prec_i,\delta_i\rangle$ be a witnessed trigger protocol whose witness extractors satisfy Definition~\ref{def:wit}.
Let $\mathsf{Acc}\in \mathfrak{R}(\psi,\mathrm{Wit})$ be any router that is sound, complete on selected non-social satisfaction, and rejects empty-witness certificates for the triggers $(\psi_{i,j})$.
Then for all beliefs $b$ and events $E$,
\[
F_{P_i}^{\mathsf{Acc}}(b,E)\;=\; F_{\mathrm{EG}(P_i)}(b,E).
\]
In particular, under such enforcement, the externally observed behavior of $P_i$ depends only on its evidence-gated projection $\mathrm{EG}(P_i)$, not on any social-only trigger structure.
\end{Theorem}

\begin{proof}
Fix $b$ and $E$.

\emph{Case 1: $\mathrm{SocialOnly}(E)$.}
Then $\mathcal{T}(E)=\emptyset$. By token soundness (Definition~\ref{def:wit}), every extracted witness set is empty, hence $\mathrm{can}(\pi_{P_i}(E))=(\bot,\emptyset)$ (either because no trigger holds or because any selected trigger has empty witness and is canonicalized). By empty rejection,
$\mathsf{Acc}(E,\bot,\emptyset)=0$, and Definition~\ref{def:router_sem} yields
\[
F_{P_i}^{\mathsf{Acc}}(b,E)=\delta_i(b,E).
\]
In $\mathrm{EG}(P_i)$, each trigger includes the conjunct $\neg\mathrm{SocialOnly}$, which is false here, so no trigger is satisfied and
\[
F_{\mathrm{EG}(P_i)}(b,E)=\delta_i(b,E).
\]
Thus the equality holds.

\emph{Case 2: $\neg\mathrm{SocialOnly}(E)$.}
Then for each $j$, $E\models(\psi_{i,j}\land\neg\mathrm{SocialOnly})$ iff $E\models\psi_{i,j}$, so $P_i$ and $\mathrm{EG}(P_i)$ select the same highest-priority trigger label when any trigger is satisfied.

If $J_{P_i}(E)=\emptyset$, then both apply fallback $\delta_i(b,E)$ and the equality is immediate.

If $J_{P_i}(E)\neq\emptyset$, let $j^\star=j^\star_{P_i}(E)$. Since $E\models \psi_{i,j^\star}$ and $\neg\mathrm{SocialOnly}(E)$, non-vacuity implies $\mathrm{Wit}_{\psi_{i,j^\star}}(E)\neq\emptyset$. By completeness,
\[
\mathsf{Acc}\big(E,j^\star,\mathrm{Wit}_{\psi_{i,j^\star}}(E)\big)=1.
\]
Hence the router accepts the canonicalized certificate, and Definition~\ref{def:router_sem} yields
\[
F_{P_i}^{\mathsf{Acc}}(b,E)=F_{P_i}(b,E)=V_{i,j^\star}(b,E).
\]
The evidence-gated contract $\mathrm{EG}(P_i)$ applies the same operator $V_{i,j^\star}$ on this event, so
$F_{\mathrm{EG}(P_i)}(b,E)=V_{i,j^\star}(b,E)$, proving equality.
\end{proof}

\begin{Corollary}[Router transparency for evidence-gated contracts]
\label{cor:eg_transparency}
Let $\mathsf{Acc}\in \mathfrak{R}(\chi,\mathrm{Wit})$ and let $C_i$ be an evidential contract whose triggers are of the form $\chi_j\land\neg\mathrm{SocialOnly}$, equipped with witness extractors satisfying Definition~\ref{def:wit}.
Then
\[
F_{C_i}^{\mathsf{Acc}}(b,E)=F_{C_i}(b,E)\quad\text{for all }b,E.
\]
\end{Corollary}

\begin{proof}
Fix $b,E$.
If no trigger fires, then $F_{C_i}(b,E)=\delta_i(b,E)$ and the certificate canonicalizes to $(\bot,\emptyset)$, so the enforced semantics also yields $\delta_i(b,E)$.
If a trigger $j$ fires, evidentiality implies $\neg\mathrm{SocialOnly}(E)$; by non-vacuity the extracted witness is nonempty, and by completeness the router accepts the corresponding certificate. Hence $F_{C_i}^{\mathsf{Acc}}(b,E)=F_{C_i}(b,E)$.
\end{proof}

\subsubsection{Token-sufficiency and persuasion separation}
\label{sec:tokensuff}

Router-level enforcement projects protocols onto evidence-gated dynamics (Theorems~\ref{thm:routernf}--\ref{thm:routerclass}). The next result sharpens this into a \emph{separation principle}: under natural invariance assumptions, the \emph{validated token set} carried by an event is a sufficient statistic for the enforced update. Consequently, altering rhetoric, prestige cues, confidence posturing, or majority claims---without changing validated tokens---cannot affect accepted belief change.

\paragraph{Token-equivalence and token invariance.}
We make explicit which components of a contract (and which enforcement mechanisms) are allowed to depend on non-token features of an event.

\begin{Definition}[Token-equivalence and token invariance]
\label{def:tokeninv}
Two events $E$ and $E'$ are \emph{token-equivalent}, written $E\equiv_{\Toks}E'$, if they carry the same valid-token set:
\[
E\equiv_{\Toks}E' \quad:\iff\quad \mathcal{T}(E)=\mathcal{T}(E').
\]
A trigger $\chi$ is \emph{token-invariant} if $E\equiv_{\Toks}E'$ implies
\[
(E\models \chi)\ \Longleftrightarrow\ (E'\models \chi).
\]
An operator $U$ (including fallback) is \emph{token-invariant} if $E\equiv_{\Toks}E'$ implies
\[
U(b,E)=U(b,E')\qquad\text{for all beliefs }b\in\Prob(\Hh).
\]
A witness extractor $\mathrm{Wit}_\chi$ is \emph{token-invariant} if $E\equiv_{\Toks}E'$ implies
\[
\mathrm{Wit}_\chi(E)=\mathrm{Wit}_\chi(E').
\]
\end{Definition}

\begin{Remark}[Scope of invariance]
Token-invariance is a \emph{design constraint}, not a meta-theorem: it formalizes that admissibility is intended to depend on \emph{validated evidence}, not on who said what or how it was phrased. If a deployment wants to treat source identity or reputation as admissible input, then that information should be encoded as (and validated as) tokens; otherwise the system is vulnerable to purely social influence in exactly the sense PBRC is designed to exclude.
\end{Remark}

\paragraph{Token-sufficiency theorem.}
We also need a mild invariance condition on the router: acceptance should not depend on non-token rhetoric when presented with the same canonicalized certificate and witness set. This is captured by \emph{token-determinedness} (Definition~\ref{def:sound_router}).

\begin{Theorem}[Token-sufficiency (persuasion separation)]
\label{thm:tokensuff}
Let $C_i$ be an evidential contract whose triggers, operators (including fallback), and witness extractors are token-invariant (Definition~\ref{def:tokeninv}). Let $\mathsf{Acc}\in\mathfrak{R}_{\mathrm{snd}}(\chi,\mathrm{Wit})$ be any sound router that is token-determined (Definition~\ref{def:sound_router}). Then for all beliefs $b$ and all token-equivalent events $E\equiv_{\Toks}E'$,
\[
F_{C_i}^{\mathsf{Acc}}(b,E)\;=\;F_{C_i}^{\mathsf{Acc}}(b,E').
\]
In particular, changing unvalidated message content (rhetoric, confidence, sender identity, majority claims) without changing the set of valid evidence tokens cannot change any enforced belief update.
\end{Theorem}

\begin{proof}
Fix a belief $b$ and token-equivalent events $E\equiv_{\Toks}E'$, so $\mathcal{T}(E)=\mathcal{T}(E')$.

Let $\pi_{C_i}(E)=(\ell,W)$ and $\pi_{C_i}(E')=(\ell',W')$ be the (uncanonicalized) certificates proposed by $C_i$. Because triggers are token-invariant, the set of satisfied triggers (and hence the selected highest-priority trigger label, or fallback) is the same on $E$ and $E'$, so $\ell=\ell'$. Because witness extractors are token-invariant, the extracted witness sets also coincide: $W=W'$. Therefore the canonicalized certificates coincide:
\[
\mathrm{can}(\pi_{C_i}(E))=\mathrm{can}(\pi_{C_i}(E')) =: (\bar{\ell},\bar{W}),
\qquad\text{with }\bar{W}\subseteq \mathcal{T}(E)=\mathcal{T}(E').
\]

By token-invariance of operators (including fallback), the candidate (ungated) update agrees:
\[
F_{C_i}(b,E)=F_{C_i}(b,E').
\]
Finally, by token-determinedness of the router and $\bar{W}\subseteq \mathcal{T}(E)=\mathcal{T}(E')$, we have
\[
\mathsf{Acc}(E,\bar{\ell},\bar{W})=\mathsf{Acc}(E',\bar{\ell},\bar{W}).
\]
Applying the enforced semantics (Definition~\ref{def:router_sem}) yields the same accept/reject decision and the same resulting update on $E$ and $E'$, hence
$F_{C_i}^{\mathsf{Acc}}(b,E)=F_{C_i}^{\mathsf{Acc}}(b,E')$.
\end{proof}

\begin{figure}[t]
\centering
\resizebox{0.96\linewidth}{!}{
\begin{tikzpicture}[
  font=\small,
  node distance=9mm and 18mm,
  block/.style={
    draw,
    rounded corners,
    align=center,
    inner sep=4pt,
    minimum height=9mm,
    text width=4.1cm
  },
  sblock/.style={block, text width=3.5cm},
  group/.style={draw, rounded corners, dashed, inner sep=6pt},
  arrow/.style={-Latex, thick}
]
\node[sblock] (E) {Event $E$\\(messages + rhetoric)};
\node[sblock, below=10mm of E] (E2) {Event $E'$\\(different rhetoric)};
\node[block, right=20mm of E] (Tok) {Validated-token projection\\$\mathcal{T}(\cdot)$};
\node[block, right=20mm of Tok] (Upd) {Enforced update\\$F^{\mathsf{Acc}}_{C_i}(b,\cdot)$};

\draw[arrow] (E) -- (Tok);
\draw[arrow] (E2) -- (Tok);
\draw[arrow] (Tok) -- (Upd);

\node[group, fit=(E)(E2),
      label={[align=center]above:\small Different social presentations}] {};
\node[align=center, above=1mm of Tok] {\scriptsize fix the validated token set};
\node[align=center, above=1mm of Upd] {\scriptsize identical enforced effect};
\end{tikzpicture}}
\caption{Token-sufficiency (Theorem~\ref{thm:tokensuff}): under token-invariant contracts and token-determined routers, rhetorical variation is semantically inert once the validated token set is fixed.}
\label{fig:tokensuff}
\end{figure}

\paragraph{A counterexample without token invariance.}
The next proposition records that token-invariance is not a cosmetic restriction: if a trigger (or witness extractor) is allowed to depend on non-token features (e.g., sender identity), then token-sufficiency can fail even with sound enforcement.

\begin{Proposition}[Token-sufficiency can fail without token invariance]
\label{prop:necess_tokeninv}
Token-invariance in Theorem~\ref{thm:tokensuff} is essential. Concretely, there exists an \emph{evidential} contract $C$ and a sound-and-complete router such that for some belief $b$ there are token-equivalent events $E\equiv_{\Toks}E'$ with
\[
F_{C}^{\mathsf{Acc}}(b,E)\neq F_{C}^{\mathsf{Acc}}(b,E').
\]
\end{Proposition}

\begin{proof}
Let $\Hh=\{h_0,h_1\}$ and fix a token $\tau^\star$ with $\mathrm{Valid}(\tau^\star)$ and $\mathrm{Supports}(\tau^\star,h_1)$.
Define a \emph{sender-sensitive} (hence not token-invariant) trigger:
\[
\psi\;:=\;\exists m\,\exists\tau\ \Big(\mathrm{Send}(a_0,i,m,t)\ \wedge\ \mathrm{HasTok}(m,\tau)\ \wedge\ \mathrm{Valid}(\tau)\Big),
\]
i.e., ``agent $i$ received a valid token from distinguished sender $a_0$.''

Let $C$ have a single trigger $\psi$ with an operator that shifts probability mass toward $h_1$ (e.g., a bounded log-odds increase for $h_1$), and let fallback be any argmax-preserving operator (e.g., identity or skeptical dilution). The contract is \emph{evidential} because $\psi$ implies $\neg\mathrm{SocialOnly}$ (it explicitly requires a valid token).

Construct two events $E$ and $E'$ that contain the same valid token set $\mathcal{T}(E)=\mathcal{T}(E')=\{\tau^\star\}$, but differ only in sender metadata: in $E$ the message carrying $\tau^\star$ is sent by $a_0$, while in $E'$ it is sent by some $a_1\neq a_0$.
Then $E\equiv_{\Toks}E'$, but $E\models \psi$ whereas $E'\not\models \psi$. Therefore $C$ applies the trigger operator on $E$ and fallback on $E'$, yielding different one-step updates. A sound-and-complete router accepts the evidence-bearing step on $E$ and (trivially) enforces fallback on $E'$, so
$F_C^{\mathsf{Acc}}(b,E)\neq F_C^{\mathsf{Acc}}(b,E')$ for any belief $b$ on which the trigger operator differs from fallback.
\end{proof}

\subsubsection{Token exposure traces and network-level invariance}
\label{sec:tokentrace}

The token-sufficiency theorem (Theorem~\ref{thm:tokensuff}) is a one-step, single-agent statement: under token-invariant contracts and token-determined enforcement, an agent's \emph{enforced} update depends on an event only through its validated token set. We now lift this to multi-agent runs. The key observation is that, once contracts and routers satisfy the same invariance discipline, the \emph{entire enforced belief trajectory} of each agent is determined by the \emph{sequence} of token sets it is exposed to over time. We formalize this as a \emph{token exposure trace} and show that token-trace equivalence implies equality of enforced belief trajectories.

\begin{Definition}[Token exposure trace and token-trace equivalence]
\label{def:tokentrace}
Fix a horizon $T\in\Nat$. For an agent $i$ with received events $E_i^0,\dots,E_i^{T-1}$, define its \emph{token exposure trace} up to $T$ as
\[
\mathrm{Trace}^T_{\Toks}(i)\;:=\;\big(\mathcal{T}(E_i^0),\mathcal{T}(E_i^1),\dots,\mathcal{T}(E_i^{T-1})\big).
\]
Two deliberation runs $\rho$ and $\rho'$ are \emph{token-trace equivalent up to $T$}, written $\rho \equiv^T_{\mathrm{trace}} \rho'$, if for every agent $i\in\A$ and every round $t<T$,
\[
\mathcal{T}(E_{i,\rho}^t)=\mathcal{T}(E_{i,\rho'}^t).
\]
\end{Definition}

\begin{Theorem}[Token exposure traces determine enforced belief trajectories]
\label{thm:tokentrace}
Assume each agent $i$ runs an evidential contract $C_i$ whose triggers, operators (including fallback), and witness extractors are token-invariant (Definition~\ref{def:tokeninv}), and is enforced by a router $\mathsf{Acc}_i$ that is sound and token-determined (Definition~\ref{def:sound_router}).\footnote{Equivalently, $\mathsf{Acc}_i\in\mathfrak{R}_{\mathrm{snd}}(\chi,\mathrm{Wit})$ and is token-determined. Completeness is not required for this invariance statement.}
Consider two runs $\rho$ and $\rho'$ with identical initial beliefs $b_{i,\rho}^0=b_{i,\rho'}^0$ for all $i$.
If $\rho \equiv^T_{\mathrm{trace}} \rho'$, then for all agents $i\in\A$ and all rounds $t\le T$,
\[
b_{i,\rho}^t \;=\; b_{i,\rho'}^t.
\]
In particular, the enforced terminal beliefs at horizon $T$ are identical across runs.
\end{Theorem}

\begin{proof}
We prove by induction on $t\in\{0,1,\dots,T\}$ that $b_{i,\rho}^t=b_{i,\rho'}^t$ for all agents $i$.

\emph{Base case ($t=0$).} Holds by hypothesis.

\emph{Inductive step.} Fix $t<T$ and assume $b_{i,\rho}^t=b_{i,\rho'}^t$ for all $i$.
By token-trace equivalence, for each agent $i$ we have
\[
\mathcal{T}(E_{i,\rho}^t)=\mathcal{T}(E_{i,\rho'}^t),
\]
i.e., $E_{i,\rho}^t\equiv_{\Toks}E_{i,\rho'}^t$.
Apply Theorem~\ref{thm:tokensuff} to agent $i$ with belief $b=b_{i,\rho}^t=b_{i,\rho'}^t$ and token-equivalent events
$E=E_{i,\rho}^t$, $E'=E_{i,\rho'}^t$ to obtain
\[
F_{C_i}^{\mathsf{Acc}_i}(b_{i,\rho}^t,E_{i,\rho}^t)
=
F_{C_i}^{\mathsf{Acc}_i}(b_{i,\rho'}^t,E_{i,\rho'}^t).
\]
By definition of the run semantics, the left-hand side is $b_{i,\rho}^{t+1}$ and the right-hand side is $b_{i,\rho'}^{t+1}$.
Thus $b_{i,\rho}^{t+1}=b_{i,\rho'}^{t+1}$ for each $i$, completing the induction.
\end{proof}

\begin{Corollary}[Topology and rhetoric are irrelevant given token traces]
\label{cor:topo_irrelevant_trace}
Under the hypotheses of Theorem~\ref{thm:tokentrace}, any two multi-agent deliberation runs (possibly with different network topology, message ordering, and rhetorical content) that induce the same token exposure traces up to horizon $T$ induce identical enforced belief trajectories up to $T$.
\end{Corollary}

\begin{figure}[t]
\centering
\resizebox{0.96\linewidth}{!}{
\begin{tikzpicture}[
  font=\small,
  node distance=9mm and 18mm,
  block/.style={
    draw,
    rounded corners,
    align=center,
    inner sep=4pt,
    minimum height=9mm,
    text width=4.2cm
  },
  sblock/.style={block, text width=3.8cm},
  group/.style={draw, rounded corners, dashed, inner sep=6pt},
  arrow/.style={-Latex, thick}
]
\node[sblock] (r1) {Run $\rho$\\topology $G$, rhetoric $R$};
\node[sblock, below=10mm of r1] (r2) {Run $\rho'$\\topology $G'$, rhetoric $R'$};
\node[block, right=20mm of r1] (trace) {Per-agent token exposure traces\\$\big(\mathcal{T}(E_i^t)\big)_{i,t}$};
\node[block, right=20mm of trace] (traj) {Enforced belief trajectories\\$\big(b_i^t\big)_{i,t}$};

\draw[arrow] (r1) -- (trace);
\draw[arrow] (r2) -- (trace);
\draw[arrow] (trace) -- (traj);

\node[group, fit=(r1)(r2),
      label={[align=center]above:\small Distinct executions}] {};
\node[align=center, above=1mm of trace] {\scriptsize if traces agree};
\node[align=center, above=1mm of traj] {\scriptsize then trajectories agree};
\end{tikzpicture}}
\caption{Network-level invariance: under token-invariant contracts and sound, token-determined routers, topology and rhetoric matter only through the token exposure traces they induce.}
\label{fig:tokentrace}
\end{figure}

\subsubsection{Topology-to-trace structural conditions under token flooding}
\label{sec:topologytrace}

Theorems~\ref{thm:tokensuff} and \ref{thm:tokentrace} reduce dependence, with contracts, public context, and initial states fixed, on the communication structure to a single question: \emph{which validated tokens reach which agents at which times}. We now give a closed-form characterization of token exposure traces under a minimal dissemination layer: \emph{token flooding}, where agents forward validated tokens verbatim along the communication graph.

\begin{Definition}[Communication graph and flooding]
\label{def:flood}
Fix a directed communication graph $G=(\A,\mathcal{E})$.
Write $\mathrm{In}_G(i):=\{j\in\A:(j,i)\in\mathcal{E}\}$.
Let $\mathcal{K}_i^t\subseteq\Toks$ denote the valid tokens known immediately before round $t$, after $t$ delivery rounds. Thus $\mathcal K_i^0$ is initial knowledge, and $E_i^t$ is delivered during round $t$. Assume an initial placement $(\mathcal{K}_i^0)_{i\in\A}$.

\emph{Flooding} is the dissemination rule:
at each round $t\ge 0$, each agent $j$ sends a message containing all tokens in $\mathcal{K}_j^t$ to every out-neighbor.
Agent $i$ receives the event $E_i^t$ consisting of those messages, so
\[
\mathcal{T}(E_i^t)=\bigcup_{j\in \mathrm{In}_G(i)} \mathcal{K}_j^t,
\qquad
\mathcal{K}_i^{t+1}=\mathcal{K}_i^t \cup \mathcal{T}(E_i^t).
\]
We assume reliable synchronous delivery, fixed token validity and annotations, and no exogenous births or expiry over the stated horizon. Self-loops are allowed and represent explicit self-delivery; local retention alone is not an incoming message. Exogenous births or time-varying validity require an augmented model and are not covered by these identities.
\end{Definition}

Let $\mathrm{dist}_G(u,v)\in \Nat\cup\{\infty\}$ be the length of the shortest directed path from $u$ to $v$ in $G$, and $\infty$ if none exists.
Define the directed diameter, with value $\infty$ for a graph that is not strongly connected, by
\[
D(G):=\max_{u,v\in\A}\mathrm{dist}_G(u,v).
\]

\begin{Remark}[Standard flooding facts vs.\ PBRC-specific use]
Distance/diameter characterizations of flooding are classical in distributed algorithms (e.g., \cite{lynch1996da}). Our use here is to compose these structural facts with PBRC's token-sufficiency and enforcement semantics to obtain end-to-end invariance statements.
\end{Remark}

\begin{Theorem}[Knowledge and incoming exposure under flooding]
\label{thm:flood_distance}
Let $P_G^{\le s}$ be the relation of positive-length directed walks of lengths $1,\ldots,s$; set $P_G^{\le0}=\emptyset$. Under Definition~\ref{def:flood}, for every $i$ and $t\ge0$,
\[
\mathcal K_i^t=\bigcup_{u:\,\mathrm{dist}_G(u,i)\le t}\mathcal K_u^0,
\qquad
\mathcal T(E_i^t)=\bigcup_{u:\,(u,i)\in P_G^{\le t+1}}\mathcal K_u^0.
\]
The first union concerns retained knowledge; the second concerns actual incoming messages.
\end{Theorem}
\begin{proof}
At $t=0$, distance zero selects only the node itself. If the knowledge identity holds at $t$, the recurrence adds the knowledge of every predecessor. A path of length at most $t+1$ either already has length at most $t$ or has a last edge from a predecessor reached within $t$ steps. This proves the knowledge identity by induction.
For the incoming event, substitute that identity:
\[
\mathcal T(E_i^t)=\bigcup_{j:(j,i)\in\mathcal E}\ \bigcup_{u:\mathrm{dist}_G(u,j)\le t}\mathcal K_u^0.
\]
A source belongs to this union exactly when it reaches a predecessor within $t$ steps. Appending the final edge gives a walk of positive length at most $t+1$ into $i$. Conversely, the prefix of such a walk reaches a predecessor within distance at most $t$. This proves the exposure identity, including returning walks with $u=i$.
\end{proof}

\begin{example}[Local knowledge is not an incoming exposure]
An isolated node with $\mathcal K_i^0=\{\tau\}$ retains $\tau$ at every time, but its incoming event is empty. The zero-length path from $i$ to itself cannot deliver a message. If a deployment explicitly adds retained local tokens to the event, the enlarged exposure instead equals $\mathcal K_i^{t+1}$; that is a different, self-inclusive event convention.
\end{example}

\begin{Corollary}[Diameter bound for evidence closure under flooding]
\label{cor:flood_closure}
If $G$ is strongly connected with diameter $D(G)<\infty$, then for every agent $i$,
\[
\mathcal{K}_i^{D(G)}=\bigcup_{j\in\A}\mathcal{K}_j^0.
\]
After $D(G)$ delivery rounds every agent possesses all initial tokens. A token initially local to an agent need not have been received back through an incoming edge.
\end{Corollary}

\begin{proof}
If $G$ is strongly connected, then $\mathrm{dist}_G(j,i)\le D(G)$ for all $j,i$. Apply Theorem~\ref{thm:flood_distance} at time $t=D(G)$.
\end{proof}

\begin{Definition}[Truncated incoming-reachability equivalence]
\label{def:truncreach}
For a horizon $T\in\Nat$, graphs on the same vertices are $T$-incoming-reachability equivalent if
\[
P_G^{\le s}=P_{G'}^{\le s}\qquad\text{for }s=1,\ldots,T.
\]
This uses positive-length walks rather than the reflexive distance relation. For $T\ge1$, $P_G^{\le1}=\mathcal E$, so universal equivalence already entails identical adjacency, including self-loops. Equality of exposure traces for a particular non-distinguishing token placement can be weaker.
\end{Definition}

\begin{Theorem}[Incoming-reachability implies exposure-trace equivalence]
\label{thm:topology_trace_equiv}
Assume flooding on $G,G'$ with identical initial token placement and fixed common context. If they are $T$-incoming-reachability equivalent, their events satisfy $\rho_G\equiv^T_{\rm trace}\rho_{G'}$. Under the hypotheses of Theorem~\ref{thm:tokentrace}, identical initial beliefs then yield identical enforced trajectories through time $T$.
\end{Theorem}
\begin{proof}
For every $t<T$ and $i$, Theorem~\ref{thm:flood_distance} expresses the exposure as the union of $\mathcal K_u^0$ over $(u,i)\in P_G^{\le t+1}$. The relations and initial placement coincide, so the two unions coincide. Apply Theorem~\ref{thm:tokentrace}.
\end{proof}

\begin{Proposition}[Monotonicity under edge addition]
\label{prop:flood_mono}
Assume flooding with a fixed initial placement $(\mathcal{K}_i^0)_{i\in\A}$. If $G'=(\A,\mathcal{E}')$ satisfies $\mathcal{E}\subseteq \mathcal{E}'$ (i.e., $G'$ adds edges), then for all agents $i$ and all rounds $t$,
\[
\mathcal{K}_{i,G}^t \subseteq \mathcal{K}_{i,G'}^t
\qquad\text{and}\qquad
\mathcal{T}(E_{i,G}^t) \subseteq \mathcal{T}(E_{i,G'}^t).
\]
\end{Proposition}

\begin{proof}
Adding edges can only weakly decrease directed distances:
$\mathrm{dist}_{G'}(u,v)\le \mathrm{dist}_G(u,v)$ for all $u,v$.
Apply Theorem~\ref{thm:flood_distance} to both graphs: the set of sources within distance $\le t$ in $G$ is contained in the corresponding set for $G'$, hence
$\mathcal{K}_{i,G}^t\subseteq \mathcal{K}_{i,G'}^t$.
Every positive-length walk in $G$ is also a walk in $G'$, so $P_G^{\le t+1}\subseteq P_{G'}^{\le t+1}$. The exposure inclusion follows from the second identity of Theorem~\ref{thm:flood_distance}. This does not assert monotonicity of accuracy or confidence as edges are added.
\end{proof}

\paragraph{Tightness and optimality.}
The distance characterization above is not merely descriptive; it is optimal among all dissemination layers that respect the communication graph by allowing tokens to traverse at most one edge per round.

\begin{Definition}[Edge-respecting token dissemination]
\label{def:edge_respecting}
Fix a directed graph $G=(\A,\mathcal{E})$ and an initial placement $(\mathcal{K}_i^0)_{i\in\A}$.
A dissemination layer is \emph{edge-respecting} if it produces token knowledge sets $(\mathcal{K}_i^t)_{i\in\A,t\ge 0}$ such that for every $t\ge 0$ and every agent $i$,
\[
\mathcal{K}_i^{t+1}\ \subseteq\ \mathcal{K}_i^{t}\ \cup\ \bigcup_{j\in \mathrm{In}_G(i)} \mathcal{K}_j^{t}.
\]
Equivalently, new tokens acquired at round $t+1$ must come from in-neighbors' knowledge at round $t$ (or be retained locally). Flooding is the special case where the inclusion holds with equality.
\end{Definition}

\begin{Lemma}[Distance lower bound for edge-respecting dissemination]
\label{lem:distance_lower}
Under any edge-respecting dissemination on $G$, for every agent $i$ and every $t\ge 0$,
\[
\mathcal{K}_i^{t}\ \subseteq\ \bigcup_{\substack{j\in\A:\\ \mathrm{dist}_G(j,i)\le t}} \mathcal{K}_j^{0}.
\]
\end{Lemma}

\begin{proof}
By induction on $t$.

\emph{Base ($t=0$).} Trivial.

\emph{Step.} Assume the claim holds for $t$. By Definition~\ref{def:edge_respecting},
\[
\mathcal{K}_i^{t+1}\subseteq \mathcal{K}_i^{t}\cup \bigcup_{j\in \mathrm{In}_G(i)}\mathcal{K}_j^{t}.
\]
Apply the induction hypothesis to $\mathcal{K}_i^{t}$ and each $\mathcal{K}_j^{t}$ and use the same one-edge extension argument as in the proof of Theorem~\ref{thm:flood_distance} to conclude
$\mathcal{K}_i^{t+1}\subseteq \bigcup_{j:\mathrm{dist}_G(j,i)\le t+1}\mathcal{K}_j^0$.
\end{proof}

\begin{Theorem}[Flooding is optimal (maximal token spread at every time)]
\label{thm:flood_optimal}
Fix $G$ and an initial placement. Let $(\mathcal{K}_i^t)^{\mathsf{any}}$ be the knowledge sets produced by any edge-respecting dissemination (Definition~\ref{def:edge_respecting}) and let $(\mathcal{K}_i^t)^{\mathsf{flood}}$ be the knowledge sets under flooding (Definition~\ref{def:flood}). Then for all agents $i$ and all rounds $t$,
\[
(\mathcal{K}_i^t)^{\mathsf{any}}\ \subseteq\ (\mathcal{K}_i^t)^{\mathsf{flood}}.
\]
Consequently, flooding minimizes the worst-case time to universal evidence closure among all edge-respecting dissemination layers on $G$.
\end{Theorem}

\begin{proof}
By Lemma~\ref{lem:distance_lower},
\[
(\mathcal{K}_i^t)^{\mathsf{any}}
\subseteq \bigcup_{j:\mathrm{dist}_G(j,i)\le t}\mathcal{K}_j^0.
\]
By Theorem~\ref{thm:flood_distance}, the right-hand side equals $(\mathcal{K}_i^t)^{\mathsf{flood}}$ under flooding.
Thus $(\mathcal{K}_i^t)^{\mathsf{any}}\subseteq (\mathcal{K}_i^t)^{\mathsf{flood}}$ for all $i,t$.
The closure-time claim follows: if flooding has not delivered some initial token to some agent by time $t$, then no edge-respecting protocol could have delivered it either.
\end{proof}

\begin{Corollary}[Tight diameter lower bound for universal closure]
\label{cor:diameter_tight}
Suppose a dissemination layer on $G$ is edge-respecting and the admissible placements include a valid token located at any single source. If it achieves \emph{universal evidence closure} by time $T$---i.e., for every initial placement and every agent $i$,
\[
\mathcal{K}_i^{T}=\bigcup_{j\in\A}\mathcal{K}_j^{0},
\]
then necessarily $D(G)\le T$ (in particular, $G$ must be strongly connected).
Conversely, if $D(G)\le T$ then flooding achieves universal closure by time $T$.
Thus, under edge-respecting dissemination, the minimum worst-case closure time is exactly the directed diameter $D(G)$.
\end{Corollary}

\begin{proof}
(\emph{Necessity}.) If $D(G)>T$, pick $u,v$ with $\mathrm{dist}_G(u,v)>T$.
Consider an initial placement with a single unique token $\tau$ placed at $u$:
$\tau\in \mathcal{K}_u^0$ and $\tau\notin \mathcal{K}_j^0$ for $j\neq u$.
By Lemma~\ref{lem:distance_lower}, $\tau\notin \mathcal{K}_v^{T}$, contradicting universal closure.

(\emph{Sufficiency}.) If $D(G)\le T$, then for all $j,i$ we have $\mathrm{dist}_G(j,i)\le T$.
By Corollary~\ref{cor:flood_closure} (or directly by Theorem~\ref{thm:flood_distance} with $t=T$), flooding yields
$\mathcal{K}_i^{T}=\bigcup_{j}\mathcal{K}_j^0$ for all $i$.
\end{proof}

\begin{Theorem}[Universal exposure-trace characterization]
\label{thm:trace_tight}
Fix $T\in\Nat$ and two graphs with the same vertices. Assume distinct source tokens $\tau_u$ are available. Under flooding the following are equivalent:
\begin{enumerate}
\item $P_G^{\le s}=P_{G'}^{\le s}$ for $s=1,\ldots,T$;
\item incoming exposure traces agree for every shared initial token placement;
\item they agree for the placement $\mathcal K_u^0=\{\tau_u\}$.
\end{enumerate}
\end{Theorem}
\begin{proof}
Theorem~\ref{thm:topology_trace_equiv} gives (1)$\Rightarrow$(2); (2)$\Rightarrow$(3) is immediate. For (3)$\Rightarrow$(1), fix $s\in\{1,\ldots,T\}$ and vertices $u,v$. By the exposure identity,
\[
\tau_u\in\mathcal T(E_{v,G}^{s-1})\quad\Longleftrightarrow\quad (u,v)\in P_G^{\le s}.
\]
The identical statement holds for $G'$. Equality of traces makes the membership tests, hence the relations, equal. For $T=0$ all three assertions are vacuous.
\end{proof}

\subsubsection{End-to-end factorization theorem}
\label{sec:endtoend}

The preceding results compose when the transition rules, context schedule, and evidence semantics are fixed. Flooding determines incoming exposure through positive-length walks; the standard diameter bound determines retained knowledge closure.

\begin{Theorem}[End-to-end factorization of topology effects]
\label{thm:endtoend}
Fix initial beliefs, contracts, and a common public context schedule through horizon $T$. Each agent runs $P_i$ under a state-holding or operator-verifying router in $\mathfrak R(\psi_i,\mathrm{Wit}_i)$, with extractors satisfying Definition~\ref{def:wit}. Let $C_i=\mathrm{EG}(P_i)$ inherit the extractors of $P_i$ and have token-invariant triggers, operators (including fallback), and extractors. Under synchronous flooding with fixed validity/annotations and no exogenous births or expiry:
\begin{enumerate}
\item $b_i^t$, $t\le T$, is determined by the fixed transition rules/context, initial beliefs, initial token placement, and $(P_G^{\le s})_{1\le s\le t}$. Changing unvalidated rhetoric while preserving those inputs does not change the enforced state.
\item Graphs with the same positive-walk relations through $T$, initial beliefs, and token placement give identical belief trajectories through $T$ and identical canonical accepted certificates for rounds $t<T$.
\item Universal retained-token knowledge closure for all placements has minimum worst-case time $D(G)$ among edge-respecting dissemination layers, with infinite time if the graph is not strongly connected; assume the distinguishing token placements in Corollary~\ref{cor:diameter_tight} are allowed.
\end{enumerate}
\end{Theorem}
\begin{proof}
Complete routing and Theorem~\ref{thm:routerclass} give
\[
b_i^{t+1}=F_{P_i}^{\mathsf{Acc}_i}(b_i^t,E_i^t)=F_{C_i}(b_i^t,E_i^t).
\]
For the fixed contract, token-invariant triggers select the same branch on token-equivalent events, and token-invariant operators give the same successor. Thus the last expression factors through $\mathcal T(E_i^t)$; no additional token-determinedness assumption on the acceptance of arbitrary certificates is needed because completeness has already removed router variation on the generated certificates. Theorem~\ref{thm:flood_distance} determines this set from $P_G^{\le t+1}$ and the initial placement. Induction on $t$ proves (1), and equality of these inputs proves belief equality in (2). The selected labels and inherited token-invariant witnesses also agree; completeness accepts these certificates on non-social selected steps, and otherwise both logs project to fallback. This proves the certificate part of (2). Item (3) is Corollary~\ref{cor:diameter_tight} and concerns knowledge, not return delivery of one's own token.
\end{proof}

Figure~\ref{fig:endtoend} summarizes the dependency structure in Theorem~\ref{thm:endtoend}: topology affects enforced beliefs through positive-walk relations and the induced token-exposure traces, with all other transition inputs fixed.

\begin{figure}[t]
\centering
\resizebox{\linewidth}{!}{
\begin{tikzpicture}[
  font=\small,
  node distance=8mm and 16mm,
  block/.style={
    draw,
    rounded corners,
    align=center,
    inner sep=4pt,
    minimum height=9mm,
    text width=4.6cm
  },
  sblock/.style={
    block,
    text width=4.0cm
  },
  group/.style={
    draw,
    rounded corners,
    dashed,
    inner sep=6pt
  },
  arrow/.style={-Latex, thick},
  darr/.style={-Latex, thick, dashed}
]

\node[block] (Flood) {Dissemination layer\\
\emph{(flooding / edge-respecting)}};

\node[block, below=of Flood] (Trace) {Token exposure traces\\
$\mathrm{Trace}^T_{\Toks}(i)=\big(\mathcal{T}(E_i^t)\big)_{t<T}$};

\node[block, below=of Trace] (PBRC) {Enforced \PBRC{} dynamics\\
(routers $\mathsf{Acc}_i$ + token-invariant contracts $C_i$)};

\node[block, below=of PBRC] (Bel) {Enforced belief trajectories\\
$\big(b_i^t\big)_{i\in\A,\;t\le T}$};

\node[block, right=14mm of PBRC] (Audit) {Audit artifacts\\
(accepted certificates / witnesses)};

\node[sblock, left=20mm of Flood] (G) {Topology $G$\\(communication graph)};

\node[sblock, below=of G] (Reach) {Positive-walk relations\\
$\big(P_G^{\le s}\big)_{1\le s\le T}$};

\node[sblock, below=of Reach] (Diam) {Diameter $D(G)$\\
(knowledge closure time)};

\node[sblock, right=20mm of Flood] (K0) {Initial token placement\\
$\big(\mathcal{K}_i^0\big)_{i\in\A}$};

\draw[arrow] (G) -- (Reach);
\draw[darr] (G.west) to[bend right=25] (Diam.west);

\draw[arrow] (G) -- (Flood);
\draw[arrow] (K0) -- (Flood);

\draw[arrow] (Flood) -- (Trace);
\draw[darr] (Reach.east) to[bend left=12] (Trace.west);

\draw[arrow] (Trace) -- (PBRC);
\draw[arrow] (PBRC) -- (Bel);
\draw[darr] (PBRC) -- (Audit);

\node[group, fit=(G)(Reach)(Diam),
      label={[align=center]above:\small Network structure}] (grpNet) {};

\node[group, fit=(K0)(Flood)(Trace),
      label={[align=center]above:\small Token dissemination and exposure}] (grpTok) {};

\node[group, fit=(PBRC)(Bel)(Audit),
      label={[align=center]above:\small Enforcement and audit}] (grpEnf) {};

\end{tikzpicture} }
\caption{End-to-end factorization (Theorem~\ref{thm:endtoend}): with complete operator-compliant enforcement and fixed token-invariant contracts/context, positive-walk relations determine incoming exposure and the resulting trajectories. Diameter separately gives the worst-case retained-knowledge closure time.}
\label{fig:endtoend}
\end{figure}

\begin{Remark}[Interpretation]
Theorem~\ref{thm:endtoend} formalizes a sharp contrast with conformity-driven failures: increasing connectivity can accelerate the propagation of \emph{validated} evidence (tokens), but it cannot create epistemic pressure in the absence of evidence because social-only events are forced into fallback behavior under evidence-gated enforcement.
\end{Remark}

\paragraph{Gate transparency for compliant evidential contracts.}
A crucial deployment requirement is that enforcement should not perturb agents that already follow an evidential \PBRC{} contract.

\begin{Proposition}[Gate transparency]
\label{prop:transparency}
Let $C_i$ be an evidential \PBRC{} contract and assume its witness extractors satisfy non-vacuity on non-social satisfaction (Definition~\ref{def:wit}). Then the nonempty-witness gate is transparent:
\[
F^{\mathsf{gate}}_{C_i}(b,E)=F_{C_i}(b,E)\quad\text{for all }b,E.
\]
\end{Proposition}

\begin{proof}
If no trigger is satisfied, then $F_{C_i}(b,E)=\delta_i(b,E)$ and $\pi_{C_i}(E)=(\bot,\emptyset)$, so the gate outputs $\delta_i(b,E)$ as well.
If some trigger is satisfied, then evidentiality implies $\neg\mathrm{SocialOnly}(E)$, and by non-vacuity the witness is nonempty, hence $\mathrm{can}(\pi_{C_i}(E))\neq(\bot,\emptyset)$ and the gate accepts $F_{C_i}(b,E)$.
Thus the gated and ungated outputs coincide in all cases.
\end{proof}

\begin{Theorem}[Epistemic accountability]
\label{thm:accountability}
Let $C_i$ be an evidential \PBRC{} contract with an argmax-preserving fallback $\delta_i$, and let the agent be enforced by a sound router that rejects empty-witness certificates and faithfully computes or verifies the numerical operator or fallback.
Consider any run producing beliefs $(b_i^t)_{t\ge 0}$ and events $(E_i^t)_{t\ge 0}$.
If $\argmax(b_i^t)\neq \argmax(b_i^0)$ for some $t$, then there exists a least index $s<t$ such that $\argmax(b_i^{s+1})\neq \argmax(b_i^s)$, and this step is necessarily evidence-triggered: the router accepted a certificate $\pi_i^s=(j,W)$ with $W\neq\emptyset$ and $W\subseteq \mathcal{T}(E_i^s)$ witnessing $E_i^s\models \varphi_{i,j}$.
In particular, a top-set change has a recorded nonempty certificate licensing its selected transition. Independence from rhetoric additionally requires token invariance; the certificate alone does not identify a sole causal token.
\end{Theorem}

\begin{proof}
Let $s$ be the smallest index such that $\argmax(b_i^{s+1})\neq \argmax(b_i^s)$.
Suppose for contradiction that the step $s$ was not evidence-triggered.
Then either no trigger was satisfied (so fallback applies) or the router rejected the proposed step (including due to an empty-witness/invalid certificate), in which case enforcement forces fallback.
In either case,
\[
b_i^{s+1}=\delta_i(b_i^s,E_i^s).
\]
By argmax preservation of $\delta_i$, $\argmax(b_i^{s+1})=\argmax(b_i^s)$, contradicting the choice of $s$.
Therefore step $s$ must have been accepted as a triggered update with a nonempty witness set $W\neq\emptyset$ satisfying $W\subseteq \mathcal{T}(E_i^s)$ and witnessing some trigger $\varphi_{i,j}$.
\end{proof}

\begin{Remark}
Theorem~\ref{thm:accountability} gives trace-level licensing under faithful numerical enforcement. Witness tokens identify the evidence policy under which the step was accepted. They do not establish that token content alone caused the change, that the content was true, or that an auditor can infer the unknown ground truth. Section~\ref{sec:failuretax} gives an assumption-relative classification of correctness loss.

\end{Remark}

Figure~\ref{fig:routernf} summarizes the router normal-form theorem: under the extractor and faithful-execution hypotheses of Theorem~\ref{thm:routernf}, the nonempty gate yields the explicitly evidence-gated behavior.

\begin{figure}[t]
\centering
\resizebox{\linewidth}{!}{
\begin{tikzpicture}[
  font=\small,
  node distance=9mm and 19mm,
  block/.style={
    draw,
    rounded corners,
    align=center,
    inner sep=4pt,
    minimum height=9mm,
    text width=4.2cm
  },
  sblock/.style={
    block,
    text width=3.7cm
  },
  group/.style={
    draw,
    rounded corners,
    dashed,
    inner sep=6pt
  },
  arrow/.style={-Latex, thick},
  darr/.style={-Latex, thick, dashed}
]

\node[block]  (P)    {Arbitrary trigger protocol $P_i$\\proposes $(b',\pi)$};
\node[block, right=of P] (Gate) {Nonempty-witness gate (Def.~\ref{def:gate})\\accept iff $\mathrm{can}(\pi)\neq(\bot,\emptyset)$};
\node[sblock, right=of Gate] (Out) {Enforced update\\$b^{+}:=F^{\mathsf{gate}}_{P_i}(b,E)$};

\node[block, below=14mm of P] (EG) {Evidence-gated compilation\\$\mathrm{EG}(P_i)$ (Def.~\ref{def:eg})};
\node[block, right=of EG] (Exec) {Direct execution\\by contract semantics};
\node[sblock, right=of Exec] (Out2) {Update without gate\\$b^{+}:=F_{\mathrm{EG}(P_i)}(b,E)$};

\draw[arrow] (P) -- node[above]{\scriptsize propose step} (Gate);
\draw[arrow] (Gate) -- node[above]{\scriptsize accept/fallback} (Out);

\draw[arrow] (EG) -- node[above]{\scriptsize execute} (Exec);
\draw[arrow] (Exec) -- (Out2);

\draw[darr] (P) -- node[left]{\scriptsize compile $\mathrm{EG}$} (EG);
\draw[darr] (Out) -- node[right, align=center]{\scriptsize (Thm.~\ref{thm:routernf})} node[left, align=center]{\scriptsize $F^{\mathsf{gate}}_{P_i}=F_{\mathrm{EG}(P_i)}$} (Out2);

\node[group, fit=(P)(Gate)(Out),
      label={[align=center]above:\small Enforcement of an arbitrary protocol}] {};

\node[group, fit=(EG)(Exec)(Out2),
      label={[align=center]below:\small Equivalent evidence-gated behavior}] {};

\end{tikzpicture} }
\caption{Router normal form: with the stated correct nonvacuous extractors and faithful updates, the nonempty gate projects a trigger protocol to its evidence-gated form. Under the gate, executing $P_i$ is behaviorally equivalent to executing $\mathrm{EG}(P_i)$ directly (Thm.~\ref{thm:routernf}).}
\label{fig:routernf}
\end{figure}

\paragraph{Sufficient logs for replay.}
Token invariance permits omission of rhetorical text only when the remaining record contains every transition dependency. Record full annotated token exposures, all rounds including rejection/fallback, initial state, contracts and validator versions, public context/time, and any random routing outcomes. A witness subset is insufficient unless the operator and routing decision factor through it. For example, a trigger witnessed by $W=\{a\}$ can use a token-invariant operator that sets $b(h)=0.6$ on exposure $\{a\}$ and $b(h)=0.9$ on exposure $\{a,b\}$.

\begin{Proposition}[Replay from sufficient inputs]
\label{prop:replay_inputs}
Suppose $b_i^{t+1}=F_i(b_i^t,X_i^t,c_i^t,r_i^t)$ is deterministic, where $X_i^t$ is full annotated exposure, $c_i^t$ is public context, and $r_i^t$ records routing or validation choices not determined by the other inputs. The initial state and these inputs for every round determine the trajectory. Witness-only replay is valid only under an additional factorization of this transition through the logged witness/payload and context.
\end{Proposition}
\begin{proof}
The supplied inputs determine the first successor; induction determines every subsequent state. Under the additional factorization, substitute the witness-based transition at each step.
\end{proof}

\begin{Remark}[Enforcement as semantic projection]
The router normal-form results (Theorems~\ref{thm:routernf} and \ref{thm:routerclass}) make the ``output gate'' semantics precise: once enforcement rejects unauditable (empty-witness) steps, any auditable trigger protocol behaves \emph{as if} all of its non-fallback behavior were explicitly evidence-gated. In this sense, external enforcement projects arbitrary deliberation protocols down to the evidential \PBRC{} shape.
\end{Remark}

\begin{Remark}[Why audit equivalence matters]
Belief equivalence alone says two protocols behave the same internally. Audit equivalence strengthens this: an external compliance gate that only accepts nonempty-witness certificates will accept/reject the same updates under $P_i$ and under $\mathrm{NF}(P_i)$. Thus $\mathrm{NF}(P_i)$ is a canonical ``PBRC-shaped'' realization of the same auditable behavior.
\end{Remark}

\begin{Remark}[Interaction with minimality theorems]
If $P_i$ additionally satisfies evidence-free peer-pressure immunity and evidence-free non-amplification, then the social-only behavior absorbed into $\delta_i^{\mathrm{NF}}$ must be argmax-preserving and non-amplifying on social-only events. Combined with Theorems~\ref{thm:necessity_argmax}--\ref{thm:necessity_evidence}, the compiled fallback retains these behavioral properties. This supports the chosen representation, but does not imply that robust protocols must originally have evidence-gated trigger syntax.
\end{Remark}


\section{Adversary models and conditional robustness bounds}
\label{sec:adversary}
Evidence gating controls admissibility, not the truth or representativeness of evidence. The results below separate assumptions about evidence availability, authentication, semantic accuracy, context, and numerical execution. Language is inert under the token-invariance assumptions of Theorem~\ref{thm:tokensuff}; sound routing alone does not make every evidence-bearing contract insensitive to language.

\subsection{Threat taxonomy}
\label{sec:threat_taxonomy}
\begin{Definition}[Token appearance history]
\label{def:seen}
$\mathrm{Seen}_{<t}(\tau)$ holds when some message transmitted strictly before round $t$ contained $\tau$. Initial possession is recorded separately.
\end{Definition}
\begin{Definition}[No-new-valid-evidence adversary]
\label{def:adv_unforge}
An adversary has no new valid evidence if every valid token it transmits has either been recorded initially or appeared earlier in the run. This is an availability restriction, stronger than ordinary signature unforgeability: an adversary may obtain and transmit a genuinely signed output for the first time without forging anything.
\end{Definition}
\begin{Definition}[Evidence-interface forgery]
\label{def:adv_forge}
A forgery is an adversarial artifact accepted as authentic by the validity layer although it lacks the required authentic origin or authorization. An authentic but false tool result is a semantic integrity failure, not necessarily a signature forgery.
\end{Definition}
\begin{Definition}[Replay adversary]
\label{def:adv_replay}
A replay adversary retransmits issued tokens without modifying their authenticated contents, potentially outside their intended time, instance, or use count.
\end{Definition}
\begin{Definition}[Colluding adversary]
\label{def:adv_collude}
A colluding adversary coordinates controlled agents or validators, including their evidence selection and attestations.
\end{Definition}
\begin{Definition}[Omission/delay adversary]
\label{def:adv_omit}
An omission/delay adversary withholds, delays, or selectively forwards evidence. This can obstruct correction or cause biased evidence-bearing updates; it is not exclusively a liveness attack.
\end{Definition}

\subsection{Evidence-generation steering and omission attacks}
\label{sec:steering_omission}
Tokens are exogenous inputs to the present operational model. In an application, social messages can influence retrieval queries and tool inputs and thereby change which authentic tokens are generated. Selectively retaining supporting results can also produce a misleading evidence stream. The token-empty invariants survive omission, but correctness under evidence-bearing events needs additional assumptions. Authentication alone establishes neither balanced acquisition nor semantic truth.

A possible extension binds tokens to tool inputs and instance identifiers and adds a $\mathrm{QueryOK}$ predicate for a preregistered acquisition policy. Mandatory counter-evidence queries require an auditable record of required queries, returned results, and omissions. This is a proposed extension; the present experiments do not implement or validate it.

\subsection{No-evidence robustness}
\label{sec:unforgeability_guarantees}
\begin{Theorem}[Immunity in an initially token-empty system]
\label{thm:byzantine_unforge}
Suppose the initial system contains no valid tokens, neither honest agents nor the environment generate new valid tokens, and adversaries satisfy Definition~\ref{def:adv_unforge}. If honest agents use evidential contracts, sound enforcement, faithful numerical execution, and argmax-preserving fallback, their top sets remain unchanged. If fallback is also confidence-nonincreasing, their maximal confidence cannot increase.
\end{Theorem}
\begin{proof}
Induction on rounds shows that no transmitted token is valid: none exists initially, none is generated honestly, and an adversary cannot introduce a first valid token. Every exposure is therefore social-only. Evidential triggers are false, sound enforcement selects fallback, and the two conclusions follow by induction using the fallback properties.
\end{proof}

\subsection{Limits under evidence-interface compromise}
\label{sec:impossibility_forgery}
\begin{Theorem}[Conditional effect of an accepted forged witness]
\label{thm:imposs_forge}
Fix a reachable belief $b$ and an adversarially deliverable event $E$. Suppose branch $j$ is the full-event priority-selected branch, the adversary can supply a certificate passing every required check, the router accepts it, and numerical execution applies $U_j(b,E)$. If $U_j(b,E)$ has a specified adverse property (for example, a false unique maximizer), that property holds after the step even if the certificate's accepted evidence was forged.
\end{Theorem}
\begin{proof}
Acceptance and faithful numerical execution give $b'=U_j(b,E)$. Substitute this equality into the assumed adverse property.
\end{proof}
This conditional boundary result does not imply that forgeability alone forces a false belief. Rejection, a higher-priority branch, incomplete routing, or an operator that preserves correctness may prevent the hypothesized transition. Cryptographic integrity and semantic accuracy must be assessed separately.

\subsection{Replay and freshness}
\label{sec:replay_freshness}
\begin{Definition}[Freshness and context]
\label{def:fresh}
For an authenticated issue time $\mathrm{Time}(\tau)$ and window $\Delta\ge0$, define
\[
\mathrm{Fresh}_{\Delta}(\tau,t)\iff 0\le t-\mathrm{Time}(\tau)\le\Delta.
\]
A separate predicate $\mathrm{ContextMatch}(\tau,c)$ binds the token to the current task, schema, and instance $c$. Future-dated tokens fail freshness.
\end{Definition}
\begin{Theorem}[Expiry of replayed evidence]
\label{thm:replay_bound}
Suppose every accepted evidence token must be valid, fresh, and context-matching, and replay cannot change authenticated timestamps. A token issued at $r$ cannot be used at $t<r$ or $t>r+\Delta$. If all available tokens have issue times at most $s$ and no later tokens become available, replay alone cannot enable an evidential branch at any $t>s+\Delta$.
\end{Theorem}
\begin{proof}
The first assertion is exactly the freshness inequality. For the second, every available token has $t-\mathrm{Time}(\tau)>\Delta$, so no nonempty accepted evidence witness exists.
\end{proof}
Freshness does not prevent repeated use within the window. Avoiding repeated likelihood multiplication requires deduplication, an idempotent operator, or an authenticated spent-token ledger included in the state and replay record. Fixed-validity flooding results do not directly cover expiry or such additional state.

\subsection{Multi-attestation under explicit integrity assumptions}
\label{sec:multi_attestation}
\begin{Definition}[$k$-of-$K$ validity]
\label{def:kofk_validity}
Let $\mathrm{Attest}(\tau,s)$ mean a verified attestation on the complete token and context from authenticated validator identity $s\in\{1,\ldots,K\}$. Define
\[
\mathrm{Valid}_k(\tau)\iff |\{s:\mathrm{Attest}(\tau,s)\}|\ge k,\qquad 1\le k\le K.
\]
Repeated signatures from the same identity count once.
\end{Definition}
\begin{Theorem}[Independent-compromise bound]
\label{thm:forge_bound}
Over a fixed attack horizon, suppose each validator is compromised independently with probability $p$, uncompromised signatures cannot be forged, and uncompromised validators never attest to the target false assertion in its bound context. Then the probability that the adversary can obtain $k$ accepted distinct attestations for that false assertion is at most
\[
\sum_{j=k}^{K}\binom Kj p^j(1-p)^{K-j}.
\]
\end{Theorem}
\begin{proof}
Success requires at least $k$ compromised identities. Their count has the $\mathrm{Binomial}(K,p)$ distribution. The success event is a subset of its upper-tail event.
\end{proof}
\begin{Proposition}[Deterministic compromise threshold]
\label{prop:f_threshold}
Under the same authentication and honest-validator truth assumptions, at most $f$ compromised validators cannot authorize the target false assertion when $k>f$.
\end{Proposition}
\begin{proof}
Only compromised validators can supply an accepted attestation to that assertion, giving at most $f$ distinct identities.
\end{proof}
These bounds do not cover correlated semantic errors, shared sources, or honest validators deceived into signing false claims. A numerical attestation count supplied by an untrusted model is not authenticated multi-attestation.

\subsection{A conditional failure taxonomy}
\label{sec:failuretax}
\begin{Definition}[Correct belief state]
\label{def:correct}
Fix $h^\star\in\Hh$. Write $\mathrm{Correct}(b)$ iff $h^\star\in\argmax(b)$. This allows ties and does not guarantee that a separate tie-breaking answer rule returns $h^\star$.
\end{Definition}
\begin{Definition}[Applicability and token soundness]
\label{def:appsound}
At time $t$ and context $c_t$, let
\[
\mathrm{App}(\tau,t)=\mathrm{Valid}(\tau)\land\mathrm{Fresh}_\Delta(\tau,t)
\land\mathrm{ContextMatch}(\tau,c_t).
\]
Let $\mathrm{TokSound}_{h^\star}(\tau)$ mean that $\tau$ neither contradicts $h^\star$ nor supports a false hypothesis. This is a truth-relative specification, not a computable consequence of signature validity.
\end{Definition}
\begin{Definition}[Evidence integrity and contract soundness]
\label{def:integrity_soundness}
Evidence integrity means that every applicable token is truth-sound. A contract is $(h^\star,\Delta)$-sound if, for every correct input belief, full event, and accepted witness for the priority-selected branch, applicability and truth-soundness of all witness tokens imply correctness of the specified output operator. This condition quantifies over the whole event whenever the operator uses it. Fallback is separately required to preserve correctness.
\end{Definition}
\begin{Theorem}[First-error classification]
\label{thm:failure_taxonomy}
Assume an initially correct state, an evidential contract, argmax-preserving fallback, and a gate rejecting empty accepted witnesses. At a first transition from a correct state to an incorrect one, at least one of the following holds:
\begin{enumerate}
\item numerical execution differs from the specified accepted operator or fallback;
\item an accepted certificate fails full-event truth, priority, witness correctness, or valid-token membership;
\item an accepted witness token fails freshness or context matching;
\item an applicable accepted witness token is not truth-sound;
\item the selected operator fails contract soundness on the actual event and input state.
\end{enumerate}
If numerical execution is faithful, the first incorrect transition is non-fallback and carries a nonempty accepted witness.
\end{Theorem}
\begin{proof}
If execution is unfaithful, case 1 holds. Otherwise an argmax-preserving fallback cannot remove $h^\star$ from the top set. The step therefore accepted a branch with a nonempty witness. If any certificate check fails, case 2 holds. Otherwise its tokens are valid. If a token is stale or context-mismatched, case 3 holds; if all are applicable but one is not truth-sound, case 4 holds. In the remaining case all witness tokens are applicable and truth-sound, yet the selected specified operator maps a correct belief to an incorrect belief. This is precisely case 5.
\end{proof}
This is an exhaustive conditional classification, not an algorithm that identifies a unique cause from a certificate. The cases can overlap; evaluating truth-soundness requires external ground truth, and omitted evidence may be causally relevant.

\begin{Corollary}[Conditional preservation of correctness]
\label{cor:safety_taxonomy}
If execution is faithful, routing is sound, all accepted witnesses are applicable, evidence integrity holds, the contract is $(h^\star,\Delta)$-sound, and fallback preserves argmax, an initially correct run remains correct.
\end{Corollary}
\begin{proof}
Each possible first-error case in Theorem~\ref{thm:failure_taxonomy} is excluded.
\end{proof}
\begin{Corollary}[First-error trace under equal exposures]
\label{cor:endtoend_diagnosis}
Under all assumptions of Theorem~\ref{thm:endtoend}, two topologies with equal incoming token exposure traces and the same initial beliefs have identical beliefs and canonical accepted certificates. Consequently the first loss of correctness, if any, has the same index and canonical certificate in both. If the applicability annotations also agree, the classification above agrees as a truth-relative statement.
\end{Corollary}
\begin{proof}
Apply the equality of belief and canonical certificate traces from Theorem~\ref{thm:endtoend}, then evaluate the same correctness predicate and classification on their identical inputs. Equal positive-walk reachability sets suffice for equal exposures by Theorem~\ref{thm:topology_trace_equiv}.
\end{proof}

\section{Contractual Dynamic Doxastic Logic (\CDDL{}) with iteration}
\label{sec:cddl}
\subsection{Positioning and semantic interface}
\label{sec:cddl_position}
\CDDL{} combines KD45 belief modalities with PDL programs and iteration as a specification language; the modal ingredients are standard \cite{harel2000dl,vanbenthem2015dlbc,leitgeb2007ddl,schmidt2008pdl_doxastic}. Our contribution here is the explicit interface to PBRC decisions and audit states. We prove soundness and exact checking on supplied finite models. We do not claim a new all-model completeness theorem or finite model property for the combined system.

\subsection{Instrumented executions and abstraction atoms}
\label{sec:cddl_bridge}
\begin{Definition}[Abstraction map and action labeling]
\label{def:abstraction}
For a recorded run with $T$ transitions, let states be indices $0,\ldots,T$. At each $t<T$ record the full event $E_i^t$, input belief $b_i^t$, router decision, certificate, and output $b_i^{t+1}$. Set
\begin{align*}
\mathsf{Trig}_{i,j}(t)&\iff\mathcal S(E_i^t)\models\varphi_{i,j},&
\mathsf{SocOnly}_i(t)&\iff\mathrm{SocialOnly}(E_i^t),\\
\mathsf{Top}_{i,h}(t)&\iff h\in\argmax(b_i^t),&
\mathsf{Acc}_{i,j}(t)&\iff\text{the router accepts branch }j.
\end{align*}
Define $\mathsf{Acc}_i=\bigvee_j\mathsf{Acc}_{i,j}$. Acceptance atoms label the \emph{source} state of the decision. Put $(t,t+1)$ in $R_{\mathsf{upd}_{i,j}}$ exactly when branch $j$ is accepted, and in $R_{\mathsf{fb}_i}$ exactly when fallback executes. For each agent and recorded step exactly one such edge is required. At $T$, top atoms use $b_i^T$; no outgoing step exists and the unused event/acceptance atoms can be set false.
\end{Definition}
This abstraction records behavior even if a contract violation is found; validation of that behavior is a separate audit check. The numerical probability vector does not determine a KD45 accessibility relation. Checking formulas containing $B_i$ requires an explicitly supplied epistemic model. Practical run-log monitoring below uses the fragment without $B_i$.

\subsection{Syntax}
\label{sec:cddl_syntax_semantics}

Fix agents $\A$, propositional atoms $\mathsf{Prop}$, and atomic actions $\mathsf{Act}_0$.

\paragraph{Formulas.}
\[
\varphi ::= p \mid \neg\varphi \mid (\varphi\land\varphi) \mid B_i\varphi \mid [\alpha]\varphi,
\]
where $p\in\mathsf{Prop}$, $i\in\A$, and $\alpha$ is a program. As usual define $\langle\alpha\rangle\varphi := \neg[\alpha]\neg\varphi$.

\paragraph{Programs (PDL).}
\[
\alpha ::= a \mid \alpha;\beta \mid \alpha\cup\beta \mid \varphi? \mid \alpha^\ast,
\]
where $a\in\mathsf{Act}_0$, $\alpha^\ast$ is Kleene star, and $\varphi?$ is a test.

\subsection{Semantics}

A \CDDL{} model is a Kripke structure
\[
M=(W,\{R_i\}_{i\in\A},\{R_a\}_{a\in\mathsf{Act}_0},V),
\]
where each $R_i$ is serial, transitive, and Euclidean (KD45 frames), each $R_a\subseteq W\times W$ interprets an atomic action, and $V$ is a valuation.
Program relations extend compositionally:
\begin{align*}
R_{\alpha;\beta} &:= R_\alpha \circ R_\beta,\qquad
R_{\alpha\cup\beta} := R_\alpha \cup R_\beta,\qquad
R_{\psi?} := \{(w,w): M,w\models \psi\},\\
R_{\alpha^\ast} &:= (R_\alpha)^\ast \quad \text{(reflexive transitive closure)}.
\end{align*}
Truth is defined by the standard clauses for KD45 belief and PDL:
\[
M,w\models B_i\varphi \iff \forall w'\,(wR_iw'\Rightarrow M,w'\models\varphi),\qquad
M,w\models[\alpha]\varphi \iff \forall w'\,(wR_\alpha w'\Rightarrow M,w'\models\varphi).
\]

\subsection{Encoding accepted decisions as programs}
Let $\mathsf{None}_i=\neg\bigvee_j\mathsf{Acc}_{i,j}$ and define
\[
\pi_i=\bigcup_{j=1}^{k_i}(\mathsf{Acc}_{i,j}?;\mathsf{upd}_{i,j})
\ \cup\ (\mathsf{None}_i?;\mathsf{fb}_i).
\]
Write $\pi(C_i)=\pi_i$ when the contract is understood. This encoding covers sound-but-incomplete routers, where a true trigger can still lead to fallback. A program guarded only by raw trigger truth would omit such fallback steps. For a complete router, mutually exclusive priority guards derived from $\mathsf{Trig}_{i,j}$ agree with acceptance; this agreement must be checked, not assumed for every router.

\begin{Lemma}[Run-to-program bridge]
\label{lem:run_bridge}
Under Definition~\ref{def:abstraction}, $R_{\pi_i}=\{(t,t+1):0\le t<T\}$. Hence $R_{\pi_i^\ast}=\{(t,u):0\le t\le u\le T\}$.
\end{Lemma}
\begin{proof}
Every recorded step has exactly one accepted branch or fallback. Its guard is true and its labeled edge is present, so the step is in $R_{\pi_i}$. Conversely, each branch consists of an identity test followed by one of the recorded edges; no other edge can be introduced. Reflexive transitive closure of the consecutive-index relation yields the second equality.
\end{proof}

\subsection{PBRC invariants expressed in \CDDL{}}
\label{sec:cddl_examples}
For compactness define the following formula, with $h$ ranging over the fixed hypothesis set:
\[
\mathsf{Stable}_i:=\bigwedge_h\left[
 (\mathsf{Top}_{i,h}\to[\pi_i]\mathsf{Top}_{i,h})\land
 (\neg\mathsf{Top}_{i,h}\to[\pi_i]\neg\mathsf{Top}_{i,h})\right].
\]
Both implications are necessary: preserving existing maximizers alone permits new ties. Social-only equality of the entire top set is expressed by
\begin{equation}
\label{eq:cddl_social_invariant}
[\pi_i^\ast](\mathsf{SocOnly}_i\to\mathsf{Stable}_i).
\end{equation}
With argmax-preserving fallback, accountability can be expressed at the source state as
\begin{equation}
\label{eq:cddl_accountability}
[\pi_i^\ast](\neg\mathsf{Acc}_i\to\mathsf{Stable}_i).
\end{equation}
Equivalently, any change in the complete top set requires an accepted non-fallback decision at that step's source. Sound routing must separately establish that the decision has a valid nonempty witness for the full-event selected trigger. Testing a trigger atom after the step would instead inspect the next event and would not certify the transition under review.

These formulas hold under the corresponding operational assumptions by the bridge lemma and the fallback theorems. On a terminal logged state, boxes for one-step actions are vacuously true. A completed finite-prefix check certifies only recorded transitions; it does not establish what happens after the log ends.

\subsection{Axioms and proof rules}

We use a standard Hilbert system consisting of:
\begin{itemize}
\item propositional tautologies and modus ponens;
\item KD45 axioms for each $B_i$ (normality $\mathbf{K}$, seriality $\mathbf{D}$, positive introspection $\mathbf{4}$, negative introspection $\mathbf{5}$) and necessitation for $B_i$;
\item normality for dynamic modalities and necessitation for $[\alpha]$;
\item PDL reduction axioms:
\begin{align*}
\mathbf{Seq}:&\ [\alpha;\beta]\varphi \leftrightarrow [\alpha][\beta]\varphi,\\
\mathbf{Choice}:&\ [\alpha\cup\beta]\varphi \leftrightarrow ([\alpha]\varphi\land[\beta]\varphi),\\
\mathbf{Test}:&\ [\psi?]\varphi \leftrightarrow (\psi\to \varphi);
\end{align*}
\item the star axiom and induction rule:
\begin{align*}
\mathbf{Star}:&\ [\alpha^\ast]\varphi \leftrightarrow \big(\varphi \land [\alpha][\alpha^\ast]\varphi\big),\\
\mathbf{Ind}:&\ \text{from } (\varphi \to [\alpha]\varphi)\ \text{ infer } (\varphi \to [\alpha^\ast]\varphi).
\end{align*}
\end{itemize}

Necessitation and $\mathbf{Ind}$ apply to theorems (globally valid premises), not to a premise assumed only at one state. No interaction axioms between numerical beliefs, $B_i$, and action relations are postulated.

\subsection{Soundness and finite-model checking}
\label{sec:cddl_results}\label{sec:cddl_sc}
\begin{Theorem}[Soundness of \CDDL{}]
\label{thm:cddl_sc}
Every theorem of the stated proof system is valid in every model with serial, transitive, Euclidean doxastic relations and the specified PDL program semantics, including reflexive transitive closure for iteration.
\end{Theorem}
\begin{proof}
Each axiom is valid and each rule preserves global validity; Appendix~\ref{app:pdl} proves these statements and induction on derivations yields the conclusion.
\end{proof}
\begin{Theorem}[Exact finite-model monitoring]
\label{thm:finite_monitor}
Given an explicit finite model, formula valuations and program relations can be computed exactly by mutual structural recursion, with finite relational closure for star. In particular, the $B_i$-free formulas above are decidable on any finite instrumented run. Including $B_i$ is possible when its accessibility relations are supplied.
\end{Theorem}
\begin{proof}
See Appendix~\ref{app:pdl} for the algorithm and its structural correctness proof. This result concerns a given finite model, not completeness of the Hilbert system over all models.
\end{proof}

\paragraph{Audit procedure.}
First validate log integrity, full-event annotations, priority selection, certificates, and numerical updates using stored inputs and versioned validators. Next construct the run graph and check that every recorded step has exactly one correctly labeled edge. Evaluate the source atoms, then the formulas. For a failed invariant, report its earliest violating source index and actual successor together with the rejected check. A violated assumption can be located in the log; determining whether a token's semantic assertion is true may still require external evidence.

\section{Design and implementation guidelines}
\label{sec:design}

\PBRC{} is intentionally \emph{protocol-level}: it specifies which public artifacts (contracts, tokens, certificates, and audit traces) must exist for belief change to be admissible and verifiable, while remaining agnostic to how agents generate messages. The goal is to preserve open communication while making \emph{changes of mind} evidence-gated, auditable, and (under a state-holding router) enforceable.

This section distills implementation guidance that preserves the paper’s formal guarantees in deployed multi-agent systems. The recurring theme is to (i) keep triggers witnessable by small token sets, (ii) separate cryptographic validity from semantic applicability, and (iii) choose conservative operators and fallbacks that make safety robust to missing or delayed evidence.

\subsection{Contract patterns (templates)}
\label{sec:patterns}

In practice, many contracts can be assembled from a small library of reusable patterns. For predictability and efficiency, these are best restricted to token-witnessable fragments such as $\FO^{\mathsf{Tok}}_{\exists\wedge}$ (Definition~\ref{def:ec}), where admissibility can be certified by a small witness set and checked locally (Propositions~\ref{prop:linearcheck}--\ref{prop:canon_wit}).

\textbf{Verified contradiction (falsifier pattern).}
A conservative default is to permit decreases in belief mass (or argmax flips away from a hypothesis) only when a validated token \emph{contradicts} that hypothesis, e.g.,
\[
\exists \tau\; \big(\mathrm{Valid}(\tau)\wedge \mathrm{Contradicts}(\tau,h)\big).
\]
Operationally, contradiction triggers are the primary safeguard against ``wrong-but-sure'' cascades driven by social pressure: if no validated contradictor exists, social exchange cannot justify moving probability mass aggressively.

\textbf{Verified support (bounded confirmation).}
When positive evidence is available, include support triggers of the form
\[
\exists \tau\; \big(\mathrm{Valid}(\tau)\wedge \mathrm{Supports}(\tau,h)\big),
\]
but couple them with \emph{bounded} operators (e.g., capped log-odds steps or margin-based updates) to avoid runaway overconfidence from a single token or noisy verifier. In many settings, it is appropriate to require stronger conditions for \emph{increasing} confidence than for decreasing it (e.g., support requires $k$ independent attestations; contradiction requires one strong falsifier).

\textbf{Freshness gating (replay resistance).}
To control replay and context drift, augment evidence triggers with freshness predicates:
\[
\exists \tau\; \big(\mathrm{Valid}(\tau)\wedge \mathrm{Supports}(\tau,h)\wedge \mathrm{Fresh}_\Delta(\tau,t)\big),
\]
and analogously for contradiction. This implements the theoretical replay bounds (Theorem~\ref{thm:replay_bound}) and is typically cheap in the logical layer (constant-time arithmetic checks per token).

\textbf{Multi-attestation (robustness to validator compromise).}
When the validity layer admits attestations, replace single-token conditions by quorum-like requirements (``$k$-of-$K$'') at the validity layer (Section~\ref{sec:adversary}). This reduces the impact of any single compromised validator or verifier and yields an explicit robustness--liveness trade-off: larger $k$ improves integrity but may delay admissible revision when evidence is sparse.

\textbf{Surprise clauses (novel or unmodeled evidence).}
Because preregistration can be brittle, include a low-priority clause that activates when a validated token is present but does not match existing templates, e.g.,
\[
\exists \tau\; \big(\mathrm{Valid}(\tau)\wedge \mathrm{Unmodeled}(\tau)\big).
\]
Within the fixed-state model the corresponding operator should conservatively hold or dilute belief. Deferral actions, tool calls, and hypothesis expansion require extensions to the state and action semantics (Section~\ref{sec:llm_inst}). Importantly, surprise clauses do not weaken the anti-conformity story: in the token-empty regime they never fire, and social-only exchange remains routed to fallback.

\textbf{Budget clauses (availability and DoS resilience).}
A practical attack surface is flooding events with many tokens (valid but irrelevant) to inflate validation latency. Two complementary mitigations are recommended:
(i) use positive witnessable triggers so checking can stop once the highest-priority branch is established; lower-priority branches also require ruling out all higher-priority triggers (Section~\ref{sec:sim_cost}); and
(ii) enforce explicit per-round validation budgets with a safe defer/hold fallback when budgets are exceeded. Budget predicates are naturally router-side (policy-level) and can be treated as part of the enforcement environment rather than the agent’s private computation.

\subsection{Token validity and semantic applicability}
\label{sec:token_design}

A common deployment failure is to conflate \emph{cryptographic validity} with \emph{semantic applicability}. In \PBRC{}, the validity layer is the boundary at which evidence becomes admissible; if this boundary is underspecified, the system becomes vulnerable to ``citation laundering'' (valid provenance, irrelevant content) or to entailment mislabeling (valid tokens, incorrect Supports/Contradicts mapping).

A robust token schema should therefore bind evidence to task context and support auditing. Concretely, tokens should (at minimum) include:
(i) a provenance binding (signatures/attestations),
(ii) a timestamp (for freshness checks),
(iii) a context binding (instance identifier, prompt hash, or task context hash),
and (iv) typed content (tool output schema, retrieval snippet hash, verifier judgment metadata).
At the predicate level, the router should be able to expose at least:
\[
\mathrm{Valid}(\tau),\quad \mathrm{Fresh}_\Delta(\tau,t),\quad \mathrm{Type}(\tau,\cdot),
\]
and, when feasible, a separate predicate capturing \emph{applicability} (e.g., ``the snippet is about the claim''), distinct from semantic entailment.
When entailment-based labeling is used, auditability improves substantially if tokens record the verifier identity/version and the exact inputs used to compute the label (so that the Supports/Contradicts relation is reproducible).

Finally, because social influence can act indirectly by steering evidence generation (biased queries, selective tool calls), the validity layer should, where possible, bind tokens to the \emph{inputs} that generated them (query hashes, tool-call arguments). This supports query-policy compliance mechanisms (Section~\ref{sec:steering_omission}) and helps prevent conformity-by-proxy.

\subsection{Worked contract: peer-pressure misinformation triage}
\label{sec:worked_contract}

We describe a proposed contract for a binary misinformation triage task in which peer interaction may be persuasive but unreliable.

\textbf{Hypotheses.}
Let $\Hh=\{\textsf{real},\textsf{fake}\}$ denote the claim status.

\textbf{Token families and validity layer.}
We use two token families:
(1) \emph{tool tokens} $\tau_{\mathrm{tool}}$ representing signed outputs of a fact-checking tool/API (structured verdict + provenance),
and (2) \emph{retrieval tokens} $\tau_{\mathrm{ret}}$ binding a snippet hash, URL, and timestamp to a retrieval procedure.
The validity layer verifies provenance, checks freshness, binds tokens to the instance context, and exposes $\mathrm{Supports}(\cdot,\cdot)$ / $\mathrm{Contradicts}(\cdot,\cdot)$ predicates (e.g., via a pinned verifier model with logged inputs/outputs).

\textbf{Triggers and priority.}
A conservative contract prioritizes high-integrity tool evidence over weaker retrieval evidence. One example trigger set is:
\begin{align*}
\varphi_{\mathrm{tool\_con}} &:= \exists \tau\; (\mathrm{App}(\tau,t)\wedge \mathrm{Type}(\tau,\mathrm{tool})\wedge \mathrm{Contradicts}(\tau,\textsf{real})),\\
\varphi_{\mathrm{tool\_sup}} &:= \exists \tau\; (\mathrm{App}(\tau,t)\wedge \mathrm{Type}(\tau,\mathrm{tool})\wedge \mathrm{Supports}(\tau,\textsf{real})),\\
\varphi_{\mathrm{ret\_con}} &:= \exists \tau\; (\mathrm{App}(\tau,t)\wedge \mathrm{Type}(\tau,\mathrm{ret})\wedge \mathrm{Contradicts}(\tau,\textsf{real})\wedge \mathrm{Fresh}_\Delta(\tau,t)).
\end{align*}
A typical priority policy checks tool-based triggers first and treats retrieval-based contradiction as weaker, freshness-gated evidence. (Equivalently, the priority order $\varphi_{\mathrm{tool\_con}} \prec \varphi_{\mathrm{tool\_sup}} \prec \varphi_{\mathrm{ret\_con}}$ is chosen so that tool triggers dominate retrieval triggers under the contract’s selection rule.)

\textbf{Operators and fallback.}
Operators should be bounded and auditable (e.g., capped log-odds steps). For example, $\varphi_{\mathrm{tool\_con}}$ applies a strong bounded decrease to $b(\textsf{real})$, $\varphi_{\mathrm{tool\_sup}}$ applies a bounded increase, and $\varphi_{\mathrm{ret\_con}}$ applies a smaller bounded decrease.
In trigger-inactive rounds (including social-only exchange), the enforced update is fallback; an argmax-preserving skeptical dilution $\delta_\lambda$ is a safe default when the objective is to prevent confidence escalation under persuasion.

\textbf{Certificates and post-hoc audit.}
Any top-hypothesis change is accompanied by a nonempty witness set (Theorem~\ref{thm:accountability}) identifying the validated token(s) that licensed the change. If the system later exhibits a first incorrect flip, the failure localization theorem (Theorem~\ref{thm:failure_taxonomy}) narrows the diagnosis to evidence integrity failures (forgery/compromise), contextual validity failures (replay/staleness), router unsoundness, or contract/operator misspecification, with numerical operator noncompliance included as a separate possible failure.

\subsection{Supplementary finite token-set implementation}
\label{sec:reference_impl}
We implement a state-holding reference router for finite token-set predicates. It selects the highest-priority true predicate on the full validated set, requires a nonempty valid witness satisfying the selected predicate on restriction, executes the registered operator, and checks the numerical output before committing state. The default witness is the full validated set; this is sound for the implemented token-set semantics but does not promise a small certificate. Custom extractors are checked independently. An incomplete decision rejects the selected branch and uses fallback; it does not discard tokens and then accept a lower-priority branch. Coverage is reported as unobserved if the supplied horizon ends before closure.

The validator and its version identifier are mandatory, and tokens with conflicting contents under one identifier are rejected. Logs retain full events, validation decisions, proposed and executed certificates, failures, and before/after beliefs. Deterministic replay additionally requires immutable callback code, parameters, and context. The 27 regression tests cover witness sufficiency, full-event truth, priority, validation, contextual freshness, numerical failures, state integrity, random-state separation, and knowledge/exposure distinctions. An invalid fallback raises without committing a numerical state.

The implementation assumes well-typed input objects and trusted local callbacks; it implements neither signatures nor arbitrary first-order satisfaction, and its logs are not tamper-evident. The probability-vector check permits a $10^{-12}$ rounding tolerance in normalization. Thus these finite checks support the implementation within its stated scope, not a floating-point proof of the real-arithmetic theorems or a production-security claim.

\subsection{Contract verification and synthesis}

Contracts can be verified against global safety invariants expressed in \CDDL{} (Section~\ref{sec:cddl}) and audited post hoc using full inputs, decisions, certificates, and numerical state transitions (Proposition~\ref{prop:replay_inputs}). A finite log check establishes only that recorded execution; a global guarantee requires a proof or an explicit transition model covering all allowed executions. Beyond manual design, an important research direction is \emph{contract synthesis}: propose candidate trigger/operator sets from logs and desired invariants (e.g., ``no top-hypothesis flips without evidence''), then validate candidates by model checking against the contract logic and by empirical stress tests under verifier noise, incomplete routers, and adversarial token floods. Even in the absence of fully automated synthesis, the template library above provides a practical starting point for systematic contract engineering.

\section{Verification complexity and cost model}
\label{sec:complexity}
\subsection{Separate costs and complete checks}
Let $n=|\mathcal T(E)|$, let $k$ count quantified token variables, and let $L$ denote formula encoding size. The cost of reading the event and authenticating its tokens is separate from logical checking. A complete accepted decision must establish event membership and validity of witness tokens, restricted witness satisfaction, full-event trigger truth, priority selection, and correct numerical execution. A schematic accounting is
\[
C_{\rm step}=C_{\rm parse}+C_{\rm validate}+C_{\rm witness}
+C_{\rm full/priority}+C_{\rm update}+C_{\rm log}.
\]
For monotone positive triggers a supplied witness can establish full-event truth. It does not by itself show that all higher-priority triggers are false. With nonmonotone formulas, even full-event truth of the named trigger needs an additional check. Signature checking and semantic labeling costs depend on the chosen implementation; they are not constant by assumption here.

\subsection{A guarded existential conjunctive fragment}
\begin{Definition}[Guarded token-existential conjunctive triggers]
\label{def:ec}
A trigger in $\FO^{\mathsf{Tok}}_{\exists\wedge}$ has the form
\[
\chi=\exists\tau_1\cdots\exists\tau_k\left(
\bigwedge_{r=1}^k\mathrm{Valid}(\tau_r)\land
\bigwedge_{m=1}^M R_m(\bar\tau_m,\bar c_m)\right),\qquad k\ge1.
\]
The $R_m$ are positive atomic predicates, with equality and disequality permitted, over token variables and fixed public parameters. Their interpretations are fixed under token restriction and do not encode aggregates that change when other tokens are added. Each quantified variable is validity-guarded. Token-only use additionally requires that all predicates and parameters agree across the compared events.
\end{Definition}
Freshness is permitted when time and context are fixed for the comparison; changing either invalidates that comparison. A threshold over authenticated distinct validator identities is not equivalent to counting arbitrary token objects.

\begin{Proposition}[Assignment certificates and bounded witnesses]
\label{prop:linearcheck}
If $E\models\chi$, a satisfying assignment $\sigma=(\tau_1,\ldots,\tau_k)$ has image $W\subseteq\mathcal T(E)$ with $1\le |W|\le k$, and $W$ witnesses $\chi$. Given $\sigma$, its validity and conjuncts can be checked in $O(k+L)$ time after constant-time indexed predicate and membership access, excluding token authentication costs. For bounded arity and $M$ conjuncts this is $O(k+M)$.
\end{Proposition}
\begin{proof}
Validity guards place each assigned value in $\mathcal T(E)$; $k\ge1$ makes its image nonempty. Restriction to that image preserves all assigned atomic facts and fixed parameters, so the same assignment satisfies the restricted formula. Reading the assignment and checking every conjunct once gives the stated cost.
\end{proof}
A witness \emph{set} alone need not identify which token fills which variable. Reconstructing an assignment may require search; enumerating $n^k$ tuples gives a simple $O(n^k(k+L))$ bound. Thus inexpensive verification of a supplied assignment must not be confused with uniformly efficient witness search when $k$ is part of the input.

\begin{Proposition}[Canonical assignment witnesses]
\label{prop:canon_wit}
Fix a public total order on token identifiers. On a satisfying finite event, choose the lexicographically least satisfying $k$-tuple and return its image and assignment payload. On an unsatisfying event return the empty result. This defines a sound nonvacuous witness extractor; it is token-invariant when the underlying predicates and parameters are token-invariant.
\end{Proposition}
\begin{proof}
There are finitely many tuples, so a nonempty set of satisfying tuples has a unique least member. Proposition~\ref{prop:linearcheck} proves witness correctness and nonemptiness. Equal annotated token inputs yield the same satisfying tuples, public order, least tuple, and image. Social-only inputs have no tuple satisfying the validity guards.
\end{proof}
Canonical selection can require more work than finding any witness. The existence theorem does not assert an efficient canonicalization algorithm for variable-size queries.

\subsection{General first-order triggers and caching}
A fixed first-order query on a finite explicitly represented structure can be evaluated by recursively enumerating quantified variables, in polynomial data complexity with an exponent depending on the query. This does not give a uniform polynomial bound in query size or guarantee small token witnesses for unrestricted formulas. Full-event negation, universal conditions, and priority checks may require examining data outside an accepted witness.

Under the full assumptions of token-trace equivalence, replay can use complete annotated token exposures, the initial belief, contracts, deterministic operator semantics, and router decisions (or deterministic complete selection). Proposition~\ref{prop:replay_inputs} explains why accepted witnesses alone are insufficient. Cache keys must include the relevant context, current time, validator/schema version, and revocation state. Authentication of immutable contents may be reusable while freshness and applicability still require reevaluation.

Budget exhaustion should cause a logged fallback or deferral decision. Early stopping is justified for a positive highest-priority trigger after a sufficient assignment is found; accepting a lower-priority trigger additionally requires certifying the absence of a higher-priority match. Arbitrary first-order contracts do not inherit the short-circuit cost of the simple toy experiment.

\section{Experimental evaluation}
\label{sec:eval}
We evaluate PBRC through controlled simulations, a matched synthetic study of label errors, a token-empty KAIROS analysis, and an evidence-enabled SciFact experiment. These settings examine different aspects of protocol behavior: social amplification, evidence propagation, validator error, and answer admission. Results are analyzed separately because the tasks, models, and observables differ.

\subsection{Metrics and reproducibility}
\label{sec:eval_metrics}
Confidence is $\conf(b)=\max_h b(h)$; accuracy and harmful/beneficial flips require ground truth. Cascade rates use the specific final-state definition below. Cost results count token checks, not wall-clock latency or cryptographic operations. Wilson 95\% intervals accompany the cascade proportions in Table~\ref{tab:cascade_rates}.

The simulation code builds on \url{https://github.com/alqithami/PBRC}, commit \texttt{37bbe372b75d5acb8c2df68087900161ca260566}. The accompanying source package contains the executable implementations, configurations, raw results, and figure-generation inputs used here. The simulation environment is Python 3.12.14, NumPy 2.3.5, pandas 2.2.3, NetworkX 3.7, and SciPy 1.17.0. Simulations I--V use seed 0; the matched study specifies its seed and repeated-trial design below.

The simulation validator is a trusted callback, and the simulation router computes transitions from caller-supplied state. The reference implementation in Section~\ref{sec:reference_impl} additionally holds and validates its own numerical state. Both implementations have finite token-set semantics; production authentication, semantic validation, and arbitrary first-order enforcement require additional components.

\subsection{Simulation I: social-only wrong-but-sure cascades are eliminated by PBRC}
\label{sec:sim_cascades}

This experiment targets the primary failure mode PBRC is designed to prevent: \emph{token-empty} deliberation (no externally validated evidence) producing a confident but incorrect group outcome via social reinforcement.

\noindent\textbf{Setup.}
We consider $n=20$ agents and a binary hypothesis space $\Hh=\{h_0,h_1\}$ with ground truth $h_0$.
Each agent maintains an explicit belief $b_i^t\in\Delta(\Hh)$.
Initial beliefs are drawn i.i.d.\ as $b_i^0(h_0)\sim \mathrm{Beta}(5,5)$, placing agents near the competence boundary so that either hypothesis can dominate early aggregates.
We compare two dynamics over $T=10$ rounds:

\begin{enumerate}
\item \textbf{Baseline (social pooling + confidence amplification).}
On a graph $G$, agent $i$ forms a neighbor aggregate $s_i^t$ and updates by a degree-modulated geometric mixture followed by sharpening:
\[
b_i^{t+1}\;\propto\;\big(b_i^t\big)^{1-w_i}\cdot\big(s_i^t\big)^{w_i},
\qquad
b_i^{t+1}\leftarrow \mathrm{Sharpen}_\gamma(b_i^{t+1}),
\]
where $s_i^t=\frac{1}{|N(i)\cup\{i\}|}\sum_{j\in N(i)\cup\{i\}} b_j^t$ is the arithmetic mean of neighbor beliefs,
$w_i=\min\{0.95,\,w_0+w_s\cdot \deg(i)/(n-1)\}$ controls susceptibility to neighbors, and
\[
\mathrm{Sharpen}_\gamma(p)(h):=\frac{p(h)^\gamma}{\sum_{h'\in\Hh} p(h')^\gamma}
\]
models confidence amplification.
\item \textbf{PBRC (evidential contract in token-empty rounds).}
No validated tokens are present (all events are \textsc{SocialOnly}).
Routers reject empty-witness certificates, hence no trigger fires and the enforced update is fallback.
We use the argmax-preserving skeptical-dilution fallback
$b_i^{t+1}=(1-\lambda)b_i^t+\lambda\cdot \mathrm{Unif}$,
which satisfies the social-only non-amplification conditions in Theorem~\ref{thm:nocascade}.
\end{enumerate}
Unless otherwise stated, $(w_0,w_s,\gamma,\lambda)=(0.4,0.5,2.0,0.1)$.

\noindent\textbf{Representative run.}
Figure~\ref{fig:sim_example_run} shows a representative complete-graph trajectory in which a minority of agents begins confidently wrong.
The baseline dynamics amplify confidence and can pull the population toward the confident minority, illustrating a wrong-but-sure cascade mechanism driven purely by social reinforcement.
This illustration uses a different initialization from the cascade-rate trials: 75\% of agents have true-hypothesis probabilities uniform on $[0.55,0.65]$, and 25\% have probabilities uniform on $[0.01,0.05]$. PBRC preserves each initial top hypothesis while diluting confidence; it does not correct initially wrong agents.

\begin{figure}[t]
\centering
\includegraphics[width=0.48\linewidth]{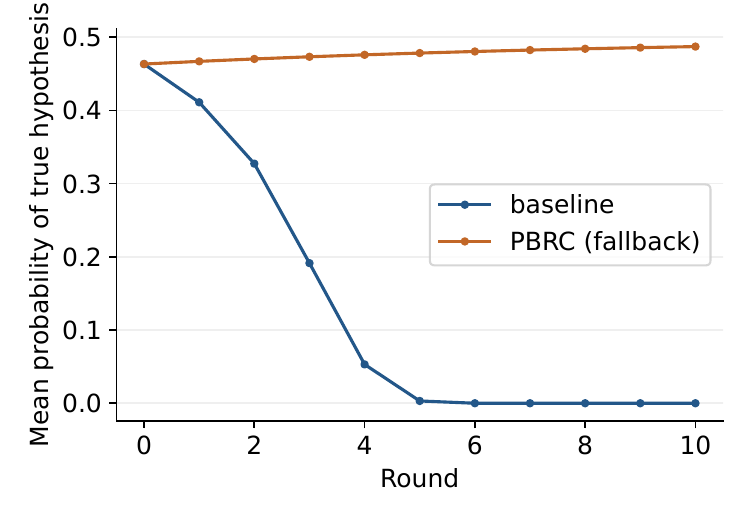}
\includegraphics[width=0.48\linewidth]{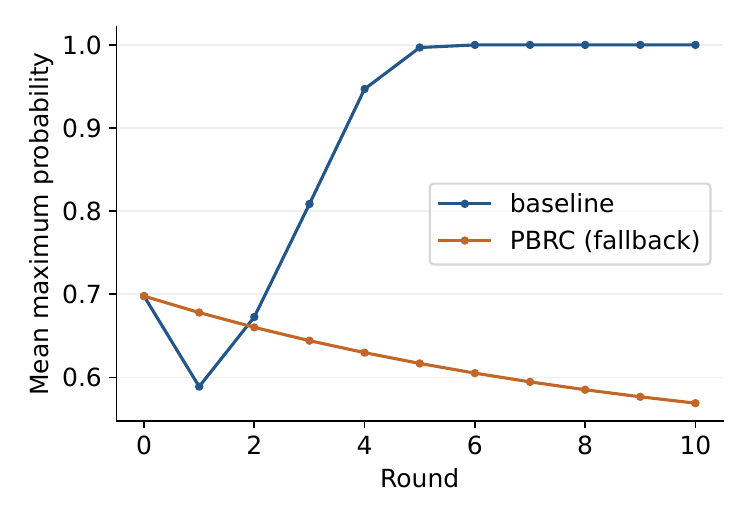}
\caption{Representative complete-graph run: baseline (social pooling + sharpening) versus PBRC fallback under token-empty interaction. Left: mean belief in the true hypothesis. Right: mean confidence $\conf(b)=\max_h b(h)$.}
\label{fig:sim_example_run}
\end{figure}

\noindent\textbf{Cascade rate versus topology.}
We operationalize a \emph{wrong-but-sure cascade} as the event that all agents end with $\argmax b_i^T\neq h_0$ and $\conf(b_i^T)\ge 0.9$.
Across $500$ trials per topology, the baseline exhibits substantial cascade rates on denser graphs, while PBRC eliminates these purely social cascades (Table~\ref{tab:cascade_rates} and Figure~\ref{fig:sim_cascade_rates}).
The baseline explicitly includes confidence sharpening, so this comparison illustrates that chosen mechanism rather than establishing superiority over arbitrary social-learning policies. The symmetric $\mathrm{Beta}(5,5)$ initialization gives each agent a 50\% probability of initially favoring the true hypothesis. PBRC preserves those initial choices; its zero observed cascade rate is not an accuracy gain. The theory rules out newly created cascades under its initial-confidence conditions, not already incorrect confident states.

\begin{table}[t]
\centering
\resizebox{\linewidth}{!}{
\begin{tabular}{lccc}
\toprule
Topology & Ring & ER($p{=}0.3$) & Complete \\
\midrule
Baseline cascade rate & $0.002$ [0.000, 0.011] & $0.256$ [0.220, 0.296] & $0.488$ [0.444, 0.532] \\
PBRC cascade rate     & $0.000$ [0.000, 0.008] & $0.000$ [0.000, 0.008] & $0.000$ [0.000, 0.008] \\
\midrule
Baseline mean confidence at $T$ & $0.988$ & $0.982$ & $0.995$ \\
PBRC mean confidence at $T$     & $0.542$ & $0.543$ & $0.543$ \\
\bottomrule
\end{tabular}}
\caption{Token-empty deliberation: wrong-but-sure cascade rate and mean confidence at $T$ (500 trials; $n{=}20$, $T{=}10$; fixed seed). Brackets show 95\% Wilson confidence intervals for cascade rates. In this token-empty setting, PBRC eliminates confidence amplification by construction; accuracy is governed by initial beliefs rather than evidence-based correction.}
\label{tab:cascade_rates}
\end{table}

\begin{figure}[t]
\centering
\includegraphics[width=0.70\linewidth]{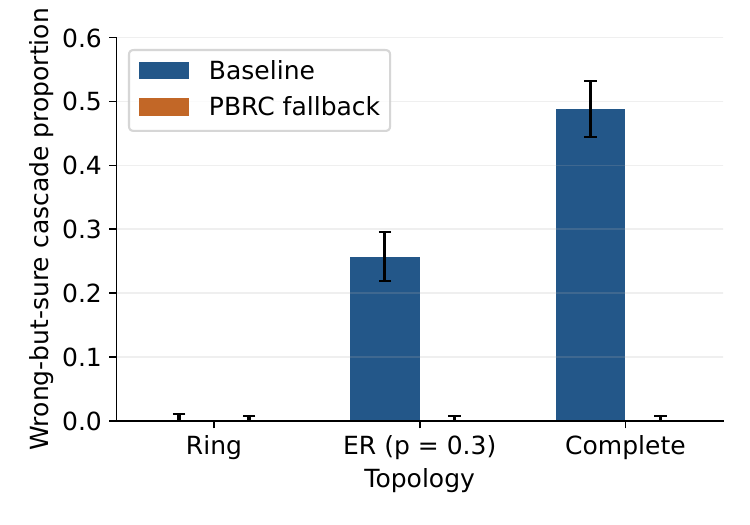}
\caption{Cascade rates for the three tested topologies (500 trials each; Wilson 95\% intervals). These are the observations in Table~\ref{tab:cascade_rates}.}
\label{fig:sim_cascade_rates}
\end{figure}

\subsection{Simulation I(b): dilution ablation}
\label{sec:sim_ablation}

Theorem~\ref{thm:nocascade} is qualitative: in token-empty rounds, if enforcement reduces updates to fallback and the fallback is non-amplifying (and argmax-preserving for top-hypothesis stability), purely social exchange cannot generate wrong-but-sure cascades.
Nevertheless, the fallback choice controls confidence decay. This token-empty experiment does not measure evidence-based correction or liveness.

\noindent\textbf{Setup and results.}
We ablate the skeptical-dilution parameter $\lambda\in\{0,0.02,0.05,0.1,0.2,0.4\}$.
Figure~\ref{fig:sim_ablation} summarizes the outcome.
Across topologies, PBRC's wrong-but-sure cascade rate remains $0$ throughout this ablation, while the baseline exhibits nontrivial cascade rates.
At the same time, increasing $\lambda$ decreases mean confidence at $T$ by driving beliefs toward uniform more aggressively.
This supports the design guidance in Section~\ref{sec:design}: $\lambda$ is best treated as an explicit governance parameter that regulates how strongly the system de-amplifies unvalidated social influence, rather than as a correctness-critical constant.

\begin{figure}[t]
\centering
\begin{minipage}{0.48\linewidth}
\centering
\includegraphics[width=\linewidth]{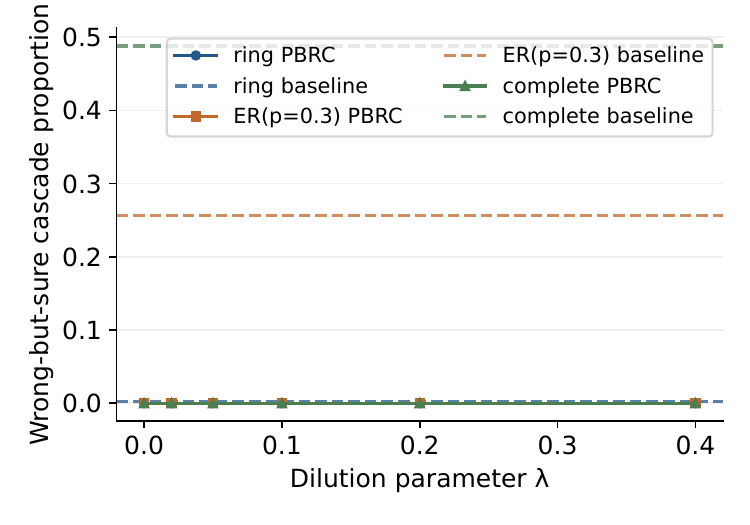}
\end{minipage}\hfill
\begin{minipage}{0.48\linewidth}
\centering
\includegraphics[width=\linewidth]{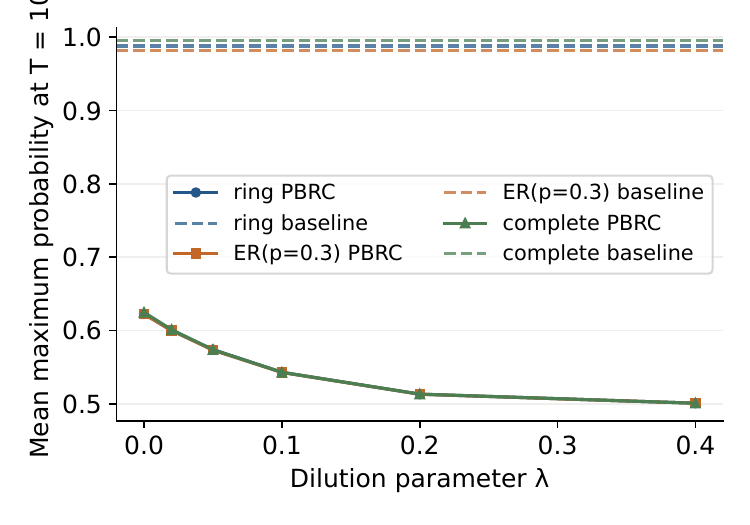}
\end{minipage}
\caption{Ablation over fallback dilution $\lambda$. Left: wrong-but-sure cascade rate remains $0$ for \PBRC{} across topologies, while the baseline exhibits higher cascade rates. Right: mean confidence at $T$ decreases as $\lambda$ increases (baseline shown as reference), controlling overconfidence while preserving the argmax.}
\label{fig:sim_ablation}
\end{figure}

\subsection{Simulation II: token-sufficiency / persuasion separation}
\label{sec:sim_token_sufficiency}
\label{sec:sim_tokensuff}

This experiment empirically instantiates Theorem~\ref{thm:tokensuff} in a minimal, fully enforceable setting.

\noindent\textbf{Setup.}
We fix a token-invariant PBRC contract with two evidential triggers:
``$\exists$ validated token supporting $h_0$'' and ``$\exists$ validated token supporting $h_1$''.
The router holds the belief state and computes the contract operator (matching the enforceability boundary assumed in Section~\ref{sec:enforcement}).
We generate $2000$ pairs of events $(E,E')$ that are \emph{token-equivalent} (same validated token set) but differ arbitrarily in rhetorical content (sender identity, message text, and confidence fields).

\noindent\textbf{Result.}
Across all trials, the enforced updates match exactly (observed mismatch rate $0$), consistent with the theorem for the implemented token-only operators. This checks a construction that ignores the altered metadata; it is not a test of an LLM's susceptibility to rhetoric.

\subsection{Simulation III: topology affects PBRC only via token arrival times}
\label{sec:sim_topology}

This experiment validates the topological characterization of evidence closure under flooding dissemination.

\noindent\textbf{Setup and result.}
We place a unique token at each node and simulate flooding under several connected graph families (ring, ER, star, complete, grid).
For each graph, we measure the time until every agent possesses every token, counting its own initial token at time zero.
Consistent with Corollary~\ref{cor:diameter_tight}, the observed closure time equals the graph diameter in all tested instances (Figure~\ref{fig:sim_diameter}).
This checks accumulated knowledge closure. Incoming per-round exposures instead use the positive-walk formula of Theorem~\ref{thm:flood_distance}; initial possession is not an incoming message.

\begin{figure}[t]
\centering
\includegraphics[width=0.65\linewidth]{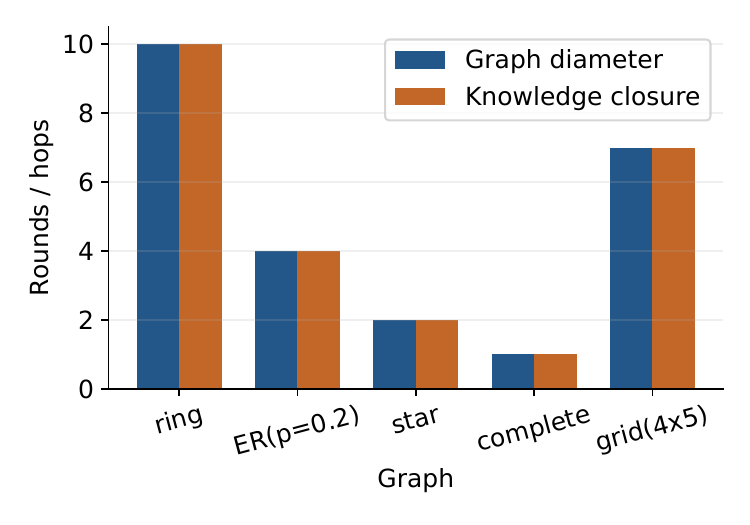}
\caption{Flooding dissemination: time to global token coverage equals the graph diameter (unique token per node).}
\label{fig:sim_diameter}
\end{figure}

\subsection{Simulation IV: sound-but-incomplete routers preserve safety but reduce liveness}
\label{sec:sim_incomplete}

Many results in Sections~\ref{sec:main}--\ref{sec:normalform} assume routers that are sound and complete with respect to token validation.
Operationally, incompleteness (false negatives) is common due to timeouts, conservative parsing, or strict rate limits.

\noindent\textbf{Setup.}
We simulate a single-agent setting with a persistent valid evidence token supporting the true hypothesis.
The router is sound but incomplete: with probability $q$ it fails to recognize the token in a given round (false negative), and thus rejects the evidential trigger even though it is semantically satisfied.

\noindent\textbf{Result.}
For each $q\in\{0,0.1,0.2,0.4,0.6,0.8\}$ the code runs 2,000 trials with a 30-round horizon. It records 31 for trials without correction. Thus Figure~\ref{fig:sim_incomplete_router} reports the mean of $\min(H,31)$, where $H$ is the first successful round, rather than an uncensored hitting-time estimate. Under independent per-round misses, $\mathbb E[\min(H,31)]=(1-q^{31})/(1-q)$; the infinite-horizon mean is $1/(1-q)$. Larger $q$ delays correction in this one-token setting.
Consistent with Proposition~\ref{prop:safety_liveness_router}, safety remains intact---no belief flip occurs without a validated witness---while liveness degrades as evidence recognition becomes less reliable.

\begin{figure}[t]
\centering
\includegraphics[width=0.65\linewidth]{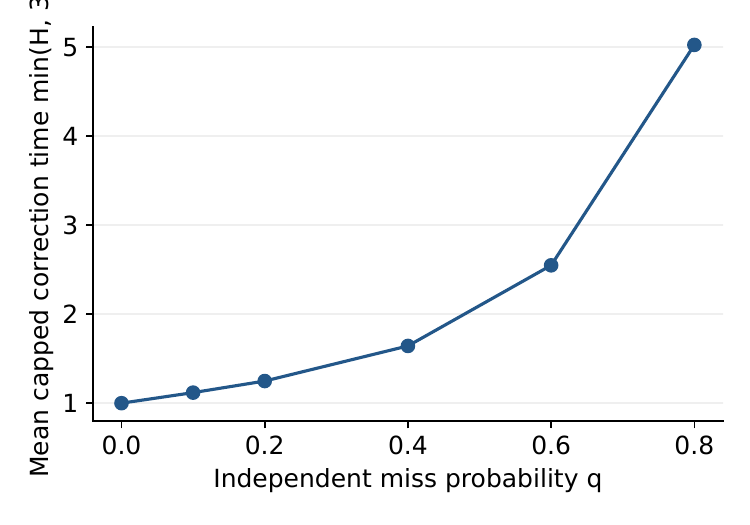}
\caption{One-token correction with independent router misses. The plotted statistic is mean $\min(H,31)$ over 2,000 trials per setting, with uncorrected 30-round trials assigned 31.}
\label{fig:sim_incomplete_router}
\end{figure}

\subsection{Simulation V: token flooding and verification cost}
\label{sec:sim_cost}

PBRC's enforceability rests on token validation and trigger checking, which exposes an availability surface: flooding events with many tokens (possibly valid but irrelevant) increases verification latency and can induce denial-of-service behavior if not budgeted.

\noindent\textbf{Setup.}
We model an event containing $N$ tokens of which exactly one is relevant to the highest-priority trigger.
We compare (i) \emph{full validation} (validate all $N$ tokens) and (ii) \emph{short-circuit validation} (validate tokens until a sufficient witness for the highest-priority trigger is found).
Under random ordering, the expected number of checks for short-circuiting is exactly $(N{+}1)/2$, whereas adversarial ordering negates this benefit.

\noindent\textbf{Result and implications.}
Figure~\ref{fig:sim_cost_dos} plots this abstract check-count model, using 2,000 sampled relevant-token positions per $N\in\{1,6,\ldots,496\}$. It does not measure cryptographic validation, semantic inference, wall-clock latency, or general first-order priority checks.
The experiment motivates two complementary mitigation patterns developed in Section~\ref{sec:complexity}: (a) adopt token-witnessable trigger fragments with small witnesses (enabling early stopping), and (b) enforce per-round validation budgets with explicit deferral fallbacks when budgets are exceeded.

\begin{figure}[t]
\centering
\includegraphics[width=0.65\linewidth]{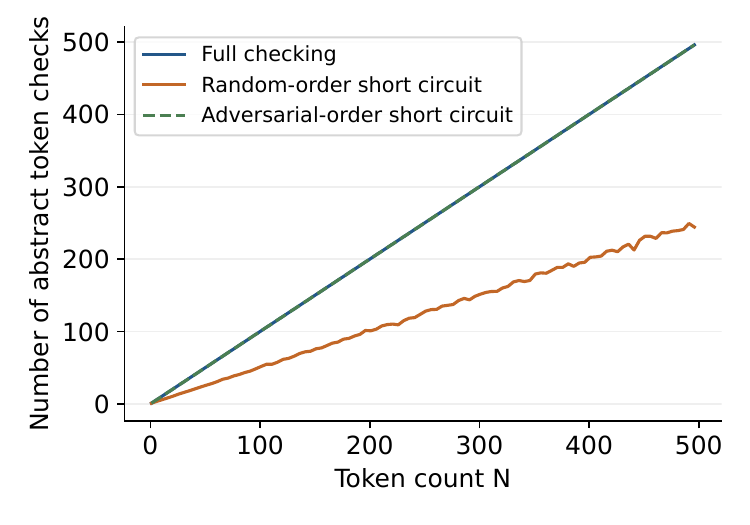}
\caption{Token flooding stress test: validation cost (number of checks) vs token count $N$. Full validation scales linearly; short-circuiting reduces expected cost under random order but not under adversarial ordering.}
\label{fig:sim_cost_dos}
\end{figure}

\subsection{Matched evidence study under label errors}
\label{sec:sim6}
This experiment examines how admission thresholds interact with independent and correlated token-label errors. It compares PBRC with an evidence-only updater using the same numerical operator, a social-pooling comparator, and no interaction. The configuration was fixed locally before execution; code, configurations, and trial-level trajectories accompany the results.

\paragraph{Matched trial design.}
With seed 20260923, we generate 400 independent trials. Each has 25 agents on an undirected Erd\H{o}s--R\'enyi graph with edge probability $0.15$, conditioned on connectivity by rejection sampling. Ground truth $y\in\{0,1\}$ is a fair draw; each initial probability of the truth is independently $\mathrm{Beta}(4,6)$, deliberately favoring an initially incorrect prediction on average. Five distinct token holders are sampled without replacement. Their five tokens have distinct issuer identities and underlying claims equal to $y$. A synthetic allow-list validates exact token objects, without implementing cryptography or semantic entailment.

We flip each token's support label with probability $\epsilon\in\{0,0.1,\ldots,0.5\}$. In the independent regime the five flips are independent; in the fully correlated regime one common flip determines all five labels. Labels remain fixed throughout a trial. Graphs, priors, holders, and error uniforms are shared across compared methods; the same uniforms couple the error-rate settings. These are 12 repeated conditions on 400 independent trial inputs, not 4,800 independent task draws. Eight synchronous rounds forward all retained tokens to neighbors and to the sender itself. With these explicit self-loops, incoming exposure after a delivery equals the newly retained knowledge set. Full coverage occurs by round 8 in 399 trials; the remaining trial is retained without imputation.

\paragraph{Operators and comparators.}
Let $p_i^t=b_i^t(h_1)$, $\ell_i^0=\operatorname{logit}(p_i^0)$, and let $c_{i,h}^t$ count distinct exposed tokens labelled $h$ after $t$ deliveries. The evidence-only updater recomputes
\[
 p_{i,\mathrm{ev}}^t=\sigma\!\left(\ell_i^0+\log(4)(c_{i,1}^t-c_{i,0}^t)\right),
 \qquad \sigma(x)=\frac{1}{1+e^{-x}}.
\]
The fixed weight $\log(4)$ is an experimental parameter, not an estimated likelihood ratio for each noise level. PBRC uses the same anchored operator when some label has at least $k$ tokens and strictly more tokens than the other label, for $k\in\{1,3,5\}$. Otherwise it applies $\delta_{0.1}$. Priority is $h_1$ before $h_0$; the strict-majority branches cannot both hold. The full exposed set is the witness, and full-event truth is checked because majority can change when tokens are added. Recomputing from the fixed initial anchor deduplicates repeated evidence and makes an unchanged accepted set numerically idempotent; it does not require a newly acquired token at every accepted step.

No interaction holds the initial belief fixed. A specified social comparator averages current neighbor logits, excluding self, mixes this with the current local logit at weight $1/2$, sharpens by a factor of 2, and then adds $\log(4)$ times the net count of newly encountered tokens. Writing $\bar\ell_i^{t-1}$ for the neighbor mean and $\Delta c_{i,h}^t$ for new distinct-token counts, its update is
\[
 p_{i,\mathrm{soc}}^t=\sigma\!\left(\operatorname{clip}_{[-20,20]}
 \left[\ell_i^{t-1}+\bar\ell_i^{t-1}
       +\log(4)(\Delta c_{i,1}^t-\Delta c_{i,0}^t)\right]\right).
\]
The comparator's seen-token set is initially empty, so locally held tokens count on its first round. It has the same evidence budget but different dynamics; it is one specified pooling baseline, not a tuned or universally representative competitor. All parameters were fixed before inspecting outcomes.

\paragraph{Measurement and verification.}
Prediction selects $h_1$ when $p_i^t\ge1/2$. Metrics first average over agents within a trial. The binary Brier score is $(1-b_i^t(y))^2$, rather than the two-class sum; lower is better, but this score alone does not establish calibration. Final harmful and beneficial flips are fractions of all agents whose final correctness differs from their initial correctness. Acceptance is the fraction of agent-rounds using a PBRC branch. First-correct time is zero for initially correct agents; we report its mean capped at round 8 and separately the fraction never correct by that horizon. A first correct prediction need not remain correct.

Two-sided 95\% $t$ intervals use the 400 trial-level observations, with 399 degrees of freedom. Comparator intervals use paired within-trial differences. They are descriptive and not adjusted for multiple comparisons. Shared agents, methods, or noise settings are not treated as independent trials. The vectorized dynamics were checked against the reference router on prespecified trials 0 and 1, every regime/rate cell, all thresholds, agents, and rounds: 14,400 numerical transitions and acceptance decisions agreed exactly. Separate flooding checks also passed. These checks do not constitute a general correctness proof or a latency benchmark.

\begin{table}[t]
\centering
\small
\setlength{\tabcolsep}{4pt}
\begin{tabular}{lrrrr}
\toprule
& \multicolumn{2}{c}{Independent errors} & \multicolumn{2}{c}{Fully correlated errors} \\
Method & Accuracy (\%) [95\% CI] & Brier & Accuracy (\%) [95\% CI] & Brier \\
\midrule
No interaction & 25.80 [24.91, 26.69] & 0.3787 & 25.80 [24.91, 26.69] & 0.3787 \\
Evidence only & 91.53 [89.00, 94.06] & 0.0725 & 84.00 [80.39, 87.61] & 0.1597 \\
Social pooling & 81.69 [78.24, 85.14] & 0.1817 & 83.93 [80.32, 87.54] & 0.1607 \\
PBRC $k=1$ & 91.53 [89.00, 94.06] & 0.0725 & 84.00 [80.39, 87.61] & 0.1597 \\
PBRC $k=3$ & 91.53 [89.00, 94.06] & 0.0725 & 84.00 [80.39, 87.61] & 0.1597 \\
PBRC $k=5$ & 51.05 [47.54, 54.56] & 0.1970 & 83.99 [80.38, 87.60] & 0.1598 \\
\bottomrule
\end{tabular}
\caption{Matched study, $\epsilon=0.2$, round 8: means over 400 paired trials. Intervals describe uncertainty across trials, not independent agent observations. All rates, metrics, trajectories, and paired differences are included in the source package.}
\label{tab:new_evidence}
\end{table}

\paragraph{Results and limits.}
Table~\ref{tab:new_evidence} gives the $\epsilon=0.2$ slice; Figure~\ref{fig:new_evidence} shows every error rate. Under independent errors, PBRC with $k=3$ reaches 91.53\% final accuracy, compared with 25.80\% for no interaction and 81.69\% for the specified social comparator. The paired accuracy difference from social pooling is 9.84 percentage points (95\% CI 7.43--12.25); the Brier difference is $-0.1092$ (95\% CI $[-0.1325,-0.0858]$). However, PBRC with $k=1$ or $k=3$ has the same final accuracy as the matched evidence-only updater in every tested cell. When all five tokens arrive, their majority contains at least three tokens, so both thresholds apply the evidence-only operator by construction. This experiment therefore supplies no final-accuracy benefit of these gates over that updater; exact agreement in these samples does not establish population equivalence.

At independent error rate 0.2, $k=3$ permits beneficial final flips for 67.76\% of agents and harmful flips for 2.03\%, with 78.18\% of agent-rounds accepted. Raising the threshold to $k=5$ lowers acceptance to 23.17\% and final accuracy to 51.05\%. Its capped mean first-correct time is 4.80 rounds, compared with 2.34 for $k=3$ and 1.57 for evidence only; the corresponding never-correct fractions are 48.95\%, 6.19\%, and 4.68\%. These jointly expose the cost of strict gating; the capped time alone is not an uncensored time-to-correction estimate.

Under fully correlated errors, $k=1$ and $k=3$ have identical final numerical beliefs in these trials, and $k=5$ provides no substantive final-accuracy protection. Distinct issuer identities do not supply statistical independence. At error rate 0.5, the $k=3$ Brier scores are 0.4385 (independent) and 0.4643 (correlated), worse than 0.3787 with no interaction despite higher classification accuracy. Thus neither admissibility nor improved top-label accuracy guarantees better probability forecasts. The controlled study demonstrates permitted correction and a safety--liveness trade-off within this model; practical LLM utility, semantic validation, acquisition bias, and verification latency remain outside this synthetic experiment. Section~\ref{sec:scifact} separately examines answer-level admission with retrieved text and an imperfect NLI validator.

\begin{figure}[t]
\centering
\includegraphics[width=0.95\linewidth]{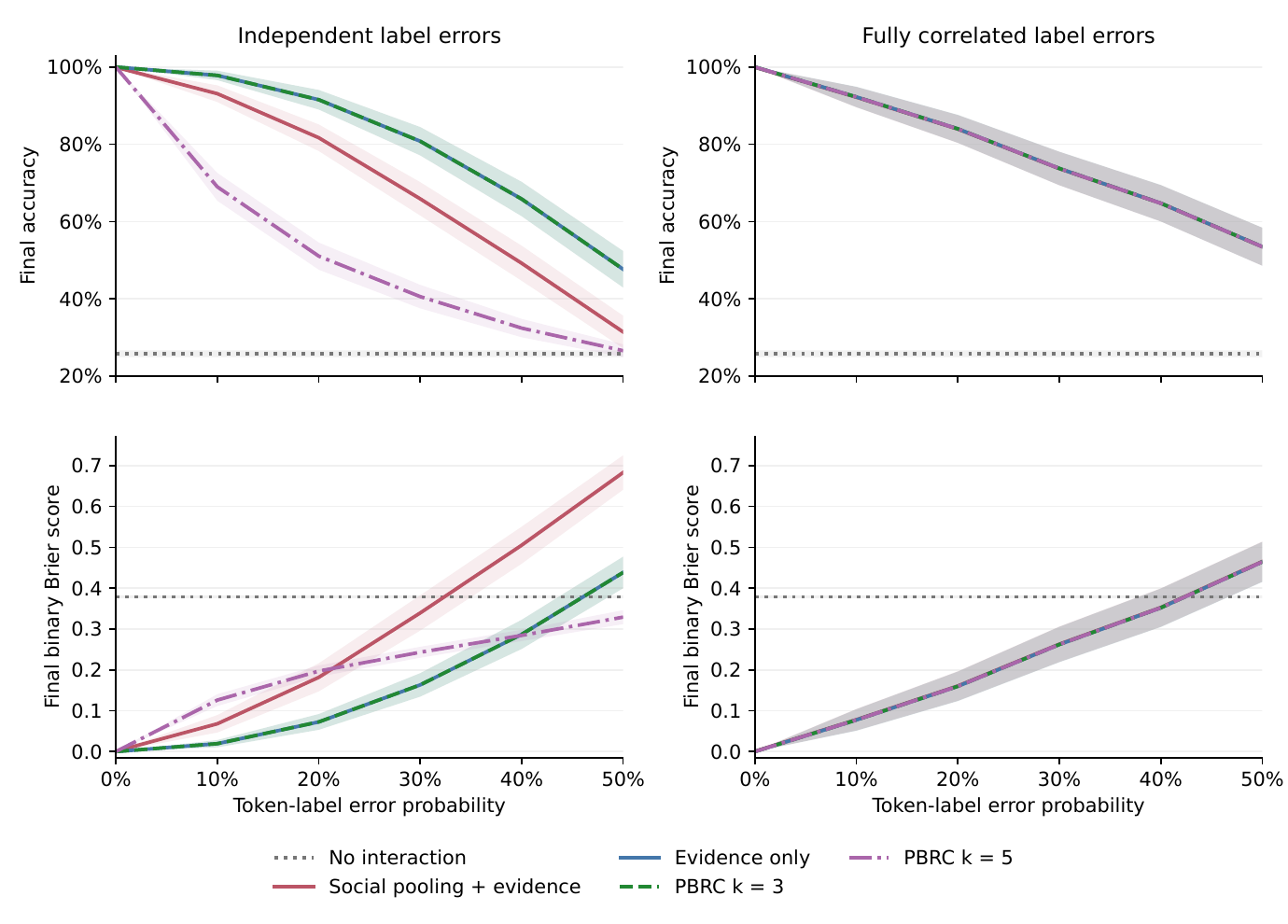}
\caption{Matched evidence study: final accuracy and binary Brier score under independent and fully correlated token-label errors. Bands show descriptive 95\% trial-level $t$ intervals. Evidence-only and PBRC $k=3$ curves overlap closely; $k=1$ is omitted from the plot because its final accuracy equals theirs in every cell, with full results in the supplied data.}
\label{fig:new_evidence}
\end{figure}

\subsection{KAIROS: token-empty answer retention}
\label{sec:benchmark_integration}\label{sec:bench_protocols}
We analyze 3,000 KAIROS \cite{song2025kairos} question instances under three conditions: independent RAW answers, SOCIAL interaction, and SOCIAL interaction with reflection. Conditions are paired by question and options, source dataset, and gold label. Outcomes use the recorded extracted predictions. A token-empty PBRC postprocessor applies conservative answer retention to the reflected outputs.

\begin{table}[t]
\centering\small
\setlength{\tabcolsep}{5pt}
\begin{tabular}{lrrrrrr}
\toprule
Condition & Correct & Accuracy (\%) & Changed & Harmful & Beneficial & Neutral \\
\midrule
RAW & 2172 & 72.4 & 0 & 0 & 0 & 0 \\
SOCIAL & 2034 & 67.8 & 838 & 411 & 273 & 154 \\
SOCIAL + reflection & 2094 & 69.8 & 697 & 327 & 249 & 121 \\
PBRC, token-empty & 2172 & 72.4 & 0 & 0 & 0 & 0 \\
\bottomrule
\end{tabular}
\caption{KAIROS evaluation, 3,000 aligned rows. Changes and harmful/beneficial counts are relative to RAW; neutral changes switch between incorrect answers. Token-empty PBRC retains RAW by construction.}
\label{tab:kairos}
\end{table}

Table~\ref{tab:kairos} retains every recorded row. The two-sided exact McNemar values against RAW are $1.476633709\times10^{-7}$ for SOCIAL and $0.00131505$ for SOCIAL with reflection. These are descriptive, unadjusted item-level comparisons. There are 2,999 unique dataset/question fingerprints: one BrainTeaser question occurs twice and is correct in every condition. Removing one copy leaves all discordant counts, and hence these exact-test values, unchanged. Shared benchmark sources and related items further limit an independent-item interpretation.

The token-empty postprocessor admits no evidence witness and returns RAW whenever a reflected answer differs. It therefore blocks all 697 reflected changes, including 327 harmful, 249 beneficial, and 121 neutral changes. Its accuracy is identically RAW accuracy; this is an answer-retention demonstration, not an independently generated reasoning improvement, an evidence-enabled result, or a numerical-belief compliance test.

The model is identified as \texttt{gpt4o}. An immutable model snapshot, complete decoding settings, all sampled completions, and full aggregation/extraction details are unavailable, limiting exact replication of model generation. RAW vote histograms contain eight votes in 2,945 instances and fewer in 55, with 103 modal ties. Analyses use the recorded decisions without imposing an additional tie-breaking rule.

\subsection{SciFact: evidence-enabled claim verification}
\label{sec:scifact}
The SciFact \cite{wadden2020scifact} experiment evaluates answer admission using retrieved scientific passages and an imperfect NLI validator. It measures categorical answer-level behavior under unverified persuasive messages, complementing the token-empty KAIROS analysis. This instantiation does not implement the full numerical-belief model or KD45/PDL monitor.

\paragraph{Data, protocol, and models.}
The primary \texttt{main200} profile selects the first 200 of the 300 public development claims under the SHA-256 ordering of \texttt{20260730:claim\_id}. Selection uses no labels. Their gold classes are 82 SUPPORT, 41 CONTRADICT, and 77 NOT\_ENOUGH\_INFO (NEI). All 809 training claims are reserved for calibration and have disjoint IDs. Texts, labels, and evidence-document IDs come from the official SciFact archive identified in the protocol. These are development-set claim classifications; they are not official test-set or joint rationale-retrieval leaderboard scores.

The subject, treatment generator, and query generator use Qwen2.5-7B-Instruct \cite{qwen2024technical} through vLLM on the RunPod BF16 backend, at temperature zero and request seed zero. Maximum output lengths are 180, 260, and 64 tokens, respectively. The model revision is \path{a09a35458c702b33eeacc393d103063234e8bc28}. The NLI model is \path{MoritzLaurer/DeBERTa-v3-base-mnli-fever-anli} \cite{laurer_deberta_model}, revision \path{6f5cf0a2b59cabb106aca4c287eed12e357e90eb}, with a 384-token input limit. Hardware and library differences can prevent bit-identical fresh inference even with these settings.

The protocol fixes model revisions, configuration, selected IDs, and calibration before held-out generation; predictions are finalized before scoring. Generation reads unlabeled claims and the corpus, with a software guard blocking access to the score-only gold directory. Hashes record the local configuration; this procedure is not an independent public preregistration or a security boundary against administrator access. Held-out claims are excluded from tuning in this study; their absence from model pretraining is not established.

\paragraph{Acquisition and calibration.}
BM25 searches all 5,183 corpus abstracts using the claim alone, with $k_1=1.5$ and $b=0.75$. It retrieves eight documents and selects two sentences per document by lexical score, supplying up to 16 title-plus-sentence passages. A separate acquisition condition uses a model-generated query influenced by the majority messages. NLI treats each passage as premise and the claim as hypothesis. Let $s$ and $c$ be the maximum entailment and contradiction scores across retrieved passages. It abstains when $\max(s,c)<\theta$ or $|s-c|<m$; otherwise it selects SUPPORT if $s>c$, and CONTRADICT otherwise. There were no equal-maxima ties in either held-out stream.

Training IDs are partitioned deterministically into 561 tuning and 248 internal-validation claims. The 45-point grid combines $\theta\in\{0.50,0.55,\ldots,0.90\}$ and $m\in\{0,0.05,\ldots,0.20\}$. The selection rule maximizes tuning macro-F1 among settings with non-abstaining precision at least 0.8; if none qualify, it maximizes macro-F1 over the whole grid, breaking ties by higher precision, coverage, and accuracy, then lower threshold and margin. \emph{No setting met the 0.8 target.} The selected $(\theta,m)=(0.90,0)$ has tuning precision 0.560 and internal-validation precision 0.508. These values are not semantic-truth guarantees. The selected setting remains fixed; the archived held-out grid is descriptive sensitivity analysis, not a replacement selection step.

\paragraph{Social interventions and comparators.}
Each claim has eight subject outputs: RAW; terse, fluent, prestige, and three-message majority pressure; clean RAG; clean RAG plus majority pressure; and socially steered RAG plus majority pressure. One further call generates the treatment messages and one generates the acquisition query, yielding 2,000 recorded call outputs in total. The treatment stance opposes RAW's SUPPORT/CONTRADICT label; when RAW is NEI, an ID-based rule chooses one of those labels. It never uses gold. The three majority messages come from one generator call, not three independent agents or independent evidence sources.

The generator is instructed to supply evidence-free rhetoric. Some outputs nevertheless assert unsupported studies, research, or data; for example, claim 845's messages invoke studies without supplying a checkable source. We therefore interpret these conditions as \emph{unverified persuasive messages}, including asserted empirical support, rather than a clean isolation of social prestige or stance from semantic persuasion. No externally validated evidence token is carried by these messages.

Citation-presence gating accepts a clean-RAG-majority answer if it cites any retrieved document, otherwise retaining RAW. Validator-only gating uses a non-abstaining clean NLI label, otherwise retaining RAW. Full PBRC considers the same clean-RAG-majority candidate and accepts only a non-abstaining token with a matching payload hash, exact claim/context binding, freshness, candidate-label agreement, and citation of the token's document. Its fallback retains RAW. The steered-acquisition version uses the corresponding steered candidate and token. The nonconservative ablation uses the same clean admission gate but falls back to the social-majority answer. A token abstention is represented as NEI, so this gate cannot make an accepted correction to NEI; it can retain an initially correct NEI answer through fallback.

\begin{table}[t]
\centering\small
\setlength{\tabcolsep}{3.5pt}
\begin{tabular}{lrrrrrr}
\toprule
Method & Correct & Accuracy (\%) [95\% CI] & Macro-F1 & H & B & $p$ vs RAW \\
\midrule
RAW & 98 & 49.0 [42.2, 55.9] & .466 & 0 & 0 & 1 \\
Social, terse & 46 & 23.0 [17.7, 29.3] & .200 & 87 & 35 & $2.79\times10^{-6}$ \\
Social, fluent & 46 & 23.0 [17.7, 29.3] & .196 & 90 & 38 & $4.92\times10^{-6}$ \\
Social, prestige & 45 & 22.5 [17.3, 28.8] & .192 & 89 & 36 & $2.37\times10^{-6}$ \\
Social, majority & 47 & 23.5 [18.2, 29.8] & .205 & 89 & 38 & $6.97\times10^{-6}$ \\
Clean RAG & 123 & 61.5 [54.6, 68.0] & .580 & 20 & 45 & .00263 \\
Clean RAG + majority & 67 & 33.5 [27.3, 40.3] & .305 & 70 & 39 & .00385 \\
Steered RAG + majority & 65 & 32.5 [26.4, 39.3] & .298 & 75 & 42 & .00293 \\
Citation presence & 67 & 33.5 [27.3, 40.3] & .310 & 69 & 38 & .00353 \\
Validator only & 100 & 50.0 [43.1, 56.9] & .489 & 26 & 28 & .892 \\
PBRC full & 106 & 53.0 [46.1, 59.8] & .527 & 12 & 20 & .215 \\
PBRC, steered acquisition & 104 & 52.0 [45.1, 58.8] & .517 & 12 & 18 & .362 \\
PBRC, social fallback & 59 & 29.5 [23.6, 36.2] & .258 & 81 & 42 & .000556 \\
\bottomrule
\end{tabular}
\caption{SciFact main200: every prespecified method and all 200 claims. H/B are harmful/beneficial changes relative to RAW. Accuracy intervals are Wilson intervals; $p$ values are two-sided exact McNemar tests. They are descriptive, unadjusted for multiple comparisons, and assume independent items. Related claims can violate this assumption.}
\label{tab:scifact_main}
\end{table}

\paragraph{Results and acceptance errors.}
Table~\ref{tab:scifact_main} reports all methods. Full PBRC achieves $106/200=53.0\%$, compared with $98/200=49.0\%$ for RAW, $123/200=61.5\%$ for clean RAG, and $67/200=33.5\%$ for clean RAG under majority pressure. Its paired difference from RAW is 4.0 percentage points, with a 5,000-resample paired-bootstrap percentile interval $[-1.5,9.5]$ points and exact $p=0.215327$. Thus the run does not establish an overall improvement over RAW and does not outperform clean RAG. Exploratory paired comparisons found 58 PBRC-only versus 19 pressured-RAG-only correct answers, but 22 PBRC-only versus 39 clean-RAG-only correct answers; these exploratory comparisons are secondary to the prespecified PBRC--RAW comparison.

Full PBRC accepts 51 non-fallback updates and uses fallback on 149 claims (74.5\%). Only 36 accepted updates change the answer: 20 beneficial, 12 harmful, and four wrong-to-different-wrong changes. Across all 51 accepts, 31 are correct and 20 incorrect (accepted precision 60.8\%); among the 36 actual changes, 20 are correct (55.6\%). These denominators must not be conflated. Citation presence accepts 196 updates, of which 132 are incorrect; validator only accepts 108, of which 54 are incorrect. Steered PBRC accepts 48, of which 21 are incorrect, and scores 52.0\%. The social-fallback ablation keeps the same 51 accepted updates but falls to 29.5\% overall, illustrating the importance of fallback behavior in this setting.

The clean validator has 54.0\% non-abstaining coverage and 50.0\% precision among those labels. Retrieval finds at least one gold evidence document for 91.1\% of the 123 evidence-bearing claims; this hit rate is not complete evidence-set or rationale recall. Steered retrieval changes all 200 query strings, has mean document-set Jaccard similarity 0.465 with clean retrieval, and reaches 91.9\% hit rate. These acquisition results do not establish improved evidence quality: steered PBRC is slightly less accurate in this sample.

\paragraph{Parsing, checks, and scope.}
All four subject parse errors and seven failed peer parses remain in the analysis; query parsing has no failures. Subject parse failure maps to NEI, which is also a scored class, rather than automatically counting as incorrect. A successful parse does not guarantee prompt compliance. For example, RAW supplies nonempty citations despite receiving no passages in 120/200 outputs. Clean RAG, clean RAG with majority, and steered RAG cite at least one document absent from their prompts on 13, six, and nine claims, respectively. The full gate requires the token's document to be cited but does not reject every additional spurious citation. All scores use this same fixed parsing and admission policy.

The 108 non-abstaining clean tokens support controlled context, stale-token, and payload-mutation checks: full PBRC rejects every mutated case while the corresponding check-removal ablation accepts all; intact controls are accepted. These are secondary mechanism tests with constructed matching candidates, not naturally occurring attack rates. The token integrity field is an unkeyed SHA-256 content hash, not a digital signature: an attacker able to replace the token can recompute it. Freshness uses the fixed protocol clock $1700000000+1$ and a seven-day expiry interval, not observed wall-clock replay protection. The study therefore tests trusted in-process admission and syntactic consistency, not adversarial issuer authentication or semantic correctness.

The results quantify a conditional reduction in admitted social-pressure changes alongside retained errors and missed corrections. Generalization is limited by the use of one model, one selected development subset, one imperfect validator, and one fixed pressure protocol.

\section{Limitations and open problems}
\label{sec:limitations}
The guarantees concern a fixed finite hypothesis space, explicit numerical beliefs, specified event semantics, and a faithful enforcement boundary. They do not show that an LLM internally represents these probabilities or reasons according to the contract. Hypothesis generation, changing ontologies, asynchronous delivery, expiry, revocation, and stateful spent-token tracking require additional formal treatment.

Evidence gating preserves the chosen fallback invariants on token-empty rounds. It does not guarantee truth, calibrated probabilities, useful learning, balanced evidence acquisition, or eventual correction. An overly restrictive contract can permanently retain an incorrect initial top hypothesis. Omission and query steering can bias evidence-bearing updates. An accepted witness establishes a contractual license; it is neither a complete numerical replay record nor proof of a unique causal explanation.

Token invariance is conditional on fixed context and annotations. General contracts may depend on message metadata, and sound-but-incomplete routers may make different acceptance decisions on token-equivalent events. Exact trajectory equivalence therefore needs the stated operator, fallback, and router assumptions. The normal-form construction needs suitable witnesses and first-order definability; it is not a representation theorem for arbitrary computational policies. Syntactic evidence gating is sufficient but not necessary for top preservation and confidence non-amplification.

Freshness limits the age of usable evidence but does not prevent repeated use inside a window. Authentication proves origin under a threat model, not semantic truth. Multi-attestation bounds require distinct authenticated identities and strong assumptions about honest validators and compromise dependence. Numerical claims must be enforced by a state-holding router or a correct operator-verification mechanism.

The logical results here establish soundness and exact monitoring of explicitly supplied finite models. All-model completeness and the finite-model property remain open questions for the combined system. Finite-prefix monitoring does not prove unobserved future behavior, and the belief modality needs semantics beyond a numerical log.

The simulations use idealized dynamics, with simulations I--V evaluated at one seed. The matched study uses fixed evidence weights, informative synthetic source claims, controlled label errors, and one pooling comparator. In KAIROS, token-empty gating preserves RAW by construction, while incomplete generation metadata limits exact replication. In SciFact, PBRC has no statistically established improvement over RAW and lower observed accuracy than clean RAG. The validator fails to meet the calibration precision target and admits incorrect answers; unsupported empirical assertions in peer messages preclude a pure-rhetoric interpretation. Claim dependence, one model/backend, a single selected development subset, and unadjusted comparisons limit generalization. Unkeyed token hashes and a synthetic protocol clock do not implement authenticated evidence or real-time replay security. Neither the simulation code nor the finite token-set reference router supplies a production semantic validator.

\section{Conclusion}
\label{sec:conclusion}
PBRC makes admissible belief changes a public contract: a selected evidence-bearing branch requires a checked witness and faithfully executed operator, while other steps use a declared fallback. Argmax-preserving, non-amplifying fallback yields precise token-empty invariants. A local transition characterization explains why those invariants do not require every safe protocol to use syntactic evidence gates.

Under the stated definability and witness assumptions, evidence splitting yields a behavior-preserving normal form; sound enforcement and token invariance support conditional trace equivalences. Flooding separates accumulated knowledge, governed by shortest-path distance, from incoming exposures, governed by positive-length walks. The logic supplies sound specifications and finite-log monitoring with an explicit operational bridge.

These results make admissibility and selected failure modes inspectable while leaving evidence truth, acquisition bias, numerical compliance, and practical utility as explicit obligations. The matched synthetic study shows beneficial revision but no final-accuracy gain over its evidence-only updater at the lower thresholds, and correlated errors defeat the benefit of distinct issuer identities. KAIROS experiments demonstrate the benefits and costs of retaining an initial answer. In SciFact, evidence-enabled PBRC attenuates the degradation associated with the specified majority-pressure condition, but admits wrong answers, has an inconclusive gain over RAW, and falls below clean RAG accuracy. This supports an auditable admission mechanism with explicit assumptions, not a general claim that stricter evidence gating improves reasoning.

\medskip\noindent\textbf{Data and code availability.}
The accompanying source package contains the simulation implementations, finite token-set reference router, matched-study configurations and trajectories, KAIROS predictions, and SciFact protocol, calibration, retrieved passages, NLI scores, model outputs, admission decisions, and analysis scripts. The base simulation repository is \url{https://github.com/alqithami/PBRC}; its commit is specified in Section~\ref{sec:eval_metrics}. Stored-output analyses can be reproduced without model inference. Third-party data and model licenses apply to the corresponding artifacts.

\begin{appendices}

\section{Soundness and finite-model monitoring for \CDDL{}}
\label{app:pdl}
\subsection{Proof of soundness}
Fix any model in the class of Theorem~\ref{thm:cddl_sc}. Propositional tautologies are valid by Boolean semantics. Normality for either modality follows because a universal successor quantifier preserves implication. Modus ponens preserves truth at each state. Necessitation preserves global validity, since a globally valid formula holds at every accessible successor.

For $B_i$, seriality proves $\mathbf D$: $B_i\varphi$ and $B_i\neg\varphi$ cannot both hold since an accessible successor exists. Transitivity proves $\mathbf4$: if all successors of $w$ satisfy $\varphi$, every successor of a successor is also a successor of $w$, so $B_iB_i\varphi$ holds. Euclideanness proves $\mathbf5$: if $\neg B_i\varphi$ holds at $w$, choose $v$ with $wR_iv$ and $\neg\varphi$ at $v$. For every $u$ with $wR_iu$, Euclideanness gives $uR_iv$, hence $\neg B_i\varphi$ at $u$. Thus $B_i\neg B_i\varphi$ holds at $w$.

The sequence axiom follows because an $\alpha;\beta$ successor is precisely a $\beta$ successor of an $\alpha$ successor. Choice is union, so universal quantification over it is the conjunction of the two boxes. A test has only the identity edge at states satisfying its formula, proving $[\psi?]\varphi\leftrightarrow(\psi\to\varphi)$.

For $\mathbf{Star}$, if $[\alpha^\ast]\varphi$ holds at $w$, the zero-length path gives $\varphi$ at $w$. Prefixing an arbitrary finite path by an $\alpha$ step shows $[\alpha][\alpha^\ast]\varphi$. Conversely, assume $\varphi\land[\alpha][\alpha^\ast]\varphi$ at $w$. Any finite $\alpha$ path from $w$ either has length zero or consists of a first step followed by an $\alpha^\ast$ path. The respective conjunct ensures $\varphi$ at its endpoint. This proves the star equivalence.

For $\mathbf{Ind}$, suppose $\varphi\to[\alpha]\varphi$ is globally valid. Starting at any state satisfying $\varphi$, induction on finite path length shows that each $\alpha$-reachable state satisfies $\varphi$. Therefore $\varphi\to[\alpha^\ast]\varphi$ is globally valid. Applying these axiom and rule facts inductively to any finite derivation proves Theorem~\ref{thm:cddl_sc}.

\subsection{Finite-model algorithm and correctness}
Let $W$ be the supplied finite state set. Represent a formula by its truth set and a program by a binary relation on $W$. Atoms use the supplied valuation; negation uses complement in $W$; conjunction uses intersection. For $B_i\varphi$, retain exactly states whose $R_i$ successors all belong to the already computed truth set of $\varphi$. For $[\alpha]\varphi$, use the same universal-successor operation with the computed $R_\alpha$.

For programs, atomic actions use supplied relations, choice uses union, sequence uses Boolean relational composition, and $\psi?$ uses the identity restricted to the computed truth set of $\psi$. For $\alpha^\ast$, start with the identity and repeatedly add pairs reachable by one further $R_\alpha$ step until no pair changes. There are only $|W|^2$ pairs, so the procedure terminates. Its invariant is that after $r$ extensions it contains exactly the pairs connected by paths of length at most $r$; stabilization therefore equals reflexive transitive closure.

Mutual induction on the finite syntax trees of formulas and programs now proves each computed set and relation equals its semantic interpretation. Tests contain strictly smaller formula subtrees, and modalities contain strictly smaller program and formula subtrees, so the recursion is well founded. Using Boolean matrices and a Floyd--Warshall transitive-closure computation, an implementation takes at most $O(L|W|^3)$ time and $O(L|W|^2)$ storage, where $L$ is total formula/program syntax size; this deliberately coarse bound includes closure and composition. Event parsing, token validation, and first-order trigger evaluation are separate preprocessing costs. This proves Theorem~\ref{thm:finite_monitor}.

For a finite run, Lemma~\ref{lem:run_bridge} supplies the action relation. Enumerating reachable source indices then checking their actual successors returns the earliest violation of either displayed PBRC invariant. No doxastic accessibility relation is inferred from numerical beliefs.

\end{appendices}


\end{document}